\documentclass{article}
\usepackage{natbib}
\usepackage{amsmath}
\usepackage{amsthm}
\usepackage{amssymb}
\usepackage{mathtools}
\usepackage{bbm}
\usepackage{bm}
\usepackage{tikz}
\usepackage{nicefrac}
\usepackage{pgfplots}
\usepackage{tabularx}
\usepackage{wrapfig}
\usepackage{subcaption}
\usepackage{hhline}
\usepackage{algorithm,algorithmicx}
\usepackage[noend]{algpseudocode}
\usetikzlibrary {datavisualization}
\usetikzlibrary{arrows.meta, positioning} 
\usetikzlibrary{patterns}
\usetikzlibrary{calc}
\usepackage{graphicx} 
\usepackage{hyperref}
\usepackage{cleveref}
\usepackage{color}
\usepackage[suppress]{color-edits}
\addauthor{hb}{magenta}
\definecolor{dgreen}{RGB}{0,165,0}
\addauthor{am}{dgreen}
\usepackage{enumitem}
\usepackage[margin=1in]{geometry}
\usepackage{authblk}

\usepackage{enumitem}

\newcommand{\Ins}{\texttt{Insert}}

\newcommand{\Draw}{\texttt{Draw}}
\newcommand{\Wei}{\texttt{CReward}}
\newcommand{\RT}{\texttt{Range Tree}}
\newcommand{\IT}{\texttt{Interval Tree}}
\newcommand{\LC}{\texttt{Left\_child}}
\newcommand{\ROOT}{\texttt{Root}}
\newcommand{\RC}{\texttt{Right\_child}}
\newcommand{\EH}{\texttt{Height}}

\newcommand{\RRF}{h}
\newcommand{\PAR}{\texttt{Par}}
\newcommand{\LZM}{\texttt{Lazy\_Update\_Message}}

\newcommand{\RA}{\texttt{Int}}

\newcommand{\Update}{\texttt{Update}}
\newcommand{\Hp}{\Psi} %
\newcommand{\hp}{\psi} %
\newcommand{\Ax}{\Gamma}
\newcommand{\ax}{\gamma}
\newcommand{\bp}{c}

\newcommand{\Unif}{\mathcal{U}}

\newcommand{\Poly}{\text{Poly}}
\newcommand{\Rg}{\mathcal{P}}
\newcommand{\rg}{\rho}
\newcommand{\fu}{f}
\newcommand{\Fu}{F}
\newcommand{\SC}{A}
\newcommand{\Exp}{\texttt{Weighted Majority}}
\newcommand{\Band}{\texttt{EXP3}}
\newcommand{\param}{n}
\newcommand{\Alg}{\texttt{ALG}}
\newcommand{\Opt}{\texttt{OPT}}
\newcommand{\Tree}{\mathcal{T}}
\newcommand{\ATree}{\mathcal{E}}

\newcommand{\LZL}{\mathcal{L}}
\newcommand{\LZR}{\mathcal{R}}

\newcommand{\AS}{\text{action parameter space}}

\newcommand{\FP}{\text{reward function parameters}}
\newcommand{\rew}{\omega}
\newcommand{\CR}{r}
\newcommand{\MF}{g} 
\newcommand{\Dist}{\mathcal{D}} 
\newcommand{\FCA}{\texttt{First\_Common\_Ancestor}}
\newcommand{\Loc}{\texttt{Locate}}
\DeclareMathOperator{\Expect}{\mathbbm{E}}
\DeclarePairedDelimiter\abs{\lvert}{\rvert}

\newtheorem{theorem}{Theorem}[section]
\newtheorem{corollary}{Corollary}[theorem]
\newtheorem{lemma}[theorem]{Lemma}

\newtheorem{observation}{Observation}
\newtheorem{remark}{Remark}
\newtheorem{example}{Example}

\newtheorem{definition}{Definition}

\algnewcommand\algorithmicinput{\textbf{Input:}}
\algnewcommand\Input{\item[\algorithmicinput]}

\title{Efficient Online Proportional Sampling with Applications to Smoothed Online Learning}

\author[1]{Amirmahdi Mirfakhar}
\author[2]{Maria-Florina Balcan}
\author[1]{Hedyeh Beyhaghi}

\affil[1]{University of Massachusetts Amherst\\}
\affil[2]{Carnegie Mellon University\\
}

\date{}

\begin{document}

\maketitle

\begin{abstract}
We study the problem of efficient online proportional sampling from a high-dimensional domain under a $\sigma$-smoothed adversary, where the sampling distribution is induced by a dynamically evolving weight function defined over a sequence of piecewise-structured partitions. This setting captures a broad range of applications,
including principal-agent games (e.g., pricing and contract design), and algorithm configuration and parameter tuning. The central challenge is maintaining an efficient data structure as the induced partition grows increasingly complex over time---naively, the number of subregions can grow as $O(t^d)$ by round $t$ in $d$ dimensions. We design a data structure that supports efficient updates and proportional sampling while avoiding the cost of explicitly maintaining this exponential growth, where the discontinuities are structured from axis-parallel hyperplanes. Under a $\sigma$-smoothed adaptive adversary, we prove a tight $O(\sqrt{\sigma T})$ bound on the depth of our data structure, and an $O(\log T)$ bound under a random-order adversary---to our knowledge, the first such results for this class of problems. We apply this framework to online learning with piecewise-structured rewards, obtaining efficient no-regret algorithms under both full-information and bandit feedback, with provable sublinear regret guarantees.
\end{abstract}

\section{Introduction}

Many sequential decision-making problems require a learner to repeatedly select actions from a continuous domain in response to an environment whose feedback is evolving over time. A natural and recurring structure in such settings is \emph{piecewise-continuous} feedback: the reward function is piecewise-Lipschitz within regions of the action space, but changes abruptly across region boundaries. This structure arises organically in a broad range of applications --- in principal-agent settings such as dynamic pricing \citep{DBLP:journals/tmlr/BalcanB24,blum2005near} and contract design \citep{zhu23sample}, where the agent's best response to the learner's action creates sharp discontinuities in the reward function, and in algorithm configuration and parameter tuning \citep{balcan2018dispersion,DBLP:conf/nips/BalcanKST22,DBLP:conf/nips/BalcanS21,gupta2016pac}, where an algorithm's overall performance over instances as a function of its parameters is naturally piecewise-continuous.

Underlying all of these settings is a common computational primitive: \emph{proportional sampling} from a dynamically evolving weight function. At each round, the learner must select an action with probability proportional to the cumulative rewards observed so far --- a distribution that changes as new feedback arrives and the partition of the domain is refined. Proportional sampling is well-studied in static settings \citep{cochran1977sampling, brewer1983sampling, cheung2014}, and plays a central role in randomized algorithms \citep{motwani1995randomized}, online learning \citep{cesa2006prediction}, and importance sampling. However, the dynamic setting we consider is fundamentally different: each round introduces a new piecewise-structured update to the weight function, and as these updates accumulate over $t$ rounds in a $d$-dimensional domain, the number of distinct regions induced by the partition grows as $O(t^d)$. This combinatorial explosion makes naively maintaining the weight function and sampling from it computationally infeasible --- and yet, to our knowledge, the question of how to do this efficiently has received almost no attention. Prior work on online learning with piecewise-structured rewards \citep{balcan2018dispersion} has focused primarily on regret minimization and sample complexity, either assuming access to a sampling oracle or ignoring computational efficiency altogether. The only related result is due to~\citep{cohen2017online}, who give an efficient algorithm for the one-dimensional piecewise-constant case---a special case of our setting. We address this gap directly.

We study the following online setting. At each round $t \in \{1, \ldots, T\}$, an adversary introduces a partition of a $d$-dimensional domain $[0,1]^d$ into regions, each assigned a structured reward function. The cumulative reward function $H_t(x)$ accumulates these rewards over time, and the learner's goal at each round is to sample an action $x_t$ proportionally to a function of $H_t(x)$. The reward functions we consider are piecewise-polynomial with partition boundaries defined by hyperplanes from a fixed set of directions, e.g., axis-parallel hyperplanes. The key difficulty is that as partitions accumulate over $t$ rounds in $d$ dimensions, the number of distinct regions grows as $O(t^d)$, and the learner must sample from the induced distribution in real time without explicitly enumerating all regions.

The central computational challenge is maintaining a representation of $H_t$ that supports efficient sampling as the partition grows. A natural approach is to maintain a data structure that tracks the evolving partition and aggregates rewards across regions. However, this faces two fundamental difficulties. First, each new partition introduced by the adversary interacts with all previous ones: a new hyperplane along one coordinate must be reconciled with all existing boundaries along every other coordinate, triggering a cascade of updates that can touch $O(t^d)$ regions in a single round. Second, even if one could maintain the partition explicitly, sampling from the induced distribution requires computing region-wise cumulative weights
--- a task that becomes increasingly expensive as the partition refines. These difficulties compound in high dimensions, and under a fully adversarial environment, such cascades can be forced to occur at every round, making efficient maintenance provably intractable. This motivates a natural restriction on the adversary, which we discuss next.

We resolve the challenges above by providing a $d$-dimensional hierarchical data structure that maintains the evolving weight function and supports efficient proportional sampling under a $\sigma$-smoothed adversary. We address the potential $O(t^d)$ cascade triggers per round with two key ideas. First, we introduce a \emph{deferred insertion} mechanism that avoids propagating each new boundary through all affected regions immediately, instead recording contributions at higher levels of the hierarchy and recovering them efficiently during sampling. Second, we represent cumulative rewards using a \emph{parametric vector encoding} that allows deferred contributions to be scaled and composed correctly across regions in $O(1)$ time, even for piecewise-linear and more general polynomial rewards. Together, these ideas reduce the per-round cost from $O(t^d)$ to a quantity that depends only on the depth of the data structure --- which, under $\sigma$-smoothness, we show remains sublinear in $T$.
Our main contributions are as follows:
\begin{itemize}
    \item \textbf{Efficient hierarchical data structure for proportional sampling.} We design a $d$-dimensional hierarchical interval-tree data structure that supports efficient updates, cumulative reward queries, and exact proportional sampling under dynamically evolving piecewise-polynomial weight functions. Under a $\sigma$-smoothed adaptive adversary, \hbedit{every} operation runs in $O(d \cdot t^{\frac{d+1}{2}})$ time per round, and $O(d \log^{d+1} t)$ under a random-order oblivious adversary. (\Cref{thm:general_query_runtime})

    \item \textbf{{Smoothed analysis of tree depth and the longest increasing subsequence problem}.} We prove a tight $O(\sqrt{T})$ bound on the depth of our data structure under a $\sigma$-smoothed adaptive adversary, and $O(\log T)$ under a random-order oblivious adversary. To our knowledge, these are the first such results for hierarchical interval-tree structures under smoothed adversaries. {Our analysis introduces and optimally solves the expected length of the ``longest increasing subsequence under smoothed adversary'' problem, which is of independent interest.} (\Cref{thm:balancedness})

    \item \textbf{Efficient online learning.} We apply our framework to online learning with piecewise-structured rewards, obtaining efficient implementations of the Weighted Majority algorithm under full-information feedback and $\textsc{EXP3}$ under bandit feedback while satisfying no-regret for piecewise constant and linear settings. (\Cref{thm:full-info,thm:bandit})
\end{itemize}

\subsection{Overview of Results and Techniques}

\paragraph{Efficient and Scalable Multi-Dimensional Data Structure.}
{Our framework is built on a hierarchical interval-tree structure with one layer per coordinate. In \(d\) dimensions, the construction is recursive: each node at one level maintains an associated interval tree over the next coordinate, yielding a tree-of-trees representation in which \hbedit{\textit{atomic regions}---subregions that do not contain any further subregions, used in contrast to \textit{non-atomic} regions---}are defined by intersections of intervals across coordinates. In the online proportional sampling setting, successive refinements induce up to \(O(t^d)\) atomic regions by round \(t\), making direct maintenance of the full structure infeasible.}
We introduce a deferred insertion mechanism, called \emph{Lazy Insertion}, which avoids cascading each newly introduced boundary through all associated structures as in classical multi-dimensional interval trees~\citep{de2000computational}. This mechanism is enabled by a novel compact parametric vector representation of rewards, which allows updates and insertions to be handled efficiently while preserving scalable propagation across subregions. We then analyze the height and balance of the resulting data structure under adaptive and oblivious \(\sigma\)-smooth adversaries, and use these bounds to derive the efficiency and regret guarantees in our proportional sampling and online learning applications.

\paragraph{Lazy and Regular Associated $\IT$s.}
Our data structure maintains cumulative weight information over all atomic and non-atomic regions induced by the evolving partition of the space. Concretely, at each level it stores an interval tree over the endpoints obtained by projecting region boundaries onto the corresponding coordinate. In a classical insertion scheme, the endpoints induced by each newly formed region must be propagated into all associated structures corresponding to the relevant subregions, so that every node explicitly reflects the current partition.
Our key observation is that, since the data structure is designed to preserve proportional sampling rather than explicitly storing every induced subregion in all levels, many of the classical insertions can be deferred. To achieve this, each node maintains two associated data structures: a regular structure, which is updated as in the classical setting, and a lazy structure, which stores deferred insertions. During an update, newly introduced endpoints are propagated only along the traversal paths needed to locate the affected regions. Whenever the traversal reaches a node whose subtree lies entirely inside the relevant interval, the cascade stops: instead of inserting the remaining endpoints into all descendant associated structures, they are recorded in the lazy structure of that node and later scaled to the underlying subregions during queries. As a result, while classical maintenance may require \(O(t^d)\) insertions per newly induced partition, our lazy mechanism reduces the number of required insertions to \(O\left(\prod_{i\in[d]} \EH(\Tree^{(i)})\right)\), where \(\EH(\Tree^{(i)})\) denotes the expected height of the interval tree at level \(i\), built over the endpoints obtained by projecting boundaries onto the \(i\)-th coordinate. The conceptual idea is illustrated in~\Cref{fig:lazyinsertion} for the two-dimensional case, which provides useful intuition for the underlying architecture.

This approach differs from classical lazy propagation in one-dimensional dynamic trees~\citep{cohen2017online,ibtehaz2021multidimensional}, where segments are inserted explicitly and only their weight updates are deferred. In contrast, \emph{Lazy Insertion} defers the insertions themselves, allowing the affected subtrees to remain implicitly represented while preserving correctness.

\paragraph{Scalable Cumulative Reward Vectors.}
To ensure that the information maintained in the lazy structures remains constantly scalable and can be accurately combined with the regular structures during sampling, we introduce \emph{scalable cumulative reward vectors} in place of scalar cumulative rewards. For an inserted region, each such vector represents the {aggregate reward} (integral of the reward function) over that region, with its entries capturing the contributions of the different monomial terms arising in the integration process. Although summing these entries recovers the usual real-valued reward of the region, the vector representation allows each term to be scaled independently through closed-form expressions. This property is essential beyond piecewise-constant rewards, such as for piecewise-linear and more general structured reward functions, where direct scalar scaling is no longer feasible. In turn, this representation enables the lazy insertion structures to scale maintained reward information down to affected subregions in \(O(1)\) time and to compose it correctly with the explicitly maintained regular information, thereby preserving accurate proportional sampling. We illustrate this mechanism further in Example~\ref{example:scaling} in Section~\ref{sec:model}.

\paragraph{Balanced Interval Trees under $\sigma$-Smoothed Adversaries.}
Having reduced the insertion cost per region from \(O(t^d)\) to a traversal-depth term, the overall efficiency is governed by the expected height of the interval tree at each level. We therefore analyze this height under \(\sigma\)-smoothed adversaries.
We prove an asymptotically tight \(O(\sqrt{\sigma T})\) upper bound under both adaptive and oblivious adversaries. Thus, for the tree growth, adaptivity is no more powerful than obliviousness within the same sample budget. This contrasts with~\citep{haghtalab2024smoothed}, where adaptive adversaries are approximated by oblivious i.i.d.\ uniform samples using more draws. In our setting, {merely} the {arrival order of the endpoints} determines the height, and oblivious sequences already match the adaptive \(O(\sqrt{\sigma T})\) worst case. Under the more restrictive random-order oblivious adversary, the {expected} height further improves to \(O(\log T)\). This analysis implies efficiency of operations without the need to rebalance the tree.

\paragraph{Time Complexity.}
Our data structure supports online proportional sampling efficiently, even though the number of atomic regions can grow as \(O(t^d)\) by round \(t\). In particular, any data structure operation including sampling an action at round \(t\) from the current proportional distribution, can be implemented in time $O\!\left(d \prod_{i=0}^d \EH(\Tree^{(i)}) \right)$, where under an adaptive \(\sigma\)-smoothed adversary this becomes \(O(d\, t^{\frac{d+1}{2}})\), while under an oblivious \(\sigma\)-smoothed adversary it becomes \(O(d \log^{d+1} t)\).
The running time captures both the hierarchical traversal of the multi-dimensional interval tree and the lazy insertion mechanism used to maintain scalable cumulative reward information.

\paragraph{Data Structure and Efficient Regret Analysis under Full Information and Bandit Feedback.}
Our data structure also extends to adversarial online learning under both \emph{full-information} and \emph{bandit} feedback. We analyze the regret of the efficient \(\Exp\) and $\Band$ algorithm under both \emph{full-information} and \emph{bandit-feedback} settings in a $d$-dimensional $\AS$, where the adversary selects piecewise constant or linear reward functions. Our analysis extends the results of~\citep{cohen2017online} and aligns with the framework of~\citep{balcan2018dispersion}, demonstrating that the regret remains sublinear in time.

\paragraph{Organization.}
{In \Cref{sec:model}, we formalize our model of proportional sampling with dynamically evolving, piecewise-continuous weight functions.
\Cref{sec:que} introduces our framework for an efficient and scalable data structure.
\Cref{sec:ds} presents efficient implementation of our algorithm for proportional sampling. 
We then extend our framework to online learning in \Cref{sec:learning}, where we describe efficient algorithms for both full-information and bandit feedback settings. \Cref{sec:rel_work} provides extensive literature review. \Cref{section:quereis,section:correctness,section:generalpar,section:complexity} provide the details of our data structure’s queries, their correctness proofs, and runtime analysis. Finally, \Cref{sec:learningapp,app:bandit} describe our online learning algorithms and establish their regret guarantees.}
\section{Related Work}
\label{sec:rel_work}

Proportional sampling arises across survey sampling, theoretical computer 
science, and online learning, but has been studied almost exclusively in 
discrete or geometrically static settings. In statistics, probability 
proportional to size (PPS) sampling has a long history~\citep{cochran1977sampling,
brewer1983sampling,cheung2014}; in randomized algorithms and sketching it 
underlies sublinear methods~\citep{alon1999space,motwani1995randomized,
cormode2011synopses}; and in randomized linear algebra it drives core sampling 
schemes~\citep{drineas2006fast}. Our setting generalizes these to continuous 
domains with adversarially evolving geometry with piecewise structure, a regime 
none of these frameworks are designed to handle. Efficient proportional sampling 
in such setups bridges the gap between theory and application by making these 
algorithms practical in realistic, evolving environments.

Another important application of proportional sampling is in online learning, 
where it underlies foundational algorithms such as the exponentially-weighted 
forecaster (EWF)~\citep{cesa2006prediction}, Follow the Perturbed 
Leader~\citep{kalai2005efficient}, and EXP3~\citep{auer2002nonstochastic}. 
The proportional sampling underlying these algorithms typically operates over 
finite or convex action spaces with scalar-valued feedback. A separate line of 
work has studied online learning in settings with piecewise structure, 
particularly in applications such as pricing, contract design, and security 
games~\citep{blum2005near,balcan2018dispersion,dick2020semi,balcan2021learning,
DBLP:journals/tmlr/BalcanB24,balcan2024accelerating,
balcan2024algorithmconfigurationstructuredpfaffian}, focusing on regret 
guarantees and sample complexity under structural assumptions, but without 
addressing efficient sampling from evolving feedback functions. Across all these 
settings, proportional sampling is not merely a statistical tool but a 
computational primitive whose efficiency in highly evolving, piecewise structured 
domains determines the practicality of these algorithms, and bridging this gap 
between theory and application is the central challenge our work addresses.

To support efficient proportional sampling in such settings, we build on and 
extend the classical data structure literature. Classical data structures such 
as balanced binary trees, range trees, and quad-trees provide efficient support 
for operations over structured domains and spatial 
decompositions~\citep{cormen2022introduction,de2000computational}. These 
structures are well-suited to settings where the underlying partition is fixed 
and aggregates are scalar-valued---supporting value updates over predetermined 
subregions but not structural change. In contrast, our setting couples dynamic 
geometric refinement with vector-valued maintenance and continuous sampling: 
adversarially chosen hyperplanes arrive over time, progressively refining the 
partition and restructuring which subregions exist. We maintain vector-valued 
statistics that allow information about each region to be efficiently scaled 
down to newly created subregions, enabling tractable maintenance as the geometry 
evolves, and support proportional sampling from distributions defined by 
cumulative, piecewise polynomial functions.

Designing data structures for online learning has also been studied in the 
context of fixed hierarchical partitions~\citep{cesa2017algorithmic,bubeck2011x}, 
where chaining and recursive partitioning of the context space are used to 
construct structured sets of experts and obtain improved regret bounds, with 
a focus on controlling the complexity of the expert class. Our setting differs 
fundamentally: the partition is not fixed but evolves as adversarially chosen 
piecewise-structured rewards arrive over time, and the central challenge is 
maintaining efficient proportional sampling over this dynamic geometry rather 
than controlling expert class complexity. Our framework yields regret guarantees 
under both full-information and bandit feedback, but its primary contribution 
is the design of scalable sampling mechanisms that make these guarantees 
achievable in practice.

Our work connects to the smoothed analysis 
framework~\citep{spielman2001smoothed,gupta2016pac}, which provides the 
necessary conditions under which learning remains tractable despite 
discontinuities. The most directly related work is due to~\cite{cohen2017online}, 
who proposed an efficient algorithm for one-dimensional learning with fixed, 
piecewise-constant feedback against a $\sigma$-smoothed adversary; we build 
on and significantly generalize their framework to accommodate 
high-dimensional domains, dynamically evolving piecewise structure, and more 
expressive piecewise polynomial functions.

Our analysis of hierarchical interval-tree height under $\sigma$-smoothed 
adversaries is the first of its kind. While prior smoothing 
results~\citep{haghtalab2024smoothed} approximate the \emph{set} of 
adversarially generated endpoints by i.i.d.\ samples, they do not capture 
the \emph{order of arrival}, which is crucial for tree growth in our setting. 
This distinction leads to our tight $O(\sqrt{\sigma T})$ bound under adaptive 
$\sigma$-smoothed adversaries, in contrast to the $O(\log T)$ behavior that 
holds under i.i.d.\ arrivals~\citep{pittel1984growing}.

Finally, recent work on structured learning in multi-dimensional settings has 
begun to address computational concerns~\citep{DBLP:conf/iclr/BalcanDL20,
balcan2024accelerating}, but falls short of the logarithmic-in-time complexity 
required for efficient online sampling in dynamic environments. Our work is the 
first to design a data structure that supports exact proportional sampling over 
a continuously partitioned domain with time-varying piecewise polynomial rewards, 
while providing regret guarantees under both full-information and bandit feedback, 
closing the gap between the theoretical foundations laid by prior work and 
the computational demands of realistic, evolving environments.
\section{Model: Adversarial Online Proportional Sampling}\label{sec:model}

We consider a proportional sampling problem over the domain \([0,1]^d\), where \(d \in \mathbbm{N}\). At each round \(t\), the learner samples a point \(x_t \in [0,1]^d\), while a piecewise-structured reward function \(\RRF_t\) is adversarially defined as follows:

\begin{enumerate}

    \item \textbf{Action Space Partitioning:}
    At each round \(t\), the domain \([0,1]^d\) is partitioned, as in Figure~\ref{fig:2dft}, into disjoint regions \(\Rg_t = \{ \rg_{1,t}, \rg_{2,t}, \dots \}\) by a set of \(k \in \mathbb{N}\) hyperplanes \(\Hp_t = \{ \hp_{1,t}, \dots, \hp_{k,t} \}\) chosen by an adversary. Each hyperplane is parallel to some direction \(\ax_i\) in a fixed set \(\Ax = \{\ax_1, \dots, \ax_m\}\) of pairwise non-parallel reference directions. Under the \(\sigma\)-smooth model, for each direction \(\ax_i\), hyperplanes parallel to \(\ax_i\) are drawn independently from an adversarially chosen \(\sigma\)-smooth distribution \(\Dist^{(i)}_t\) over \([0,1]\), with independence across directions, {where a distribution is $\sigma$-smooth if its probability density function is bounded above by $\sigma \geq 1$}. Smoothness is imposed separately within each directional group. \hbedit{We present the partitioning in full generality; however, for simplicity, we present our technical operations for the axis-parallel case.}

    \item \textbf{Reward Assignment:}
    A piecewise-structured reward function \(\RRF_t : [0,1]^d \to \mathbb{R}_{\geq 0}\) is assigned over the partition \(\Rg_t\). Specifically, each region \(\rg_{j,t} \in \Rg_t\) is associated with a non-negative reward function \(h_{j,t}\), and \(h_t(x)=h_{j,t}(x)\) for all \(x \in \rg_{j,t}\). We assume that all region-wise reward functions share a common functional form (e.g., a fixed polynomial), while their parameters—namely, the coefficients of the corresponding monomials—may vary across regions and rounds. Thus, the reward function differs across regions only through its parameters, not its structure. We define the cumulative reward function as
\[
H_t(x)\doteq\sum_{t'=1}^t \RRF_{t'}(x).
\]
\end{enumerate}

We study the problem of efficiently sampling \(x_t\) from a distribution proportional to a transformation \(\zeta: \mathbb{R} \to \mathbb{R}_{\geq 0}\) of the cumulative reward:
\begin{align}
    x_t \sim \frac{\zeta\!\left(H_t(x)\right)}{\int_{x \in [0,1]^d} \zeta\!\left(H_t(x)\right)}.
    \label{eq:proportional}
\end{align}
Two canonical choices of \(\zeta\) recover standard sampling schemes: 
When \(\zeta(H_t(x))=H_t(x)\), this reduces to proportional sampling, i.e.,
    \(
    x_t \sim \frac{H_t(x)}{\int_{x \in [0,1]^d} H_t(x)}.
    \) 
When \(\zeta(H_t(x))=\exp(\eta H_t(x))\) for \(\eta \in (0,1)\), we obtain exponential-weighted sampling:
    \(
    x_t \sim \frac{\exp(\eta H_t(x))}{\int_{x \in [0,1]^d} \exp(\eta H_t(x))}.
    \)
The first case naturally extends to polynomial choices of \(\zeta\), while the second case connects this sampling framework to exponential-weight methods in adversarial online learning.

\paragraph{Smoothed Adversarial Models.}
The sequence of partitions is governed by the hyperplanes \(\{\Hp_t\}_{t\ge1}\).
At round \(t\), the adversary selects distributions \(\Dist^{(\cdot)}_t\) corresponding to directions $\Ax$, where the hyperplanes in \(\Hp_t\) are sampled from. The adversary is \emph{adaptive} if the selection of the distribution depends on previously sampled hyperplanes, based on the past hyperplane sets \(\Hp_1,\dots,\Hp_{t-1}\), and \emph{oblivious} if the selection is fixed a priori. We also consider \emph{random-order oblivious}, in which the adversary first selects $Tk$ distributions, and the sampled hyperplanes arrive in a uniformly random order.
\section{Our Framework: Efficient and Scalable  Data Structure $\Tree$}\label{sec:que}

We now introduce our approach for designing a new multi-dimensional data structure that enables efficient proportional sampling and further supports online learning against a \(\sigma\)-smoothed adversary.

\subsection{Reward Function Parametric Representation.}

{As discussed in the introduction, supporting efficient operations in this rapidly evolving partition requires handling both deferred insertions and downward-scalable auxiliary structures that remain consistent with regular insertions. We capture this through a novel \emph{parametric representation} of the reward. Rather than storing rewards directly, we assign each region \(\rg_{i,t} \in \Rg_t\) a coefficient vector in a fixed reward family, from which the reward is recovered via a known mapping. We establish additivity {of these vectors} across rounds, updates of the form \(H_t = H_{t-1} + h_t\) reduce to simple vector additions. This yields efficient maintenance of cumulative rewards while preserving consistency across regular and deferred operations.}

We reinterpret the piecewise reward function \(\RRF_t\), assumed to be piecewise polynomial, as a composition of two components: a mapping function \(\MF\) and a parameter function \(\fu_t\). We model the piecewise polynomial reward function \(\RRF_t\) as
$\RRF_t(x)=\MF(x,\fu_t(x))$,
where:
\(\MF:\mathbb{R}^d\times\mathbb{R}^{\param}\to\mathbb{R}\) 
    is a known mapping that specifies a fixed family of log-concave 
    polynomial rewards through a fixed set of \(\param\) monomials in 
    \(x\), and \(\fu_t:\Rg_t\to\mathbb{R}^{\param}\) is a parameter function 
    that assigns to each induced region its coefficient vector.

For instance, for $\RRF_t(x) = 3x^{(0)} + 5x^{(1)}$, the monomial 
vector is $\langle x^{(0)}, x^{(1)}\rangle$, $\fu_t = \langle 3, 
5\rangle$, and $\MF(x, \fu_t(x)) = \langle x^{(0)}, x^{(1)}\rangle 
\cdot \fu_t(x)$ recovers $\RRF_t$. Thus the monomial family is fixed 
while the adversary controls region-wise coefficients via reward parameter function $\fu_t$; 
examples for higher-degree polynomials are in~\Cref{examp:decompose}. 
As formalized in~\Cref{sec:aux_param}, this induces a natural additive structure over time for parameter functions as 
$$\Fu_t(x) \doteq \sum_{t' \leq t} \fu_{t'}(x).$$

\paragraph{Remark.}
In both feedback models, $\MF$ is fixed and known; under full 
information the learner observes $\fu_t$ entirely, while under bandit 
feedback only $\fu_t(\rg_{i,t})$ for the region containing $x_t$ is 
revealed.

\paragraph{Reward Aggregation over Regions.}
\begin{definition}
{The aggregate reward of region $\rg$ under $\Fu_t$ is}
\begin{align}
    \CR(\rg) = \int_{x \in \rg} \MF(x, \Fu_t(x))\,dx. \label{eq:crew}
\end{align}
\end{definition}

\subsection{Downward Scalability via Vector Encoding of $\CR(\rg)$.}
We represent the cumulative reward of a region $\rg$ as a vector 
$\vec{\CR}(\rg)$, where each entry corresponds to a monomial component 
of the reward, enabling computation of $\CR(\rg')$ for any subregion 
$\rg' \subseteq \rg$ by scaling each entry independently using only 
the {boundaries of $\rg'$ and $\rg$.}

For example, consider a linear reward of the form
\(
g(x,\Fu_t(x)) = c_1 x^{(0)} + c_2 x^{(1)}.
\)
For a rectangular region \(\rg = [a_1{:}a_2][b_1{:}b_2]\), the cumulative reward decomposes to the entries of \(\vec{\CR}(\rg)\) as
\[
\CR(\rg) = c_1 \int_{\rg} x^{(0)} + c_2 \int_{\rg} x^{(1)}
\xrightarrow{\text{vector representation}}
\vec{\CR}(\rg) = \left\langle c_1 \int_{\rg} x^{(0)},\; c_2 \int_{\rg} x^{(1)} \right\rangle.
\]
To compute the {aggregate} reward over a subregion \(\rg'=[a'_1:a'_2][b'_1:b'_2] \subseteq \rg\), each component is scaled independently using closed-form factors that depend only on the endpoints of \(\rg\) and \(\rg'\), and the results are summed as
\[\CR(\rg') = 
\vec{\CR}^{(0)}(\rg)\cdot \frac{(a_2'^2 - a_1'^2)(b'_2 - b'_1)}{(a_2^2 - a_1^2)(b_2 - b_1)}
+ \vec{\CR}^{(1)}(\rg)\cdot \frac{(a'_2 - a'_1)(b_2'^2 - b_1'^2)}{(a_2 - a_1)(b_2^2 - b_1^2)}.\]
We defer the multivariate construction and scaling procedure to any kind of region to \Cref{section:scaling}.

\subsection{Tree-Based Data Structure \(\Tree\)}
\(\Tree\) is a multi-layered tree-of-trees, one layer per coordinate, 
that maintains cumulative rewards and supports proportional sampling 
over a $d$-dimensional evolving space. We defer implementation details 
to~\Cref{sec:ds} and focus here on the key primitives. At each round 
$t$, \(\Tree\) supports:
\begin{itemize}
    \item \(\Ins(\rg_{i,t})\): insert region \(\rg_{i,t} \in \Rg_t\) 
    induced by discontinuities of \(\fu_t\).
    \item \(\Update(\rg_{i,t})\): update cumulative rewards with 
    \(\fu_t(\rg_{i,t})\) across all affected and newly created regions.
    \item \(\Wei(\rg, t)\): return aggregate reward \(\CR(\rg)\) under 
    \(\Fu_t\) for a query region \(\rg\).
    \item \(\Draw(t)\): sample $x_t \in \AS$ with probability 
    proportional to \(\zeta\!\left(H_t(x)\right)\) at round $t$.
\end{itemize}
Proportional Sampling at each round proceeds as follows:
\begin{enumerate}
    \item \textbf{Sampling:} issue $\Draw(t)$ to $\Tree(t{-}1)$ to 
    proportionally sample $x_t$ w.r.t.\ cumulative parameters $\Fu_{t-1}$.
    \item \textbf{Update:} upon observing $\fu_t$, insert new regions 
    $\rg_{i,t} \in \Rg_t$ and update $\Tree$ with $\fu_t(\rg_{i,t})$.
\end{enumerate}
By the following theorem, the per-round cost of $\Draw(t)$ (similar to other queries in~\Cref{section:complexity}.) decomposes 
recursively across layers, each contributing a factor of 
$\EH(\Tree^{(i)})$, as the height of that layer measured as its longest 
root-to-leaf path, a fundamental concept governing query complexity in 
tree-based structures, whose efficiency under smoothness in our highly 
evolving setting is established below.
\begin{theorem}\label{thm:general_query_runtime}
In $\Tree$ over a $d$-dimensional domain, $\Draw(t)$ runs in
$O\!\left(d \prod_{i=0}^d \mathbbm{E}[\EH(\Tree^{(i)})] \right)$ time,
which becomes $O(d\, t^{\frac{d+1}{2}})$ under adaptive and oblivious  and $O(d \log^{d+1} t)$ under a 
random-order $\sigma$-smoothed adversaries. Where the expectation reflects the adversarial randomness.
\end{theorem}

\paragraph{Balancedness of $\Tree$ under Smoothness.}

A critical bottleneck for efficiency in any tree-based structure is the
height: in the worst case, $\EH(\Tree^{(i)})$ grows as $O(TK)$, and 
classical rebalancing is prohibitively expensive in our tree-of-trees 
design due to $O(t^d)$ cascading updates. Surprisingly, we show that 
smoothness alone prevents any adversary from forcing linear height 
growth, remarkably resolving the need for explicit rebalancing and 
making $\Tree$ provably balanced without it. By the following theorem, we establish tight bounds sublinear in $T$ independent of the adversary's adaptivity; detailed discussion and overview of the techniques are presented 
in~\Cref{sec:detailsbalanced}.
\begin{theorem}\label{thm:balancedness}
Fix a coordinate $i \in [d]$ and let $\mathcal{A}$ be a $\sigma$-smooth 
adversary. Let $x_1, \dots, x_{Tk} \in [0,1]$ be points drawn 
sequentially from distributions $\Dist^{(i)}_1, \dots, \Dist^{(i)}_{Tk}$ 
and inserted into $\Tree^{(i)}$ as partition endpoints along coordinate 
$i$. Then all bounds below are tight:
\[
\mathbb{E}[\EH^{(i)}] =
\begin{cases}
O(\sqrt{\sigma Tk}), & \text{if $\mathcal{A}$ is adaptive or oblivious},\\[3pt]
O(\log(Tk)), & \text{if $\mathcal{A}$ is oblivious and random-order}.
\end{cases}
\]
\end{theorem}
We reduce both adversary types to an $r^*$-averse adversary 
(\Cref{def:pathaverseadv}) maximizing height along a canonical root-to-leaf path, 
identify a monotone optimal strategy, and prove tight bounds via a 
dynamic program for the adaptive case and~\cite{pittel1984growing} 
for the random-order case; full proofs in~\Cref{sec:proofbalanced}.

\paragraph{Connection to Longest Increasing Subsequence.}
Our analysis reveals a surprising tight connection to the longest 
increasing subsequence under a smoothed adversary, a natural and 
previously unstudied problem, for which we establish a tight 
$O(\sqrt{\sigma T})$ bound as a byproduct of our techniques.
\section{Online Proportional Sampling for $\zeta(H_t(x)) = \MF(x,\Fu_t(x))$}\label{sec:ds}
We now introduce \(\Tree\), our data structure for online proportional sampling in this case that maintains aggregate reward vectors $\vec{\CR}(.)$ over regions in the evolving space, and supports efficient and accurate sampling through deferred insertions and downward-scaling. For concreteness and clarity, we present the main ideas in the two-dimensional axis-parallel setting and defer the remaining technical details, along with the full \(d\)-dimensional generalization, to the appendix.
\paragraph{$2$-Dimensional Case with Axis-Parallel Hyperplanes.}\label{sec:2dAPrelax}
We consider a two-dimensional action space with axis-parallel hyperplanes, and focus on the case \(k=2\), where at each round \(t\) one hyperplane parallel to each axis arrives:
\(
\hp_{1,t} = \{x^{(0)} = \bp^{(0)}_t\},
\;
\hp_{2,t} = \{x^{(1)} = \bp^{(1)}_t\}, 
\;
\text{where } \bp^{(0)}_t, \bp^{(1)}_t \sim \Dist^{(0)}_t,\Dist^{(1)}_t.
\) Crucially, the hyperplanes partition each coordinate into intervals, so every region can be represented via its coordinate-wise projections as the alignment of one \emph{horizontal} and one \emph{vertical} interval. For example, \(\rg_{1,t} = [0{:}\bp^{(0)}_t][\bp^{(1)}_t{:}1]\) (\Cref{fig:2dft}).

\begin{wrapfigure}{r}{0.25\textwidth}
    \vspace{-20pt}
    \centering
    \scalebox{0.7}{
    \begin{tikzpicture}
        \definecolor{lightblue}{RGB}{200,220,255}
        \draw[thick, red!50] (0,0) -- (4,0);
        \draw[thick, red!50] (0,4) -- (4,4);
        \draw[thick, lightblue] (0,0) -- (0,4);
        \draw[thick, lightblue] (4,0) -- (4,4);
        \coordinate (ct) at (1,0);
        \coordinate (cpt) at (0,3);
        \coordinate (intersection) at (1,3);
        \draw[dashed] (ct) -- ++(0,4) node[above] {$\hp_{1,t}$};
        \draw[dashed] (cpt) -- ++(4,0) node[right] {$\hp_{2,t}$};
        \fill[black] (intersection) circle (2pt);
        \node at (0.5,3.5) {$\rg_{1,t}$};
        \node at (2.5,3.5) {$\rg_{2,t}$};
        \node at (0.5,1.5) {$\rg_{3,t}$};
        \node at (2.5,1.5) {$\rg_{4,t}$};
        \node[below left] at (0,0) {$0$};
        \node[below] at (4,0) {1};
        \node[left] at (0,4) {1};
    \end{tikzpicture}
    }
    \vspace{-5pt}
    \caption{\(\fu_t\).}
    \vspace{-5pt}
    \label{fig:2dft}
\end{wrapfigure}

\subsection{Description of the Data Structure $\Tree$} \label{section:ds-ps}

We now describe the technical core of $\Tree$, a two-dimensional 
interval-tree structure for maintaining $\Fu_t$. At a high level, 
$\Tree$ is a \emph{tree of trees}: each level corresponds to one 
coordinate, with interval trees tracking projections onto that 
coordinate, where every node represents an interval whose endpoints 
are determined by its children's left and rightmost endpoints. Each node points to interval trees 
over the other coordinate, maintaining the corresponding projections. 
To support deferred insertions, these projections are organized into 
\emph{regular} and \emph{lazy} insertion trees: projections are 
inserted regularly when the corresponding region only intersects the 
node interval, and lazily when it fully contains it. Both tree types 
share the same structure but are maintained and queried differently, 
together capturing all regions without overlap.
\usetikzlibrary{shapes.geometric,positioning}

\definecolor{lightblue}{RGB}{200,220,255}
\definecolor{lightred}{RGB}{255,210,210}

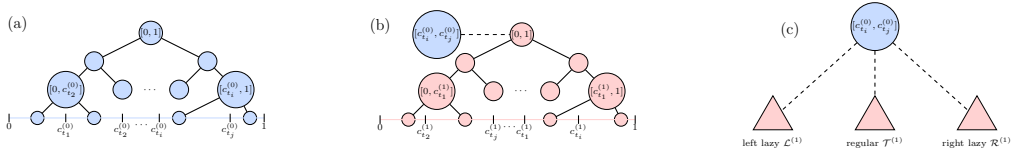
\begin{figure}[h]
    \centering

    \begin{minipage}{0.25\textwidth}
    \centering
    \scalebox{0.5}{
    \begin{tikzpicture}[scale=0.75]
    \node[font=\large] at (0.3,3.3) {(a)};
    \node[draw, circle, fill=lightblue, font=\scriptsize, minimum size=5mm, inner sep=0pt] (root) at (5,3) {$[0,1]$};

    \node[draw, circle, fill=lightblue, minimum size=5mm, inner sep=0pt] (left1) at (3,2) {};
    \node[draw, circle, fill=lightblue, minimum size=5mm, inner sep=0pt] (right1) at (7,2) {};
    \draw[thick] (root) -- (left1);
    \draw[thick] (root) -- (right1);

    \node[draw, circle, fill=lightblue, font=\scriptsize, minimum size=5mm, inner sep=0pt] (left2) at (2,1) {$[0,\bp^{(0)}_{t_2}]$};
    \node[draw, circle, fill=lightblue, minimum size=5mm, inner sep=0pt] (right2) at (4,1) {};
    \draw[thick] (left1) -- (left2);
    \draw[thick] (left1) -- (right2);

    \node[font=\scriptsize] at (5,1) {$\cdots$};
    \node[draw, circle, fill=lightblue, minimum size=5mm, inner sep=0pt] (right4) at (6,1) {};
    \node[draw, circle, fill=lightblue, font=\scriptsize, minimum size=5mm, inner sep=0pt] (right3) at (8,1) {$[\bp^{(0)}_{t_i},1]$};
    \draw[thick] (right1) -- (right4);
    \draw[thick] (right1) -- (right3);

    \draw[thick] (left2) -- (1,0);
    \draw[thick] (left2) -- (3,0);
    \draw[thick] (right3) -- (6,0);
    \draw[thick] (right3) -- (8.5,0);

    \node[draw, circle, fill=lightblue, minimum size=3.5mm, inner sep=0pt] at (1,0) {};
    \node[draw, circle, fill=lightblue, minimum size=3.5mm, inner sep=0pt] at (3,0) {};
    \node[draw, circle, fill=lightblue, minimum size=3.5mm, inner sep=0pt] at (6,0) {};
    \node[draw, circle, fill=lightblue, minimum size=3.5mm, inner sep=0pt] at (8.5,0) {};

    \draw[thick, lightblue] (0,0) -- (9,0);

    \draw[thick] (0,0.15) -- (0,-0.15);
    \node[below, font=\scriptsize] at (0,0) {$0$};

    \draw[thick] (2,0.15) -- (2,-0.15);
    \node[below, font=\scriptsize] at (2,0) {$\bp^{(0)}_{t_1}$};

    \draw[thick] (4,0.15) -- (4,-0.15);
    \node[below, font=\scriptsize] at (4,0) {$\bp^{(0)}_{t_2}$};

    \node[below, font=\scriptsize] at (4.7,0) {$\cdots$};

    \draw[thick] (5.3,0.15) -- (5.3,-0.15);
    \node[below, font=\scriptsize] at (5.3,0) {$\bp^{(0)}_{t_i}$};

    \draw[thick] (7.8,0.15) -- (7.8,-0.15);
    \node[below, font=\scriptsize] at (7.8,0) {$\bp^{(0)}_{t_j}$};

    \draw[thick] (9,0.15) -- (9,-0.15);
    \node[below, font=\scriptsize] at (9,0) {$1$};

    \end{tikzpicture}
    }
    \end{minipage}
    \hspace{0.025\textwidth}
    \begin{minipage}{0.26\textwidth}
    \centering
    \scalebox{0.5}{
    \begin{tikzpicture}[scale=0.75]
    \node[font=\large] at (0,3.3) {(b)};
    \node[draw, circle, fill=lightblue, font=\scriptsize, minimum size=5mm, inner sep=0pt] (rt) at (2,3) {$[\bp^{(0)}_{t_i},\bp^{(0)}_{t_j}]$};
    \node[draw, circle, fill=lightred, font=\scriptsize, minimum size=5mm, inner sep=0pt] (root) at (5,3) {$[0,1]$};
    \draw[dashed] (rt) -- (root);

    \node[draw, circle, fill=lightred, minimum size=5mm, inner sep=0pt] (left1) at (3,2) {};
    \node[draw, circle, fill=lightred, minimum size=5mm, inner sep=0pt] (right1) at (7,2) {};
    \draw[thick] (root) -- (left1);
    \draw[thick] (root) -- (right1);

    \node[draw, circle, fill=lightred, font=\scriptsize, minimum size=5mm, inner sep=0pt] (left2) at (2,1) {$[0,\bp^{(1)}_{t_1}]$};
    \node[draw, circle, fill=lightred, minimum size=5mm, inner sep=0pt] (right2) at (4,1) {};
    \draw[thick] (left1) -- (left2);
    \draw[thick] (left1) -- (right2);

    \node[font=\scriptsize] at (5,1) {$\cdots$};
    \node[draw, circle, fill=lightred, minimum size=5mm, inner sep=0pt] (right4) at (6,1) {};
    \node[draw, circle, fill=lightred, font=\scriptsize, minimum size=5mm, inner sep=0pt] (right3) at (8,1) {$[\bp^{(1)}_{t_1},1]$};
    \draw[thick] (right1) -- (right4);
    \draw[thick] (right1) -- (right3);

    \draw[thick] (left2) -- (1,0);
    \draw[thick] (left2) -- (3,0);
    \draw[thick] (right3) -- (6,0);
    \draw[thick] (right3) -- (8.5,0);

    \node[draw, circle, fill=lightred, minimum size=3.5mm, inner sep=0pt] at (1,0) {};
    \node[draw, circle, fill=lightred, minimum size=3.5mm, inner sep=0pt] at (3,0) {};
    \node[draw, circle, fill=lightred, minimum size=3.5mm, inner sep=0pt] at (6,0) {};
    \node[draw, circle, fill=lightred, minimum size=3.5mm, inner sep=0pt] at (8.5,0) {};

    \draw[thick, lightred] (0,0) -- (9,0);

    \draw[thick] (0,0.15) -- (0,-0.15);
    \node[below, font=\scriptsize] at (0,0) {$0$};

    \draw[thick] (1.6,0.15) -- (1.6,-0.15);
    \node[below, font=\scriptsize] at (1.6,0) {$\bp^{(1)}_{t_2}$};

    \draw[thick] (4,0.15) -- (4,-0.15);
    \node[below, font=\scriptsize] at (4,0) {$\bp^{(1)}_{t_j}$};

    \node[below, font=\scriptsize] at (4.6,0) {$\cdots$};

    \draw[thick] (5.1,0.15) -- (5.1,-0.15);
    \node[below, font=\scriptsize] at (5.1,0) {$\bp^{(1)}_{t_1}$};

    \draw[thick] (7,0.15) -- (7,-0.15);
    \node[below, font=\scriptsize] at (7,0) {$\bp^{(1)}_{t_i}$};

    \draw[thick] (9,0.15) -- (9,-0.15);
    \node[below, font=\scriptsize] at (9,0) {$1$};

    \end{tikzpicture}
    }
    \end{minipage}
    \hspace{0.025\textwidth}
    \begin{minipage}{0.26\textwidth}
    \centering
    \scalebox{0.5}{
    \begin{tikzpicture}[scale=0.9]
    \node[font=\large] at (-2.5,1.6) {(c)};
    \node[draw, circle, fill=lightblue, font=\scriptsize, minimum size=7mm, inner sep=0pt] (v) at (0,1.7)
    {$[\bp^{(0)}_{t_i},\bp^{(0)}_{t_j}]$};

    \node[draw, regular polygon, regular polygon sides=3,
          shape border rotate=0,
          fill=lightred, minimum size=12mm, inner sep=1pt] (lazyL) at (-3,-1) {};

    \node[draw, regular polygon, regular polygon sides=3,
          shape border rotate=0,
          fill=lightred, minimum size=12mm, inner sep=1pt] (reg) at (0,-1) {};

    \node[draw, regular polygon, regular polygon sides=3,
          shape border rotate=0,
          fill=lightred, minimum size=12mm, inner sep=1pt] (lazyR) at (3,-1) {};

    \draw[dashed, thick] (v) -- (lazyL);
    \draw[dashed, thick] (v) -- (reg);
    \draw[dashed, thick] (v) -- (lazyR);

    \node[below=2pt of lazyL, font=\scriptsize] {left lazy $\LZL^{(1)}$};
    \node[below=2pt of reg, font=\scriptsize] {regular $\Tree^{(1)}$};
    \node[below=2pt of lazyR, font=\scriptsize] {right lazy $\LZR^{(1)}$};

    \end{tikzpicture}
    }
    \end{minipage}

    \caption{A $2$-dimensional interval-tree $\Tree$. (a): top-level interval tree $\Tree^{(0)}$ over x-coordinate. (b): second-level interval tree $\Tree^{(1)}$, either regular or lazy, over y-coordinate associated with a node in the top-level. (c): each top-level node points to regular and lazy insertion trees over the other coordinate.}
    \label{fig:tree-composite}
\end{figure}

Here are the components of the data structure in more detail:

\begin{enumerate}

\item \textbf{Top-Level Tree \(\Tree^{(0)}\) (Horizontal Intervals).}
As illustrated in \Cref{fig:tree-composite}(a), the top-level tree organizes the first coordinate of the action space. Its leaves are the atomic horizontal intervals obtained by projecting atomic regions onto this coordinate. Each internal node accordingly represents a horizontal interval as the contiguous union of its children's intervals, with the root corresponding to the full interval \([0,1]\).

\item \textbf{Second-Level Trees (Vertical Intervals).}
As shown in \Cref{fig:tree-composite}(c), each node \(v_0 \in \Tree^{(0)}\) is associated with three interval trees over the second coordinate: one regular insertion tree \(\Tree^{(1)}_{v_0}\) and two lazy insertion trees \(\LZL^{(1)}_{v_0}\) and \(\LZR^{(1)}_{v_0}\). These trees have the same structure as the top-level tree, but over intervals induced by projecting regions onto the second coordinate (\Cref{fig:tree-composite}(b)).

\usetikzlibrary{shapes.geometric, positioning, backgrounds, arrows.meta, decorations.pathmorphing, decorations.markings}

\definecolor{lightblue}{RGB}{200,220,255}
\definecolor{specialblue}{RGB}{120,175,255}
\definecolor{darkblue}{RGB}{60,90,145}
\definecolor{lightred}{RGB}{190,95,95}
\definecolor{lightgreen}{RGB}{220,240,210}

\begin{figure}[h]
\centering
\scalebox{0.45}{
\begin{tikzpicture}[xscale=1.9, scale=0.9, >=Stealth, every node/.style={font=\small}]

\begin{scope}[on background layer]
    \fill[lightgreen, opacity=0.75] (1.1,1.05) rectangle (11.3,2.35);
    \node[left, lightred, font=\large] at (1.0,2.35) {$\bp^{(1)}_2$};
    \node[left, lightred, font=\large] at (1.0,1.05) {$\bp^{(1)}_1$};

    \fill[darkblue!45, opacity=0.75] (5.85,0.2) -- (7.25,3.15) -- (8.95,0.2) -- cycle;
\end{scope}

\draw[thick, darkblue] (0.2,0.2) -- (12.0,0.2);

\foreach \x in {1.1,1.4,3.15,4.2,5.85,8.15,8.95,10.25,11.3}{
    \draw[thick, darkblue] (\x,0.38) -- (\x,0.02);
}

\node[below, font=\large] at (1.1,0.02) {$\bp^{(0)}_1$};
\node[below, font=\large] at (5.85,0.02) {$\bp^{(0)}_i$};
\node[below, font=\large] at (8.95,0.02) {$\bp^{(0)}_j$};
\node[below, font=\large] at (11.3,0.02) {$\bp^{(0)}_2$};

\node[draw, circle, fill=lightblue, minimum size=6.5mm] (v) at (7.9,4.35) {$v_0$};
\node[draw, circle, fill=lightblue, minimum size=6.5mm] (u) at (7.25,3.15) {$u_0$};

\node[draw, circle, fill=specialblue, minimum size=4mm] (topR) at (6.9,5.1) {};
\node[draw, circle, fill=lightblue, minimum size=3.5mm] (topSmall) at (7.0,6.6) {};
\node[draw, circle, fill=lightblue, minimum size=4mm] (root) at (5.9,7.2) {};

\node[draw, circle, fill=lightblue, minimum size=4mm] (leftA) at (5.5,4.55) {};
\node[draw, circle, fill=lightblue, minimum size=4mm] (leftB) at (4.35,3.35) {};
\node[draw, circle, fill=lightblue, minimum size=4mm] (leftC) at (2.65,2.35) {$z_0$};

\node[draw, circle, fill=lightblue, minimum size=4mm] (rightA) at (8.75,3.55) {};
\node[draw, circle, fill=lightblue, minimum size=4mm] (rightB) at (9.8,2.55) {};
\node[draw, circle, fill=lightblue, minimum size=4mm] (rightC) at (10.7,1.6) {};

\draw[
    thick,
    dashed,
    decorate,
    decoration={snake, amplitude=0.8mm, segment length=4mm},
    darkblue
]
(root) -- (topSmall);

\draw[thick, darkblue] (topSmall) -- (10.5,4.5);
\draw[thick, darkblue] (root) -- (1.85,3.85);

\draw[
    thick,
    dashed,
    decorate,
    decoration={snake, amplitude=0.8mm, segment length=4mm},
    darkblue
]
(topSmall) -- (topR);

\draw[thick, darkblue] (topR) -- (v);
\draw[thick, darkblue] (topR) -- (leftA);

\draw[
    thick, darkblue,
    postaction={decorate},
    decoration={
        markings,
        mark=between positions 0 and 1 step 0.06 with {
            \draw[darkblue!60] (0,-1.5pt) -- (0,1.5pt);
        }
    }
]
(topR) -- (leftA) -- (leftB) -- (leftC) -- (1.1,0.2);

\draw[
    thick, darkblue,
    postaction={decorate},
    decoration={
        markings,
        mark=between positions 0 and 1 step 0.06 with {
            \draw[darkblue!60] (0,-1.5pt) -- (0,1.5pt);
        }
    }
]
(topR) -- (v) -- (rightA) -- (rightB) -- (rightC) -- (11.3,0.2);

\draw[thick, darkblue] (v) -- (u);
\draw[thick, darkblue] (u) -- (5.85,0.2);
\draw[thick, darkblue] (u) -- (8.95,0.2);

\newcommand{\redtriangle}[7]{%
\begin{scope}[shift={#1}, scale=#2]
    \filldraw[draw=lightred, fill=lightred!35, thick]
        (-0.9,-0.55) -- (0.9,-0.55) -- (0,1.35) -- cycle;

    \draw[thick, lightred] (-0.9,-0.55) -- (0.9,-0.55);

    \node[lightred, font=\bfseries] at #4 {#3};

    \node[red, font=\bfseries\scriptsize] at (0,0.1) {$\ \vec{\CR}_{#6,#7}$};

    #5
\end{scope}
}

\redtriangle{(3.3,1.3)}{0.40}{$\LZR^{(1)}_{z_0}$}{(0,2.05)}{
    \draw[thick, lightred] (-0.45,-0.70) -- (-0.45,-0.40);
    \draw[thick, lightred] (0.45,-0.70) -- (0.45,-0.40);
    \node[below, lightred] at (-0.45,-0.70) {$\bp^{(1)}_1$};
    \node[below, lightred] at (0.45,-0.70) {$\bp^{(1)}_2$};
}{z_0}{w_1}
\draw[->, thick, darkblue, bend right=10] (leftC) to (3.10,1.75);

\redtriangle{(6.35,3.8)}{0.40}{$\LZL^{(1)}_{v_0}$}{(0,2.15)}{
    \draw[thick, lightred] (-0.45,-0.70) -- (-0.45,-0.40);
    \draw[thick, lightred] (0.45,-0.70) -- (0.45,-0.40);
    \node[below, lightred] at (-0.45,-0.70) {$\bp^{(1)}_1$};
    \node[below, lightred] at (0.45,-0.70) {$\bp^{(1)}_2$};
}{v_0}{w_1}
\draw[->, thick, darkblue, bend right=20] (v) to (6.7,4.40);

\redtriangle{(9.7,3.82)}{0.40}{$\Tree^{(1)}_{v_0}$}{(0,2.15)}{
    \draw[thick, lightred] (-0.45,-0.70) -- (-0.45,-0.40);
    \draw[thick, lightred] (0.45,-0.70) -- (0.45,-0.40);
    \node[below, lightred] at (-0.45,-0.70) {$\bp^{(1)}_1$};
    \node[below, lightred] at (0.45,-0.70) {$\bp^{(1)}_2$};
}{v_0}{v_1}
\draw[->, thick, darkblue, bend left=15] (v) to (9.4,4.50);

\redtriangle{(5.35,1.4)}{0.45}{$\LZL^{(1)}_{u_0}$}{(0,2.1)}{
    \node[lightred] at (-0.45,-0.55) {$\times$};
    \node[lightred] at (0.45,-0.55) {$\times$};
    \node[below, lightred] at (-0.45,-0.7) {$\bp^{(1)}_1$};
    \node[below, lightred] at (0.45,-0.7) {$\bp^{(1)}_2$};
}{u_0}{w_1}
\draw[->, thick, darkblue, bend right=15] (u) to (5.6,2.1);

\redtriangle{(9,1.8)}{0.45}{$\Tree^{(1)}_{u_0}$}{(0,2.1)}{
    \node[lightred] at (-0.45,-0.55) {$\times$};
    \node[lightred] at (0.45,-0.55) {$\times$};
    \node[below, lightred] at (-0.45,-0.7) {$\bp^{(1)}_1$};
    \node[below, lightred] at (0.45,-0.7) {$\bp^{(1)}_2$};
}{u_0}{v_1}
\draw[->, thick, darkblue, bend left=10] (u) to (8.55,2.70);

\end{tikzpicture}
}
\caption{Lazy and Regular insertion of the green region
\(\rg=[\bp^{(0)}_1:\bp^{(0)}_2][\bp^{(1)}_1:\bp^{(1)}_2]\).}
\label{fig:lazyinsertion}
\end{figure}
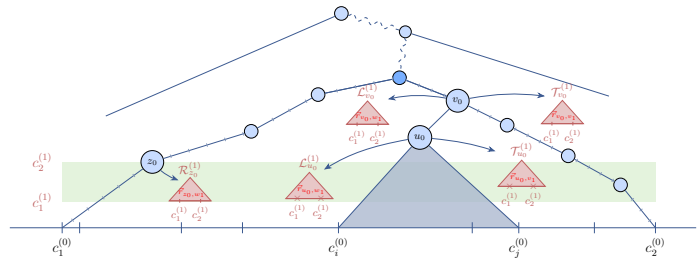

Formally, as illustrated in~\Cref{fig:lazyinsertion}, for a region
\(\rg=[\bp^{(0)}_1:\bp^{(0)}_2][\bp^{(1)}_1:\bp^{(1)}_2]\) in $\fu_t$, its vertical projection
\([\bp^{(1)}_1:\bp^{(1)}_2]\) is inserted as follows:
\begin{itemize}
    \item Regular \(\Tree^{(1)}_{v_0}\): if \(\RA(v_0)\) intersects but is not strictly contained in \([\bp^{(0)}_1,\bp^{(0)}_2]\) (equivalently, contains either \(\bp^{(0)}_1\) or \(\bp^{(0)}_2\)), the projection is regularly inserted at \(v_0\) (e.g., \(v_0\) and \(z_0\), but not \(u_0\) in \Cref{fig:lazyinsertion});
    \item Lazy \(\LZL^{(1)}_{v_0}\) / \(\LZR^{(1)}_{v_0}\): if \(\RA(\LC(v_0))\) or \(\RA(\RC(v_0))\) is contained in \([\bp^{(0)}_1:\bp^{(0)}_2]\), the projection is recorded lazily in the corresponding left or right lazy tree, avoiding insertion into that subtree (e.g., \(v_0\) and \(z_0\) in \Cref{fig:lazyinsertion}).
\end{itemize}
During queries, deferred contributions stored in lazy structures are recovered via scaling. For example, in~\Cref{fig:lazyinsertion}, the vertical interval \([\bp^{(1)}_1:\bp^{(1)}_2]\) of the region $\rg$ is not inserted into \(\Tree^{(1)}_{u_0}\), \(\LZL^{(1)}_{u_0}\), or \(\LZR^{(1)}_{u_0}\). Instead, the information stored in \(\LZL^{(1)}_{v_0}\) is scaled and applied to regions whose horizontal projections lie within \(\RA(u_0)\).

\item
\textbf{Aggregate Reward Vectors.}
As mentioned, we maintain a \hbedit{aggregate} reward vector \(\vec{\CR}(\rg)\) for each region in the sampling structure. These vectors are stored at the nodes of the second-level trees and are indexed by the chain of nodes followed from the top-level tree, where \emph{\(v\)-nodes} correspond to regular insertions and \emph{\(w\)-nodes} to lazy insertions (e.g., \(\vec{\CR}_{v_0,v_1}\) and \(\vec{\CR}_{v_0,w_1}\) in the \Cref{fig:lazyinsertion}). This indexing cleanly separates their contributions: regular vectors are applied directly, while lazy ones are recovered via scaling. 

For instance, in two dimensions, a node \(v_0 \in \Tree^{(0)}\) together with \(v_1 \in \Tree^{(1)}_{v_0}\) defines a region \(\rg=[\RA(v_0)][\RA(v_1)]\), and \(\vec{\CR}_{v_0,v_1}\) stores the \hbedit{aggregate} reward from regular insertions, which are applied directly and are not scaled to subregions. In contrast, \(\vec{\CR}_{v_0,w_1}\) corresponds to a lazy insertion and is scalable along the first coordinate. More generally, \(\vec{\CR}_{v_0,\ldots,w_j,\ldots}\) represents the region \([\RA(v_0)]\cdots[\RA(w_j)]\cdots\), with a lazy component at level \(j\) is scalable along dimension \(j-1\). This notation extends naturally to higher dimensions.

\item \textbf{Auxiliary Trees \(\ATree^{(0)}\) and \(\ATree^{(1)}\).}
The auxiliary trees index all horizontal and vertical endpoints introduced by the adversary up to round \(t\). They do not store reward information; they are standard interval trees used only to quickly identify the atomic regions needed by the drawing procedure.

\end{enumerate}

\subsection{Query Handling, Structural Properties, and Time 
Complexity}
\label{section:ds-qhandling}
To obtain the efficiency and correctness of $\Tree$, matching the 
$\mathbb{E}\!\left[\prod_{i\in[d]}\EH(\Tree^{(i)})\right]$ complexity 
of \Cref{thm:general_query_runtime}, all query handling is built on 
two primitives applied recursively along each coordinate over the 
projection of regions onto that coordinate. For a region $\rg$ with $[\bp^{(i)}_1:\bp^{(i)}_2]$ as its projection 
along coordinate $i$ onto a given level-$i$ interval tree (either 
lazy or regular), the primitives are:
\begin{itemize}
    \item \textbf{Locate} leaves $l_i$ and $u_i$ as atomic intervals 
    with $\bp^{(i)}_1, \bp^{(i)}_2 \in \RA(l_i), \RA(u_i)$.
    \item \textbf{Traverse} from $l_i$ and $u_i$ upward along paths 
    $p_{l_i}$ and $p_{u_i}$, meeting at the first common ancestor 
    $v_{l_i,u_i}$.
\end{itemize}

\paragraph{Handling \(\Ins(\rg)\) and \(\Update(\rg)\).}
Both queries proceed by recursively applying 
locate-and-traverse across all layers of $\Tree$. At each level, the 
paths $p_{l_i}$ and $p_{u_i}$ determine which next-level structures 
receive the projection: regular insertions on $p_{l_i}\cup p_{u_i}$, 
and lazy insertions before the paths merge, right-lazy on 
$p_{l_i}\setminus p_{u_i}$ and left-lazy on $p_{u_i}\setminus p_{l_i}$. 
While insertion decides regular or lazy placement, the update step 
ensures arithmetic soundness of the cumulative reward vectors: regular updates integrate over $\rg'$, a subregion whose $i-1$-th 
projection intersects that of $\rg$, while lazy updates apply over $\rg''$, a subregion whose $i-1$-th projection is uniformly
covered by that of $\rg$, preserving downward scalability. In two dimensions these take the form
\[
\vec{\CR}_{v_0,v_1} \gets \vec{\CR}_{v_0,v_1} + \vec{\SC}_{\rg'}, 
\qquad 
\vec{\CR}_{v_0,w_1} \gets \vec{\CR}_{v_0,w_1} + \vec{\SC}_{\rg''},
\]
where $\rg' = [\RA(v_0)\cap[\bp^{(0)}_1,\bp^{(0)}_2]][.]$ and $\rg'' = [\RA(v_0)][.]$, with $\vec{\SC}_{\rg'}$ and $\vec{\SC}_{\rg''}$ computed as the 
vectorized aggregate reward via~\eqref{eq:crew} with $\fu_t(\rg)$ 
as the coefficient vector.

\paragraph{Handling \(\Draw(t)\).}
Has two sampling stages:

(i) atomic region $\rg\sim\dfrac{\int_{x \in \rg} \MF(x, \Fu_t(x))dx}{\int_{x' \in [0,1]^d} \MF(x', \Fu_t(x'))dx}$, 
(ii) sample $x_t \sim\dfrac{\MF(x_t, \Fu_t(x_t))dx}{\int_{x \in \rg} \MF(x, \Fu_t(x))dx}$ within $\rg$.

\paragraph{Stage 1: Atomic region sampling.}
An atomic region is sampled by sequentially selecting an atomic 
interval along each of the $d$ coordinates via root-to-leaf traversals 
of $\ATree^{(0)}, \dots, \ATree^{(d-1)}$, with left/right moves at 
each node taken with probability computed by $\Wei(\cdot)$ conditioned 
on intervals fixed along prior coordinates. In two dimensions, 
\emph{horizontal sampling} traverses $\ATree^{(0)}$ to select 
$\RA(v_0)$, followed by \emph{vertical sampling} of $\ATree^{(1)}$ 
conditioned on $\RA(v_0)$, together forming atomic region 
$[\RA(v_0)][\RA(v_1)]$ proportionally to $\CR([\RA(v_0)][\RA(v_1)])$. Computing $\Wei(\cdot)$ is highly non-trivial: it constructs set of interval trees 
$\mathcal{S}^{(1)}_{v_0}$ via recursive locate-and-traverse by 
gathering $\Tree^{(1)}_{v_0}$ with $\LZL^{(1)}_{v_0'}$ or 
$\LZR^{(1)}_{v_0'}$ for each $v_0' \in p_{v_0}$ and scaling lazy 
contributions, since atomic endpoints are distributed across structures 
in $\mathcal{S}^{(1)}_{v_0}$ and require boundary marginal pruning at 
each; completeness, endpoint-uniqueness, and downward scalability of 
lazy reward vectors together ensure correct aggregation without 
double-counting.

\paragraph{Stage 2: Sampling within $\rg$.} Given the sampled atomic region $\rg$, we sample $x_t \in \rg$ 
proportionally to $\MF(x, \Fu_t(x))$ via uniform sampling for 
piecewise-constant rewards, and Hit-and-Run (\Cref{them:vempa}) for 
more complex piecewise log-concave rewards such as high-degree 
polynomial and $\exp$ function.

\section{Efficient Adversarial Online Learning \(\zeta(H_t(x)) = \exp(\eta\MF(x,\Fu_t(x)))\)}
\label{sec:learning}

We now show how \(\Tree\) serves as the computational backbone for 
$\Exp$ and \(\Band\) (\Cref{alg:exp3} and~\ref{alg:band}), 
adapting our data structure to support efficient per-round sampling 
under \(\zeta(x) = \exp(\eta x)\) in our highly evolving setting. At 
each round $t$, the adversary selects hyperplanes $\Hp_t$ and defines 
a piecewise-structured reward $\RRF_t$ over the induced partition 
$\Rg_t$ of $[0,1]^d$, and the regret of policy \(\pi\) is
\refstepcounter{equation}
\noindent
$\operatorname{Regret}(\pi, T) \doteq 
\max_{x \in [0,1]^d} \sum_{t=1}^T \RRF_t(x) 
- \Expect[\sum_{t=1}^T \RRF_t(x_t)].$
\hfill(\theequation)

We consider \emph{full-information} and \emph{bandit} feedback, 
focusing on piecewise-constant and piecewise-linear rewards $h_t$
respectively---a scope dictated by the non-additivity of exponentially 
weighted cumulative rewards under \(\zeta(x)=\exp(\eta x)\), which 
breaks the lazy scaling mechanism and renders both exact integration 
and regret bounding intractable beyond degree one. Each feedback model 
therefore requires a tailored set of assumptions:
\begin{itemize}
    \item \emph{Full-information:} piecewise-constant rewards; adversary 
    $\sigma$-smoothed along one direction,
    \item \emph{Bandit:} piecewise-linear rewards; general 
    $\sigma$-smoothed adversary of \Cref{sec:model} 
\end{itemize}
The following theorem reflects the implication of $\Tree$ for online 
learning under full-information feedback, where the setup restriction 
stems from the computational inefficiency of scaling arising from the 
arithmetic challenges of handling $\zeta(x) = \exp(\eta x)$ 
(reflected in Example~\ref{examp:full}); the full setup, challenges, 
and proof are presented in~\Cref{app:full-info}.
\begin{theorem}[Full-Information Regret and Time Complexity]
\label{thm:full-info}
Consider a $d$-dimensional domain with axis-parallel $\Ax$  where at each round, hyperplanes 
parallel to $\ax_1 \in \Ax$ are drawn from $\Dist^{(1)}_t$ and 
hyperplanes parallel to each $\ax_i \in \Ax \setminus \{\ax_1\}$ are 
selected from a fixed set $\mathcal{M}^{(i)}$ of at most $M$ 
hyperplanes. The regret of $\Exp$ for piecewise-constant rewards is $O\!\left(\sqrt{T\left(\log\!\left(k^2T^3\sigma\right) + 
(d-1)\log M + 1\right)}\right)$, with per-round time 
complexity $O(dM^{d-1}\sqrt{\sigma Tk})$ against adaptive/oblivious 
adversaries and $O(dM^{d-1}\log(Tk))$ against a random-order 
adversary.
\end{theorem}
Since new endpoints arrive only along $\ax_1$, the height of 
$\Tree^{(0)}$ follows \Cref{thm:balancedness}, giving 
$O(\sqrt{\sigma KT})$ under adaptive and oblivious adversaries and 
$O(\log KT)$ under random-order, while any $i$-th level tree 
$\Tree^{(i)}_{v_{i-1}}$, $\LZL^{(i)}_{v_{i-1}}$, or 
$\LZR^{(i)}_{v_{i-1}}$ for $i \in \{1, \dots, d-1\}$ is initialized 
once over $\mathcal{M}^{(i)}$ and remains of height $O(\log M)$ 
throughout. The fixed $M$ endpoints along each remaining coordinate 
enable efficient leaf-to-leaf scaling across all lazy structures in 
$\mathcal{S}^{(1)}_{v_0}$, bounding the per-round cost. 
The regret bound follows from the standard $\Exp$ analysis applied 
to the piecewise-constant setting.

The following theorem establishes the regret and time complexity of 
$\Band$ under the general $\sigma$-smooth model with a fixed grid 
discretization; full details, with regret and time complexity analysis 
are in~\Cref{app:bandit}. 
\begin{theorem}[Bandit Regret and Time Complexity]
\label{thm:bandit}
Consider a $d$-dimensional domain where the adversary follows the 
general $\sigma$-smooth model of \Cref{sec:model} and the action space 
is discretized into a fixed grid of $(1/\mu)^d$ cells prior to the 
game. The regret of $\Band$ for piecewise-linear rewards is 
$O\!\left(\operatorname{poly}(d,k,\sigma,\log T)\,T^{2/3}\right)$, 
with per-round time complexity $O\!\left(d\log^d(k\sigma T^{1/3})\right)$.
\end{theorem}
Since the grid is fixed prior to the game, $\Tree$ remains static 
throughout with no insertions, and each level-$i$ tree maintains 
height $O(\log(1/\mu))$. The per-round cost follows from the 
$d$-stage traversal of the static hierarchy. The regret bound 
balances the bias from grid discretization, which depends on how 
well the fixed grid captures discontinuities in $\RRF_t$ under 
$\sigma$-smoothness, against the variance from importance-weighted 
estimation; the optimal $\mu$ yields the stated $T^{2/3}$ rate.
\section*{Acknowledgments}
The authors thank Guy Blelloch, Avrim Blum, and Cameron Musco for helpful discussions, and the anonymous reviewers for their comments.
\bibliography{references}

\begin{thebibliography}{32}
\providecommand{\natexlab}[1]{#1}
\providecommand{\url}[1]{\texttt{#1}}
\expandafter\ifx\csname urlstyle\endcsname\relax
  \providecommand{\doi}[1]{doi: #1}\else
  \providecommand{\doi}{doi: \begingroup \urlstyle{rm}\Url}\fi

\bibitem[Alon et~al.(1999)Alon, Matias, and Szegedy]{alon1999space}
Noga Alon, Yossi Matias, and Mario Szegedy.
\newblock The space complexity of approximating the frequency moments.
\newblock \emph{J. Comput. Syst. Sci.}, 58\penalty0 (1):\penalty0 137--147,
  1999.
\newblock \doi{10.1006/JCSS.1997.1545}.
\newblock URL \url{https://doi.org/10.1006/jcss.1997.1545}.

\bibitem[Auer et~al.(2002)Auer, Cesa{-}Bianchi, Freund, and
  Schapire]{auer2002nonstochastic}
Peter Auer, Nicol{\`{o}} Cesa{-}Bianchi, Yoav Freund, and Robert~E. Schapire.
\newblock The nonstochastic multiarmed bandit problem.
\newblock \emph{{SIAM} J. Comput.}, 32\penalty0 (1):\penalty0 48--77, 2002.
\newblock \doi{10.1137/S0097539701398375}.
\newblock URL \url{https://doi.org/10.1137/S0097539701398375}.

\bibitem[Balcan and Beyhaghi(2024)]{DBLP:journals/tmlr/BalcanB24}
Maria{-}Florina Balcan and Hedyeh Beyhaghi.
\newblock New guarantees for learning revenue maximizing menus of lotteries and
  two-part tariffs.
\newblock \emph{Transactions on Machine Learning Research}, 2024.
\newblock URL \url{https://openreview.net/forum?id=mhawjZcmrJ}.

\bibitem[Balcan and Sharma(2021)]{DBLP:conf/nips/BalcanS21}
Maria{-}Florina Balcan and Dravyansh Sharma.
\newblock Data driven semi-supervised learning.
\newblock In Marc'Aurelio Ranzato, Alina Beygelzimer, Yann~N. Dauphin, Percy
  Liang, and Jennifer~Wortman Vaughan, editors, \emph{Advances in Neural
  Information Processing Systems 34: Annual Conference on Neural Information
  Processing Systems 2021, NeurIPS 2021, December 6-14, 2021, virtual}, pages
  14782--14794, 2021.
\newblock URL
  \url{https://proceedings.neurips.cc/paper/2021/hash/7c93ebe873ef213123c8af4b188e7558-Abstract.html}.

\bibitem[Balcan et~al.(2018)Balcan, Dick, and Vitercik]{balcan2018dispersion}
Maria-Florina Balcan, Travis Dick, and Ellen Vitercik.
\newblock Dispersion for data-driven algorithm design, online learning, and
  private optimization.
\newblock In \emph{2018 IEEE 59th Annual Symposium on Foundations of Computer
  Science (FOCS)}, pages 603--614. IEEE, 2018.

\bibitem[Balcan et~al.(2020)Balcan, Dick, and Lang]{DBLP:conf/iclr/BalcanDL20}
Maria{-}Florina Balcan, Travis Dick, and Manuel Lang.
\newblock Learning to link.
\newblock In \emph{8th International Conference on Learning Representations,
  {ICLR} 2020, Addis Ababa, Ethiopia, April 26-30, 2020}. OpenReview.net, 2020.
\newblock URL \url{https://openreview.net/forum?id=S1eRbANtDB}.

\bibitem[Balcan et~al.(2021)Balcan, Khodak, Sharma, and
  Talwalkar]{balcan2021learning}
Maria{-}Florina Balcan, Mikhail Khodak, Dravyansh Sharma, and Ameet Talwalkar.
\newblock Learning-to-learn non-convex piecewise-lipschitz functions.
\newblock In \emph{Advances in Neural Information Processing Systems 34: Annual
  Conference on Neural Information Processing Systems 2021, NeurIPS 2021,
  December 6-14, 2021, virtual}, pages 15056--15069, 2021.
\newblock URL
  \url{https://proceedings.neurips.cc/paper/2021/hash/7ee6f2b3b68a212d3b7a4f6557eb8cc7-Abstract.html}.

\bibitem[Balcan et~al.(2022)Balcan, Khodak, Sharma, and
  Talwalkar]{DBLP:conf/nips/BalcanKST22}
Maria{-}Florina Balcan, Misha Khodak, Dravyansh Sharma, and Ameet Talwalkar.
\newblock Provably tuning the elasticnet across instances.
\newblock In Sanmi Koyejo, S.~Mohamed, A.~Agarwal, Danielle Belgrave, K.~Cho,
  and A.~Oh, editors, \emph{Advances in Neural Information Processing Systems
  35: Annual Conference on Neural Information Processing Systems 2022, NeurIPS
  2022, New Orleans, LA, USA, November 28 - December 9, 2022}, 2022.
\newblock URL
  \url{http://papers.nips.cc/paper\_files/paper/2022/hash/b21a34c4e8dba253f05f4a5adc68ba73-Abstract-Conference.html}.

\bibitem[Balcan et~al.(2024{\natexlab{a}})Balcan, Nguyen, and
  Sharma]{balcan2024algorithmconfigurationstructuredpfaffian}
Maria-Florina Balcan, Anh~Tuan Nguyen, and Dravyansh Sharma.
\newblock Algorithm configuration for structured pfaffian settings,
  2024{\natexlab{a}}.
\newblock URL \url{https://arxiv.org/abs/2409.04367}.

\bibitem[Balcan et~al.(2024{\natexlab{b}})Balcan, Seiler, and
  Sharma]{balcan2024accelerating}
Maria{-}Florina Balcan, Christopher Seiler, and Dravyansh Sharma.
\newblock Accelerating {ERM} for data-driven algorithm design using
  output-sensitive techniques.
\newblock In \emph{Advances in Neural Information Processing Systems 38: Annual
  Conference on Neural Information Processing Systems 2024, NeurIPS 2024,
  Vancouver, BC, Canada, December 10 - 15, 2024}, 2024{\natexlab{b}}.
\newblock URL
  \url{http://papers.nips.cc/paper\_files/paper/2024/hash/850d6cd6cca1398d9251e5ae870fad0e-Abstract-Conference.html}.

\bibitem[Blum and Hartline(2005)]{blum2005near}
Avrim Blum and Jason~D. Hartline.
\newblock Near-optimal online auctions.
\newblock In \emph{Proceedings of the Sixteenth Annual {ACM-SIAM} Symposium on
  Discrete Algorithms, {SODA} 2005, Vancouver, British Columbia, Canada,
  January 23-25, 2005}, pages 1156--1163. {SIAM}, 2005.
\newblock URL \url{http://dl.acm.org/citation.cfm?id=1070432.1070597}.

\bibitem[Brewer and Hanif(1983)]{brewer1983sampling}
Ken~RW Brewer and Muhammad Hanif.
\newblock \emph{Sampling with unequal probabilities}.
\newblock New York: Springer, 1983.

\bibitem[Bubeck et~al.(2011)Bubeck, Munos, Stoltz, and
  Szepesv{\'a}ri]{bubeck2011x}
S{\'e}bastien Bubeck, R{\'e}mi Munos, Gilles Stoltz, and Csaba Szepesv{\'a}ri.
\newblock X-armed bandits.
\newblock \emph{Journal of Machine Learning Research}, 12\penalty0 (5), 2011.

\bibitem[Cesa-Bianchi and Lugosi(2006)]{cesa2006prediction}
Nicolo Cesa-Bianchi and Gabor Lugosi.
\newblock \emph{Prediction, Learning, and Games}.
\newblock Cambridge University Press, 2006.

\bibitem[Cesa-Bianchi et~al.(2017)Cesa-Bianchi, Gaillard, Gentile, and
  Gerchinovitz]{cesa2017algorithmic}
Nicol{\`o} Cesa-Bianchi, Pierre Gaillard, Claudio Gentile, and S{\'e}bastien
  Gerchinovitz.
\newblock Algorithmic chaining and the role of partial feedback in online
  nonparametric learning.
\newblock In \emph{Conference on Learning Theory}, pages 465--481. PMLR, 2017.

\bibitem[Cheung(2014)]{cheung2014}
Adam Ka~Lok Cheung.
\newblock \emph{Probability Proportional Sampling}, pages 5069--5071.
\newblock Springer Netherlands, Dordrecht, 2014.
\newblock ISBN 978-94-007-0753-5.

\bibitem[Cochran(1977)]{cochran1977sampling}
William~Gemmell Cochran.
\newblock \emph{Sampling techniques}.
\newblock john wiley \& sons, 1977.

\bibitem[Cohen-Addad and Kanade(2017)]{cohen2017online}
Vincent Cohen-Addad and Varun Kanade.
\newblock Online optimization of smoothed piecewise constant functions.
\newblock In \emph{Artificial Intelligence and Statistics}, pages 412--420.
  PMLR, 2017.

\bibitem[Cormen et~al.(2022)Cormen, Leiserson, Rivest, and
  Stein]{cormen2022introduction}
Thomas~H Cormen, Charles~E Leiserson, Ronald~L Rivest, and Clifford Stein.
\newblock \emph{Introduction to algorithms}.
\newblock MIT press, 2022.

\bibitem[Cormode et~al.(2012)Cormode, Garofalakis, Haas, and
  Jermaine]{cormode2011synopses}
Graham Cormode, Minos Garofalakis, Peter~J. Haas, and Chris Jermaine.
\newblock Synopses for massive data: Samples, histograms, wavelets, sketches.
\newblock \emph{Found. Trends Databases}, 4\penalty0 (1--3):\penalty0 1--294,
  January 2012.
\newblock ISSN 1931-7883.

\bibitem[De~Berg(2000)]{de2000computational}
Mark De~Berg.
\newblock \emph{Computational geometry: algorithms and applications}.
\newblock Springer Science \& Business Media, 2000.

\bibitem[Dick et~al.(2020)Dick, Pegden, and Balcan]{dick2020semi}
Travis Dick, Wesley Pegden, and Maria{-}Florina Balcan.
\newblock Semi-bandit optimization in the dispersed setting.
\newblock In \emph{Proceedings of the Thirty-Sixth Conference on Uncertainty in
  Artificial Intelligence, {UAI} 2020, virtual online, August 3-6, 2020},
  volume 124 of \emph{Proceedings of Machine Learning Research}, pages
  909--918. {AUAI} Press, 2020.
\newblock URL \url{http://proceedings.mlr.press/v124/dick20a.html}.

\bibitem[Drineas et~al.(2006)Drineas, Kannan, and Mahoney]{drineas2006fast}
Petros Drineas, Ravi Kannan, and Michael~W. Mahoney.
\newblock Fast monte carlo algorithms for matrices {II:} computing a low-rank
  approximation to a matrix.
\newblock \emph{{SIAM} J. Comput.}, 36\penalty0 (1):\penalty0 158--183, 2006.
\newblock \doi{10.1137/S0097539704442696}.
\newblock URL \url{https://doi.org/10.1137/S0097539704442696}.

\bibitem[Gupta and Roughgarden(2016)]{gupta2016pac}
Rishi Gupta and Tim Roughgarden.
\newblock A pac approach to application-specific algorithm selection.
\newblock In \emph{Proceedings of the 2016 ACM Conference on Innovations in
  Theoretical Computer Science}, pages 123--134, 2016.

\bibitem[Haghtalab et~al.(2024)Haghtalab, Roughgarden, and
  Shetty]{haghtalab2024smoothed}
Nika Haghtalab, Tim Roughgarden, and Abhishek Shetty.
\newblock Smoothed analysis with adaptive adversaries.
\newblock \emph{Journal of the ACM}, 71\penalty0 (3):\penalty0 1--34, 2024.

\bibitem[Ibtehaz et~al.(2021)Ibtehaz, Kaykobad, and
  Rahman]{ibtehaz2021multidimensional}
Nabil Ibtehaz, M~Kaykobad, and M~Sohel Rahman.
\newblock Multidimensional segment trees can do range updates in
  poly-logarithmic time.
\newblock \emph{Theoretical Computer Science}, 854:\penalty0 30--43, 2021.

\bibitem[Kalai and Vempala(2005)]{kalai2005efficient}
Adam~Tauman Kalai and Santosh~S. Vempala.
\newblock Efficient algorithms for online decision problems.
\newblock \emph{J. Comput. Syst. Sci.}, 71\penalty0 (3):\penalty0 291--307,
  2005.
\newblock \doi{10.1016/J.JCSS.2004.10.016}.
\newblock URL \url{https://doi.org/10.1016/j.jcss.2004.10.016}.

\bibitem[Lov{\'a}sz and Vempala(2007)]{lovasz2007geometry}
L{\'a}szl{\'o} Lov{\'a}sz and Santosh Vempala.
\newblock The geometry of logconcave functions and sampling algorithms.
\newblock \emph{Random Structures \& Algorithms}, 30\penalty0 (3):\penalty0
  307--358, 2007.

\bibitem[Motwani and Raghavan(1995)]{motwani1995randomized}
Rajeev Motwani and Prabhakar Raghavan.
\newblock \emph{Randomized Algorithms}.
\newblock Cambridge University Press, 1995.

\bibitem[Pittel(1984)]{pittel1984growing}
Boris Pittel.
\newblock On growing random binary trees.
\newblock \emph{Journal of Mathematical Analysis and Applications},
  103\penalty0 (2):\penalty0 461--480, 1984.

\bibitem[Spielman and Teng(2001)]{spielman2001smoothed}
Daniel~A. Spielman and Shang{-}Hua Teng.
\newblock Smoothed analysis of algorithms: Why the simplex algorithm usually
  takes polynomial time.
\newblock \emph{CoRR}, cs.DS/0111050, 2001.
\newblock URL \url{https://arxiv.org/abs/cs/0111050}.

\bibitem[Zhu et~al.(2023)Zhu, Bates, Yang, Wang, Jiao, and Jordan]{zhu23sample}
Banghua Zhu, Stephen Bates, Zhuoran Yang, Yixin Wang, Jiantao Jiao, and
  Michael~I. Jordan.
\newblock The sample complexity of online contract design.
\newblock In \emph{Proceedings of the 24th {ACM} Conference on Economics and
  Computation, {EC} 2023, London, United Kingdom, July 9-12, 2023}, page 1188.
  {ACM}, 2023.

\end{thebibliography}
\newpage
\appendix
\begin{center}
    {\LARGE\bfseries Appendix}
\end{center}
\textbf{Roadmap to the Appendix}
\begin{itemize}
    \item \Cref{sec:aux_param} presents \Cref{examp:decompose}, illustrating the idea behind the parametric representation of reward functions. \Cref{section:scaling} then introduces the vector encoding of cumulative rewards and the downward-scaling formulation that transfers cumulative rewards from a region to its subregions without re-integration.

    \item \Cref{sec:proofbalanced} proves \Cref{thm:balancedness}, giving tight bounds on the expected height of the interval trees at each layer. Specifically, \Cref{sec:reductiontoone} reduces the problem to one dimension; \Cref{sec:canpath} introduces the canonical root-to-leaf path and the $r^*$-averse adversary; \Cref{sec:monotopt} establishes the monotone structure of the optimal adversarial strategy; and \Cref{sec:adaptiveadvh,sec:oblivioush,sec:randord} prove the three parts of the theorem — the tight $\Theta(\sqrt{\sigma Tk})$ bounds for adaptive and oblivious adversaries, and the $O(\log(Tk))$ bound for random-order adversaries, respectively.

    \item Serving as a warm-up, \Cref{section:ds-qhandlingapp} presents how the basic queries $\Ins$ and $\Update$ are handled through an example insertion in two dimensions. \Cref{sec:exampdraw} presents the backbone of our sampling by describing the $\Draw$ query, whose final step uses the hit-and-run technique (\Cref{them:vempa}).

    \item \Cref{section:quereis} handles the full set of queries $\Ins$, $\Update$, and $\Draw$ via detailed steps and pseudocode for the $2$-dimensional case, forming the backbone of query handling in higher dimensions. \Cref{Sec:Weigeneral} details the cumulative reward retrieval query $\Wei(.)$ required by $\Draw$, explaining how downward scaling is carried out within the data structure. \Cref{section:draw} then uses this query to sample coordinate-by-coordinate, which in two dimensions corresponds to \emph{horizontal} and \emph{vertical} sampling.

    \item \Cref{section:correctness} proves the correctness of all data structure operations: the structural properties induced by $\Ins$ and $\Update$, the uniqueness and completeness of the second-level structures, the correctness of both $\Wei$ subroutines, and finally that $\Draw$ returns each atomic region with the correct probability.

    \item \Cref{section:generalpar} generalizes the constructions and query-handling operations from the $2$-dimensional case to $d$ dimensions.

    \item \Cref{section:complexity} analyzes the time complexity of $\Ins$, $\Update$, $\Wei$, and $\Draw$ under adaptive, oblivious, and random-order adversaries, along with the memory complexity of the data structure. 

    \item \Cref{sec:learningapp} overviews the application to efficient adversarial online learning: the regret objective and the two feedback models (full-information and bandit), constituting the setup for the proofs of Theorems 6.1 and 6.2. 

     \item \Cref{app:full-info} treats the full-information setting, where the learner runs $\Exp$ by sampling proportionally to $\exp(\eta \MF(x, \Fu_t(x)))$, whose non-additivity across lazy and regular structures requires restricting the adversary. Specifically, \Cref{app:challenges} identifies the computational obstacles of exponential weighting — non-additivity of $\exp$ (\Cref{examp:full}), non-existence of closed-form aggregate rewards for degree $\ge 2$, and inefficient scaling under piecewise-linear rewards — and \Cref{app:full-setup} presents the resulting restricted adversarial model (smoothness along one designated direction, at most $M$ fixed hyperplanes along the rest). \Cref{app:full-ds} presents the corresponding data structure and the modified $\Wei$/$\Draw$ routines with multiplicative exponential scaling; \Cref{section:regretful} collects the smoothness-related lemmas underlying the regret analysis, which proves the piecewise-constant and more general piecewise-linear regret bounds with corollaries instantiating $\eta$; and \Cref{thm:expcomp} establishes the $O(dM^{d-1}\sqrt{\sigma Tk} + d)$ per-round time complexity.

    \item \Cref{app:bandit} treats the bandit setting, where the learner runs $\Band$ by sampling proportionally to $\exp(\eta \MF(x, \hat{\Fu}_t(x)))$ over a fixed grid of $(1/\mu)^d$ cells, whose static partition removes the need for insertions, lazy structures, and any restriction on the adversary beyond standard $\sigma$-smoothness. Specifically, \Cref{app:bandit-setup} presents the grid-based setup and adversarial model, and \Cref{app:bandit-ds} presents the corresponding data structure — storing estimated parameter vectors $\hat{\Fu}_t$ at the leaves — and its $\Update$/$\Draw$ handling. \Cref{section:regretband} proves the regret bound for piecewise-linear rewards with its $O(\mathrm{poly}(d,k,\sigma,\log T)\, T^{2/3})$ instantiation, and \Cref{thm:bandcomp} establishes the $O(d\log^d(1/\mu) + d^3)$ per-round time complexity.
\end{itemize}
\section{Parametric Representation of Reward Functions (\Cref{sec:que})}\label{sec:aux_param}

We begin with simple examples illustrating how the mapping function \(\MF\) and the parameter function \(\fu_t\) jointly represent the reward.

\begin{example}[Parametric Representation of Piecewise Polynomial Rewards]
\label{examp:decompose}
We illustrate the decomposition \(\RRF_t(x) = \MF(x,\fu_t(x))\) with simple examples.

\paragraph{Piecewise-linear reward.}
Consider a reward function of the form
\[
\RRF_t(x) = 3x^{(0)} + 5x^{(1)},
\]
where \(x^{(i)}\) denotes the \(i\)-th coordinate of \(x\). The monomial vector is \(\langle x^{(0)}, x^{(1)} \rangle\) fixed among all piecewise linear, and the corresponding coefficient vector is \(\langle 3,5 \rangle\). We define the parameter function as \(\fu_t(x) = \langle 3,5 \rangle\), and the mapping function as
\[
\MF(x,\fu_t(x)) = \langle x^{(0)}, x^{(1)} \rangle \cdot \fu_t(x),
\]
which recovers \(\RRF_t(x)\).

\paragraph{Higher-order polynomial reward.}
For a reward function such as
\[
\RRF_t(x) = 2x^{(0)} + 3(x^{(1)})^2,
\]
we define \(\fu_t(x) = \langle 2,3 \rangle\) and
\[
\MF(x,\fu_t(x)) = \langle x^{(0)}, (x^{(1)})^2 \rangle \cdot \fu_t(x),
\]
which again recovers \(\RRF_t(x)\).

\medskip
In both cases, the mapping function \(\MF\) specifies the monomial structure, while the parameter function \(\fu_t\) provides the corresponding coefficients.
\end{example}

We next record the key structural property that makes this representation useful algorithmically.

\begin{observation}[Additivity of the Parameter Function]
The parameter function \(\fu_t\) defines the coefficient vector for the reward at round \(t\), and this contribution is additive across rounds. Thus, we define the cumulative parameter function as
\[
\Fu_t(x) \doteq \sum_{t' \leq t} \fu_{t'}(x),
\]
where \(\Fu_t(x)\) represents the total coefficient vector at point \(x \in [0,1)^d\), summing the contributions from all rounds \(t' \leq t\).
In particular, \(\Fu_t(x)\) accumulates \(\fu_{t'}(\rg_{i,t'})\) over all regions \(\rg_{i,t'}\) such that \(x \in \rg_{i,t'}\).
\end{observation}
\subsection{Scaling Cumulative Rewards Across Regions via a Vector Encoding}\label{section:scaling} 

\paragraph{Vector Encoding of Region's Cumulative Rewards.}
In our framework, we store a vector-valued cumulative reward \(\vec{\CR}(\rg) \in \mathbb{R}^{\param}\) for each region \(\rg\), rather than the scalar reward \(\CR(\rg)\). This vector is obtained by integrating the parameter function \(\Fu_t(x)\) over \(\rg\), as formalized in Equation~\eqref{eq:crew}. Intuitively, each coordinate of this vector corresponds to the contribution of a single monomial in the fixed polynomial reward family. Thus, \(\vec{\CR}(\rg)\) compactly captures how different components of the reward accumulate over the region.

This representation enables a key operation in our framework: \emph{scaling}. Instead of recomputing integrals over new regions from scratch, we can reuse the vector representation of a parent region and transfer it to subregions using only geometric information.

We now formalize this idea.

\subsubsection{Cumulative Multivariate Polynomial Reward Function over Hyperrectangular Regions}

We consider reward functions that are multivariate polynomials defined over axis-aligned rectangular regions. Let
\[
\rg : [x_1^{(0)} : x_2^{(0)}][x_1^{(1)} : x_2^{(1)}]\cdots[x_1^{(d-1)} : x_2^{(d-1)}]
\]
be a \(d\)-dimensional hyperrectangle over which we have integrated a polynomial reward function. Our goal is to transfer this cumulative reward to another rectangular region
\[
\rg': [x_1'^{(0)} : x_2'^{(0)}][x_1'^{(1)} : x_2'^{(1)}]\cdots[x_1'^{(d-1)} : x_2'^{(d-1)}],
\]
using only the geometry of \(\rg\) and \(\rg'\), without recomputing the integral from the original function.

To formalize this process, we represent the reward function in monomial form.

\begin{definition}\label{def:hpoly}
Let \( \Poly \) be a polynomial reward function of total degree at most \( \param \) over \( d \) variables. We represent it as
\[
\Poly(x^{(0)}, x^{(1)}, \ldots, x^{(d-1)}) = \sum_{\alpha \in A} c_\alpha \prod_{i=0}^{d-1} \left(x^{(i)}\right)^{\alpha_i},
\]
where:
\begin{itemize}
    \item \( A \subseteq \mathbbm{Z}_{\ge 0}^d \) is a finite set of multi-indices such that \( \sum_{i=0}^{d-1} \alpha_i \leq \param \) for each \( \alpha \in A \),
    \item Each monomial is indexed by \( \alpha = (\alpha_0, \alpha_1, \ldots, \alpha_{d-1}) \),
    \item \( c_\alpha \in \mathbbm{R} \) is the coefficient of the monomial indexed by \( \alpha \).
\end{itemize}
\end{definition}

We fix a canonical ordering over the set \( A = \{\alpha^1, \alpha^2, \dots, \alpha^{\abs{A}}\} \), and define the coefficient vector as
\[
\vec{C} \doteq \langle c_{\alpha^1}, c_{\alpha^2}, \dots, c_{\alpha^{\abs{A}}} \rangle.
\]
This ordering allows us to represent polynomial quantities consistently as vectors.

\begin{observation}
The number of multi-indices in the set \( A \) for a polynomial \( \Poly \) as defined in \eqref{def:hpoly} is at most \( O(n^d) \), where \( n \) is the degree of the polynomial.
\end{observation}

We now recall that the integral of a multivariate polynomial over a rectangular region decomposes into independent contributions from each monomial, which can be computed in closed form.

Given
\[
\rg: [x^{(0)}_1:x^{(0)}_2][x^{(1)}_1:x^{(1)}_2] \cdots [x_1^{(d-1)}:x_2^{(d-1)}],
\]
for each coordinate \( i \in \{0, 1, \dotsc, d-1\} \), we define
\begin{align}
    \Delta^{(i)}(\alpha_i) := \frac{\left(x_2^{(i)}\right)^{\alpha_i + 1} - \left(x_1^{(i)}\right)^{\alpha_i + 1}}{\alpha_i + 1}. \label{eq:polydelta}
\end{align}

\begin{remark}[Polynomial Reward Function Integration over Rectangular Regions]
For any degree-\( \param \) polynomial \( \Poly \) as defined in Definition~\ref{def:hpoly}, the cumulative reward over \( \rg \) can be expressed as
\begin{align}
\int_{\rg} \Poly(x^{(0)}, x^{(1)}, \cdots, x^{(d-1)}) \, d\rg = \sum_{\alpha \in A} c_{\alpha} \prod_{i=0}^{d-1} \Delta^{(i)}(\alpha_i). \label{eq:polyinteg}
\end{align}
\end{remark}

\subsubsection{Scaling $\int_\rg \Poly$ to $\int_{\rg'} \Poly$} 

We now describe how to transfer the cumulative reward of a polynomial from a region \(\rg\) to a subregion \(\rg'\) without recomputing the integral from scratch.

Suppose we have already computed \(\int_\rg \Poly\) using~\eqref{eq:polyinteg}. A direct computation over \(\rg'\) would require access to all coefficients \(c_\alpha\). Instead, we show how to reuse the information from \(\rg\).

Using the decomposition from~\eqref{eq:polyinteg}, we can express the integral over \( \rg' \) as:
\begin{align}
\int_{\rg'} \Poly
= \sum_{\alpha \in A} c_{\alpha} \prod_{i=0}^{d-1} \Delta^{(i)}(\alpha_i) \cdot 
\frac{(x_2'^{(i)})^{\alpha_i + 1} - (x_1'^{(i)})^{\alpha_i + 1}}{(x_2^{(i)})^{\alpha_i + 1} - (x_1^{(i)})^{\alpha_i + 1}}.
\end{align}

To avoid explicitly storing the coefficients \( c_\alpha \), we define a cumulative reward vector for the region \( \rg \), denoted \( \vec{\SC}_{\rg} \in \mathbb{R}^{|A|} \), as:
\begin{align}
\vec{\SC}_{\rg} \doteq \langle c_{\alpha^1} \prod_{i=0}^{d-1} \Delta^{(i)}(\alpha^1_i), \dotsc, c_{\alpha^{|A|}} \prod_{i=0}^{d-1} \Delta^{(i)}(\alpha^{|A|}_i) \rangle. \label{eq:scalevecs}
\end{align}

Each entry of this vector corresponds to the contribution of a single monomial to the integral over \(\rg\).

The integral over \( \rg \) is simply the sum of the entries of this vector:
\begin{align}
    \int_\rg \Poly = \sum_{j=1}^{|A|} \vec{\SC}_\rg^{(j)}.
\end{align}

We can now compute the integral over \( \rg' \) by scaling each entry independently using geometry-dependent factors:
\begin{align}
    \int_{\rg'} \Poly = \sum_{j=1}^{|A|} \vec{\SC}_\rg^{(j)} \cdot \prod_{i=0}^{d-1} 
    \frac{(x_2'^{(i)})^{\alpha^j_i + 1} - (x_1'^{(i)})^{\alpha^j_i + 1}}{(x_2^{(i)})^{\alpha^j_i + 1} - (x_1^{(i)})^{\alpha^j_i + 1}}. \label{eq:scalingmain}
\end{align}

This shows that the cumulative reward over \(\rg'\) can be computed using only the vector \(\vec{\SC}_\rg\) and the geometric relationship between \(\rg\) and \(\rg'\).

\begin{definition}[Scaling]
Let \(\rg\) be a region with associated parameter function \(\Fu_t\), and let \(\CR(\rg)\) denote its cumulative reward. We define \emph{scaling} as the process of computing \(\CR(\rg')\) for any subregion \(\rg' \subseteq \rg\) with the same \(\Fu_t\), without performing a new integration over \(\rg'\). Scaling operates on the vector \(\vec{\CR}(\rg)\) using closed-form scaling rules derived from \Cref{eq:scalingmain}.
\end{definition}

\begin{example}\label{example:scaling}
Let \(\rg = [a_1{:}a_2][b_1{:}b_2]\) be a rectangular region in two dimensions, and suppose the mapping function is linear: \(\MF(x, \Fu_t(x)) =  \langle x^{(0)}, x^{(1)} \rangle\cdot \Fu_t(x) \). If \(\Fu_t(x) = \langle 2c_1, c_2 \rangle\), corresponding to the monomials \(x^{(0)}\) and \(x^{(1)}\), then the cumulative reward and its vector representation are:
\begin{align*}
   \CR(\rg) &= c_1 (a_2^2 - a_1^2)(b_2 - b_1) + \frac{c_2}{2} (a_2 - a_1)(b_2^2 - b_1^2),\\
   \vec{\CR}(\rg) &= \left\langle c_1 (a_2^2 - a_1^2)(b_2 - b_1),\; \frac{c_2}{2} (a_2 - a_1)(b_2^2 - b_1^2) \right\rangle.
\end{align*}

Unlike real-valued rewards, \(\vec{\CR}(\rg)\) enables direct scaling to subregions \(\rg' = [a'_1{:}a'_2][b'_1{:}b'_2]\) without re-integration. Since each entry corresponds to a known monomial, we can scale them independently using closed-form expressions determined by the monomial degrees and the geometry of \(\rg'\), without access to \(\Fu_t\). Specifically:
\begin{align*}
    \CR(\rg') =\; & \vec{r}^{(0)}(\rg) \cdot \frac{(a'^2_2 - a'^2_1)(b'_2 - b'_1)}{(a^2_2 - a^2_1)(b_2 - b_1)} 
    + \vec{r}^{(1)}(\rg) \cdot \frac{(a'_2 - a'_1)(b'^2_2 - b'^2_1)}{(a_2 - a_1)(b^2_2 - b^2_1)},
\end{align*}
where the superscript \(i\) in \(\vec{r}^{(i)}(\rg)\) refers to the entries of the vector \(\vec{\CR}(\rg)\), corresponding to different monomial components. This expression shows that \(\CR(\rg')\) can be computed solely from the endpoints of the subregion \(\rg'\) and the precomputed vector \(\vec{\CR}(\rg)\), without requiring direct access to the parameter function used to compute \(\vec{\CR}(\rg)\).
\end{example}
\section{ Tree Height under Smoothed Adversaries: Proof of Theorem~\ref{thm:balancedness}}
\label{sec:proofbalanced}

\subsection{Balancedness of the Data Structure is Guaranteed Under Smoothness}\label{sec:detailsbalanced}

In this section, we analyze the height of the data structure under adaptive and oblivious adversaries. Since the adversary acts independently across coordinates under the smoothing assumption, it suffices to analyze a fixed coordinate $i$ (see~\cref{sec:reductiontoone}).

\begin{theorem} \label{thm:balancednessmain}
Fix a coordinate $i \in [d]$ and let $\mathcal{A}$ be a $\sigma$-smooth adversary. Let $x_1, \dots, x_{Tk} \in [0,1]$ be points drawn sequentially from $\sigma$-smooth distributions $\mathcal{D}^{(i)}_1, \dots, \mathcal{D}^{(i)}_{Tk}$ supported on $[0,1]$ and inserted into the interval tree $\Tree^{(i)}$ as partition endpoints along coordinate $i$. Then,
\[
\mathbb{E}[\EH^{(i)}] =
\begin{cases}
O(\sqrt{\sigma Tk}), & \text{if $\mathcal{A}$ is adaptive},\\[3pt]
O(\sqrt{\sigma Tk}), & \text{if $\mathcal{A}$ is oblivious},\\[3pt]
O(\log(Tk)), & \text{if $\mathcal{A}$ is oblivious and random-order}.
\end{cases}
\]
Where the upper bound for all adversaries are proven to be asymptotically tight.
\end{theorem}

 At a high level, we reduce both adaptive and oblivious adversaries to an $r^*$-averse adversary (see~\cref{def:pathaverseadv}) that seeks to maximize the height of $\Tree^{(i)}$ along the canonical root-to-leaf path $r^*$. We then characterize a monotone optimal strategy and analyze the three settings separately. For the adaptive case, we derive an optimal dynamic program and prove a tight $\Theta(\sqrt{\sigma Tk})$ bound. We then show that oblivious adversaries cannot, in general, improve this worst-case behavior to a logarithmic bound by analyzing the simple though tricky example
\(
\mathcal{D}_t=\Unif\!\left(\left[\frac{t-1}{\sigma Tk},\,\frac{1}{\sigma}+\frac{t-1}{\sigma Tk}\right]\right),
\)
in which all $Tk$ monotonically shifted uniform distributions have a nonempty common intersection, yet the expected height is still $\Theta(\sqrt{Tk})$. Finally, for the random-order oblivious case, using the result of~\cite{pittel1984growing} we show that the additional symmetry yields the sharper $O(\log(Tk))$ bound.

\paragraph{Longest Increasing Subsequence under Smoothed Adversary.} Our tree-height analysis is closely related to a previously unstudied problem of independent interest. We introduce and solve the expected length of the \emph{longest increasing subsequence}---a well-studied problem in algorithm design---under the smoothed adversarial input model. The problem is as follows: Suppose an adversary (oblivious or adaptive) is generating an online sequence of real numbers in $[0,1]$, with the goal of maximizing the expected length of the longest increasing subsequence. At every round, the adversary chooses a $\sigma$-smooth distribution, where a random draw from the distribution forms the next input in the sequence. Our analysis establishes a surprising asymptotically tight upper bound of \(O(\sqrt{\sigma T})\) for this problem.

\subsection{Reducing the Problem to $1$-dimension} \label{sec:reductiontoone}

We reduce the analysis to a single dimension. To prove Theorem~\ref{thm:balancedness}, it suffices to study a $\sigma$-smooth adversary that maximizes the expected height of $\Tree$ along one coordinate.

By construction, the data structure is layered coordinate-wise, and the adversary acts independently across coordinates. Thus, we focus on the first coordinate without loss of generality, and the same analysis applies to any level-$i$ interval tree $\Tree^{(i)}$. The endpoints along this coordinate are inserted into a primary interval tree $\Tree^{(0)}$, which stores at most $O(kt)$ endpoints up to time $t$, corresponding to projections of the regions induced by the $k$ hyperplanes in $\Hp_t$ at each round. Endpoints from the remaining coordinates are stored in secondary trees attached to nodes of $\Tree^{(0)}$, decoupling the first coordinate from the others.

Moreover, under this setup in~\Cref{sec:model}, the distributions $\mathcal{D}^{(i)}_j$ are independent across coordinates. Hence, the randomness governing $\Tree^{(0)}$ is independent of that of the secondary trees. It follows that analyzing $\Tree^{(0)}$ in isolation suffices: once $\Tree^{(0)}$ is balanced, the same holds for every coordinate layer, and thus for the entire structure $\Tree$. Consequently, a height bound for the first coordinate yields Theorem~\ref{thm:balancedness}.

We next analyze the maximum expected height that a $\sigma$-smooth adversary can induce in the interval tree $\Tree^{(0)}$.

We first argue that it suffices to consider a one-dimensional worst-case reduction. Consider an adversary that selects hyperplanes perpendicular to the first coordinate, i.e., vertical hyperplanes that only partition the $x$-axis while spanning the remaining coordinates. In this case, all induced regions are intervals along the first coordinate (together with the box boundary $[0,1]^d$), and all resulting endpoints are inserted exclusively into $\Tree^{(0)}$. This maximizes the adversary's ability to increase the height along this coordinate. Therefore, an upper bound for the case where all $Tk$ endpoints are inserted into $\Tree^{(0)}$ yields a valid upper bound for the general problem.

We now formalize the induced insertion process. In round $\tau \in \{1,\dots,T\}$, the adversary selects a batch of endpoints
$\bp^{(0)}_{1,\tau}, \dots, \bp^{(0)}_{k^{(0)}_\tau,\tau}$ along the first coordinate. Over $t$ rounds, these generate a sequence of at most $kt$ insertions into $\Tree^{(0)}$. For notational simplicity, we relabel all endpoints sequentially as $x_1, \dots, x_{Tk}$, and use this notation throughout the proof.

Finally, we flatten the batched process into $Tk$ sequential insertions. That is, instead of selecting up to $k$ endpoints per round, the adversary chooses one endpoint at a time over $Tk$ rounds. This only strengthens the adversary, as it can adapt each choice based on all previous insertions. Therefore, an upper bound under this sequential process directly implies the same bound for the original batched setting.

We thus consider the following equivalent formulation:
\begin{quote}
\textbf{Over $Tk$ rounds, the adversary selects distributions $\mathcal{D}_1, \dots, \mathcal{D}_{Tk}$ supported on $[0,1]$, each with density bounded by $\sigma$. In round $t \in \{1,\dots,Tk\}$, a point $x_t \sim \mathcal{D}_t$ is sampled and inserted into $\Tree^{(0)}$. The goal is to bound the worst-case expected height of $\Tree^{(0)}$ over all adaptive (or oblivious) choices of $\{\mathcal{D}_t\}_{t=1}^{Tk}$.}
\end{quote}

\subsection{Focusing on Root-to-Leaf Path Averse}\label{sec:canpath}

We now introduce a reduction that simplifies the analysis of the expected height of $\Tree^{(0)}$, denoted by $\mathbb{E}[\EH(\Tree^{(0)})]$. Recall that the height of a binary tree (including $\IT$s) is the length of its longest root-to-leaf path. Since the insertion process induces a random tree, bounding its height can be reduced to understanding how an adversary can keep \emph{extending} a single root-to-leaf path. Indeed, for the purpose of maximizing height, an adversary that concentrates its choices on lengthening one path is at least as powerful as one that disperses its effort across multiple paths. Accordingly, it suffices to analyze the strongest strategy for lengthening a fixed path.

By symmetry among root-to-leaf paths, we may fix a convenient representative without loss of generality. We focus on the path that starts at $\ROOT(\Tree^{(0)})$ and, at each internal node, proceeds to the right child; we refer to this as the \emph{canonical path}. In \cref{def:pathaverseadv}, we formalize an adversary that steers insertions to extend this path, and we later characterize the optimal policy for doing so. This reduction is the key step toward deriving the desired upper bound on the maximum expected height of $\Tree^{(0)}$, and hence on the balancedness of the entire structure.

\begin{definition}[Canonical Path $r^*$]\label{def:canpath}
Let $r^*$ denote the path in $\Tree^{(0)}$ that starts at the root and repeatedly takes the right child. Formally,
$$r^* = (v_0=\ROOT(\Tree^{(0)}),\, v_1,\, v_2,\, \ldots),$$
where $v_i = \RC(v_{i-1})$ for all $i \ge 1$.
\end{definition}

\begin{figure}[!ht]
    \centering
  \begin{tikzpicture}[scale=0.5, every node/.style={circle, draw, minimum size=0.5cm}, >=Stealth]

\node (A) at (0,0) {$v_0$};
\node (B) at (2,-2) {$v_1$};
\node (C) at (4,-4) {$v_2$};
\node (D) at (6,-6) {};

\draw[->] (A) -- (B);
\draw[->] (B) -- (C);
\draw[dotted, ->] (C) -- (D);

\coordinate (TA) at ($(A)+(-2.5,-2.5)$);
\draw (A) -- (TA);
\draw (TA) -- ++(-1,-1.5) -- ++(2,0) -- cycle;

\coordinate (TB) at ($(B)+(-2.5,-2.5)$);
\draw (B) -- (TB);
\draw (TB) -- ++(-1,-1.5) -- ++(2,0) -- cycle;

\coordinate (TC) at ($(C)+(-2.5,-2.5)$);
\draw (C) -- (TC);
\draw (TC) -- ++(-1,-1.5) -- ++(2,0) -- cycle;

\end{tikzpicture}

    \caption{The canonical path $r^*$.}
    \label{fig:path*}
\end{figure}
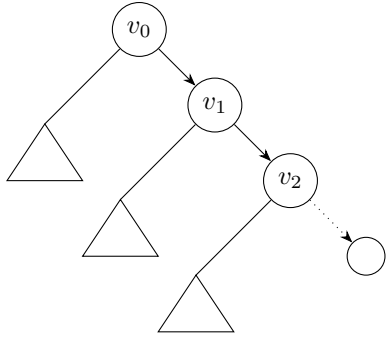

Figure~\ref{fig:path*} illustrates the canonical path $r^*$. We focus on adversaries that seek to maximize the growth of $\Tree^{(0)}$ by steering insertions to extend $r^*$ over the $Tk$ (flattened) insertion rounds. We refer to such adversaries as \emph{$r^*$-averse}.

\begin{definition}[$r^*$-averse Adversary and Optimal Policy $\pi_{r^*}$] \label{def:pathaverseadv}
An \emph{$r^*$-averse adversary} is an adversary that, over $Tk$ rounds, selects distributions $\mathcal{D}_1,\dots,\mathcal{D}_{Tk}$ (supported on $[0,1]$ and satisfying the $\sigma$-smoothness constraint) with the objective of extending the canonical path $r^*$, and thereby maximizing the expected height of $\Tree^{(0)}$. We denote by $\pi_{r^*}$ an optimal policy for this objective.
\end{definition}

\subsection{The Monotonic Structure of the Optimal Strategy $\pi_{r^*}$ for the $r^*$-Averse Adversary}\label{sec:monotopt}

We now characterize the structure of the optimal policy $\pi_{r^*}$ for the $r^*$-averse adversary. We show that, in both the adaptive and oblivious settings, the adversary may be restricted without loss of generality to a simple monotone family of strategies. First, among all $\sigma$-smooth distributions on $[0,1]$, it is optimal to choose a uniform distribution over an interval of length $1/\sigma$. Second, the support of this interval may be assumed to move weakly to the right over time. Together, these two properties imply that the optimal strategy is fully described by a sequence of nonnegative shifts.

We formalize these two structural properties below.

\begin{lemma}[Reduction to uniform intervals] \label{lem:uniform-optimal}
Fix any round $t$ and any admissible history up to round $t-1$. Among all distributions on $[0,1]$ whose density is bounded above by $\sigma$, there exists an optimal choice for the $r^*$-averse adversary at round $t$ that is uniform over an interval of length $1/\sigma$.
\end{lemma}

\begin{proof}
Fix a round $t$ and condition on the history up to round $t-1$. At this point, the adversary chooses a distribution $\mathcal{D}_t$ on $[0,1]$ subject to the density constraint $f_t(x) \le \sigma$ for all $x \in [0,1]$. Since the objective is to maximize the probability of extending the canonical path $r^*$, the adversary benefits from concentrating as much mass as possible on the most favorable region of the domain. Under the pointwise density constraint, the maximum mass that can be assigned to any measurable set of length $\ell$ is at most $\sigma \ell$. Therefore, the most concentrated feasible choice is to saturate the density bound on an interval of minimal possible length, namely $1/\sigma$, and assign zero mass elsewhere. This yields a uniform distribution over an interval of length $1/\sigma$. Hence, among all $\sigma$-smooth distributions, there exists an optimal one of the form $\Unif([a,a+1/\sigma])$ for some $a \in [0,1-1/\sigma]$.
\end{proof}

\begin{lemma}[Monotone right-shifting structure] \label{lem:monotone-optimal}
Let $\mathcal{D}_t = \Unif([a_t,a_t+1/\sigma])$ denote the optimal interval choice at round $t$. Then there exists an optimal $r^*$-averse strategy such that $a_1 \le a_2 \le \cdots \le a_{Tk}$. Equivalently, the support of the chosen distribution moves weakly to the right over time.
\end{lemma}

\begin{proof}
Consider any optimal strategy expressed as a sequence of uniform interval distributions. Suppose that at some round $t$, the support moves to the left, that is, $a_t < a_{t-1}$. Since the adversary aims to extend the canonical path $r^*$, moving mass to the left cannot improve its ability to generate larger samples than those already realized along the path. In contrast, shifting the support weakly to the right can only improve, or leave unchanged, the chance of producing such samples. Therefore, replacing any leftward move by a nonnegative shift cannot decrease the adversary's objective. Repeating this adjustment whenever a leftward move occurs yields another optimal strategy whose supports are monotone nondecreasing over time. Thus, there exists an optimal strategy satisfying $a_1 \le a_2 \le \cdots \le a_{Tk}$.
\end{proof}

The two lemmas imply that the optimal $r^*$-averse strategy has a simple canonical form. It starts from the leftmost feasible interval and then shifts it weakly to the right over time. More precisely, the optimal strategy may be written as
\begin{align*}
\mathcal{D}_1 &:= \Unif\!\left(\left[0,\frac{1}{\sigma}\right]\right), \\
\mathcal{D}_t &= \mathcal{D}_{t-1} + \epsilon_t, \qquad t \ge 2,
\end{align*}
for some sequence of shifts $\epsilon_t \ge 0$.

The distinction between the adaptive and oblivious settings lies only in how these shifts are chosen. In the adaptive case, $\epsilon_t$ may depend on the observed history up to round $t-1$, in particular on the running maximum $\max\{x_1,\dots,x_{t-1}\}$ and the remaining horizon. In the oblivious case, the full sequence $(\epsilon_t)_{t=2}^{Tk}$ is fixed a priori.

An analogous structural property holds for any choice of canonical path. For a general path, the optimal policy adapts to its branching pattern, and the resulting evolution may not admit the rightward monotonicity description. Nevertheless, by symmetry among root-to-leaf paths, focusing on $r^*$---which always follows the right child---is without loss of generality and naturally yields a monotone rightward-shifting strategy.

We now analyze the height of the data structure under adaptive and oblivious $r^*$-averse adversaries, assuming the monotone optimal strategy $\pi_{r^*}$.

To make the dependence on the current maximum and the remaining horizon explicit, we formulate a \emph{Bellman optimality equation} for the choice of \(\mathcal{D}_t\) at each round. For analytical convenience, we first discretize the action space: instead of distributions supported on the continuous interval \([0,1]\), we consider distributions supported on a grid of \(N\) points. This converts the Bellman recursion into a standard finite dynamic programming problem. We analyze the discrete model and then lift the resulting insights and bounds to the original continuous setting.

\subsection{Adaptive $r^*$-averse Adversary}\label{sec:adaptiveadvh}
We now analyze the expected height of $\Tree^{(0)}$ under an $r^*$-averse adaptive adversary following the monotone strategy $\pi_{r^*}$.

Let $M_t := \max\{x_1,\dots,x_t\}$ denote the running maximum up to round $t$. Under the canonical form of the strategy, each distribution $\mathcal{D}_t$ is uniform over an interval of length $1/\sigma$, i.e.,
\[
\mathcal{D}_t = \Unif([a_t,\, a_t + \tfrac{1}{\sigma}]),
\]
where $0 = a_1 \le a_2 \le \cdots \le a_{Tk} \le 1 - \tfrac{1}{\sigma}$, and the shifts $a_t$ are chosen adaptively based on the history through $M_{t-1}$.

At round $t$, the tree height increases if and only if $x_t > M_{t-1}$. Since $x_t \sim \Unif([a_t, a_t + \tfrac{1}{\sigma}])$, the probability of this event depends on how much of the support lies to the right of $M_{t-1}$. Define
\[
p_t \doteq \bigl(a_t + \tfrac{1}{\sigma} - M_{t-1}\bigr)_+,
\]
which is the length of the portion of the support of $\mathcal{D}_t$ that lies above $M_{t-1}$. Then
\[
\Pr(x_t > M_{t-1} \mid M_{t-1}) = \sigma\, p_t.
\]

Under the $r^*$-averse reduction, the height of $\Tree^{(0)}$ is exactly the length of the canonical path $r^*$, so it is convenient to denote this quantity by $L_{Tk}$. Since each extension of $r^*$ corresponds to a new record, we have
\[
L_{Tk} = 1 + \sum_{t=2}^{Tk} \mathbf{1}\{x_t > M_{t-1}\}.
\]
Taking expectations and applying the law of total expectation yields
\begin{align}
\mathbb{E}[\EH(\Tree^{(0)})]
= \mathbb{E}[L_{Tk}]
= 1 + \sum_{t=2}^{Tk} \mathbb{E}\!\left[\Pr(x_t > M_{t-1} \mid M_{t-1})\right]
= 1 + \sum_{t=2}^{Tk} \sigma\, \mathbb{E}[p_t]. \label{eq:adaptiveeq1}
\end{align}

After the adversary selects $\mathcal{D}_t$, a sample $x_t \sim \mathcal{D}_t$ is drawn. If $x_t > M_{t-1}$, then the running maximum increases to $M_t = x_t$, and the interval $(M_{t-1},x_t]$ becomes inactive for all future rounds. Thus, each increase of the running maximum permanently removes part of the active range.

Since these increments are disjoint and contained in $[0,1]$, their total length is at most $1$. Equivalently, with the convention $M_0=0$,
\[
\sum_{t=1}^{Tk} (x_t-M_{t-1})_+ = M_{Tk} \le 1.
\]
Taking expectations yields
\[
\mathbb{E}\!\left[\sum_{t=1}^{Tk} (x_t-M_{t-1})_+\right] \le 1.
\]

Moreover, conditional on $M_{t-1}$, the portion of $\mathcal{D}_t$ lying to the right of $M_{t-1}$ has length $p_t$ and density $\sigma$. Therefore,
\[
\mathbb{E}\!\left[(x_t-M_{t-1})_+ \mid M_{t-1}\right]
= \int_0^{p_t} \sigma x\,dx
= \frac{\sigma p_t^2}{2}.
\]
Summing over $t$ and applying the law of total expectation gives
\begin{align}
\mathbb{E}\!\left[\sum_{t=1}^{Tk} (x_t-M_{t-1})_+\right]
&= \mathbb{E}\!\left[\sum_{t=1}^{Tk} \mathbb{E}\!\left[(x_t-M_{t-1})_+ \mid M_{t-1}\right]\right], \nonumber \\
&= \mathbb{E}\!\left[\sum_{t=1}^{Tk} \int_0^{p_t} \sigma x\,dx\right], \nonumber \\
&= \frac{1}{2}\,\mathbb{E}\!\left[\sum_{t=1}^{Tk} \sigma p_t^2\right] \le 1. \label{eq:adaptiveeq2}
\end{align}
Hence,
\(
\mathbb{E}\!\left[\sum_{t=1}^{Tk} \sigma p_t^2\right] \le 2.
\)

We now combine the equation~\eqref{eq:adaptiveeq1} and inequality~\eqref{eq:adaptiveeq2}. We have
\begin{align*}
\sum_{t=1}^{Tk}\sigma p_t
&\le
\sqrt{\left(\sum_{t=1}^{Tk}1\right)\left(\sum_{t=1}^{Tk}\sigma^2 p_t^2\right)} \\
&=
\sqrt{Tk}\,\sqrt{\sigma\sum_{t=1}^{Tk}\sigma p_t^2},
\end{align*}
where the inequality follows from Cauchy--Schwarz. Taking expectations, we obtain
\begin{align*}
\sum_{t=1}^{Tk}\sigma\,\mathbb{E}[p_t]
&=
\mathbb{E}\!\left[\sum_{t=1}^{Tk}\sigma p_t\right] \\
&\le
\sqrt{Tk}\,\mathbb{E}\!\left[\sqrt{\sigma\sum_{t=1}^{Tk}\sigma p_t^2}\right] \\
&\overset{(a)}\le
\sqrt{Tk}\,\sqrt{\sigma\,\mathbb{E}\!\left[\sum_{t=1}^{Tk}\sigma p_t^2\right]} \\
&\le
\sqrt{2\sigma Tk},
\end{align*}
where the inequality (a) follows from Jensen's inequality. Substituting into~\eqref{eq:adaptiveeq1} yields
\[
\mathbb{E}[\EH(\Tree^{(0)})]
\le
1+\sqrt{2\sigma Tk}
=
O(\sqrt{\sigma Tk}).
\]

\subsection{Oblivious $r^*$-averse Adversary}\label{sec:oblivioush}

We now turn to the oblivious $r^*$-averse adversary. In this setting, the distributions $\mathcal{D}_1,\dots,\mathcal{D}_{Tk}$ are fixed in advance, and only then are the samples $x_1,\dots,x_{Tk}$ drawn. Accordingly, the shifts $\epsilon_t$ are no longer functions of the running maximum $M_{t-1}$, but are instead chosen a priori.

This distinction changes the structure of the analysis. We therefore begin with the simpler case in which all supports $\mathcal{D}_t=[a_t,a_t+\frac{1}{\sigma}]$ have a nonempty common intersection, and then study a concrete example within this class that already exhibits the sharp $\Theta(\sqrt{Tk})$ behavior.

\paragraph{Intersecting supports.}
We say that the supports intersect if
\[
\bigcap_{t=1}^{Tk} \left[a_t,\,a_t+\frac{1}{\sigma}\right] \neq \emptyset.
\]
Assuming the strategy is monotone, the left endpoints satisfy
\[
0 = a_1 \le a_2 \le \cdots \le a_{Tk} \le 1-\frac{1}{\sigma}.
\]
Under this monotonicity, the intersection condition is equivalent to
\[
a_{Tk} \le \frac{1}{\sigma}.
\]

For such a sequence, the expected height of $\Tree^{(0)}$ admits an exact integral representation. Indeed,
\begin{align}
\mathbb{E}[\EH(\Tree^{(0)})]
&= \sum_{t=1}^{Tk} \Pr(x_t > M_{t-1}) \nonumber\\
&= \sum_{t=1}^{Tk} \int_{a_t}^{a_t+\frac{1}{\sigma}}
\sigma \,\Pr(M_{t-1} < x)\,dx \nonumber\\
&= \sum_{t=1}^{Tk} \int_{a_t}^{a_t+\frac{1}{\sigma}}
\sigma \,\prod_{t'=1}^{t-1}\Pr(x_{t'} < x)\,dx \nonumber\\
&= \sum_{t=1}^{Tk} \int_{0}^{1}
\prod_{t'=1}^{t-1}\Pr\!\left(x_{t'} < a_t+\frac{u}{\sigma}\right)\,du.
\label{eq:obliviouseq1}
\end{align}
Here we used the fact that, once the oblivious adversary fixes the intervals, the samples are independent.

The expression inside the product can be computed explicitly. For every $t'<t$, since $a_t\ge a_{t'}$, we have
\[
\Pr\!\left(x_{t'} < a_t+\frac{u}{\sigma}\right)
=
\min\!\left\{1,\,u+\sigma(a_t-a_{t'})\right\}.
\]
Substituting this into~\eqref{eq:obliviouseq1} yields
\begin{align}
\mathbb{E}[\EH(\Tree^{(0)})]
&= \sum_{t=1}^{Tk} \int_0^1 \prod_{t'=1}^{t-1}
\min\!\left\{1,\,u+\sigma(a_t-a_{t'})\right\}\,du.
\label{eq:oblivious-min-form}
\end{align}

To upper bound this expression, fix $t$ and $u$, and define the active set
\[
\mathcal{B}_t(u)
\doteq
\left\{t'<t:\ \sigma(a_t-a_{t'}) < 1-u\right\}.
\]
For indices $t'\notin \mathcal{B}_t(u)$, the quantity $u+\sigma(a_t-a_{t'})$ is at least $1$, so the corresponding factor in the product is exactly $1$. Therefore only the active indices matter, and
\[
\prod_{t'=1}^{t-1}
\min\!\left\{1,\,u+\sigma(a_t-a_{t'})\right\}
=
\prod_{t' \in \mathcal{B}_t(u)}
\left(u+\sigma(a_t-a_{t'})\right).
\]

We now bound this product exponentially. Since every active factor lies in $[0,1]$, the inequality
\[
y \le e^{-(1-y)}
\qquad\text{for all } y \in [0,1]
\]
implies
\[
\prod_{t' \in \mathcal{B}_t(u)}
\left(u+\sigma(a_t-a_{t'})\right)
\le
\exp\!\left(
-\sum_{t' \in \mathcal{B}_t(u)}
\left(1-u-\sigma(a_t-a_{t'})\right)
\right).
\]
Defining
\[
R_t(u)
\doteq
\sum_{t' \in \mathcal{B}_t(u)}
\left(1-u-\sigma(a_t-a_{t'})\right),
\]
we obtain
\begin{align}
\mathbb{E}[\EH(\Tree^{(0)})]
&\le
\sum_{t=1}^{Tk}\int_0^1 e^{-R_t(u)}\,du.
\label{eq:oblivious-final-bound}
\end{align}

We next instantiate this bound on a concrete monotone example:
\[
a_t=\frac{t-1}{\sigma Tk},
\qquad t=1,\dots,Tk.
\]
This sequence clearly satisfies
\[
a_{Tk}=\frac{Tk-1}{\sigma Tk}\le \frac{1}{\sigma},
\]
so all supports intersect. Moreover,
\[
\sigma(a_t-a_{t'})=\frac{t-t'}{Tk}.
\]
Hence, for fixed $t$ and $u$, the active set becomes
\[
\mathcal{B}_t(u)
=
\left\{t'<t:\frac{t-t'}{Tk}<1-u\right\},
\]
which is equivalent to
\[
t' > t-Tk(1-u).
\]
If we let
\[
N_t(u)\doteq |\mathcal{B}_t(u)|,
\]
then
\[
N_t(u)=\min\!\left\{t-1,\left\lfloor Tk(1-u)\right\rfloor\right\}.
\]
In this case, the active gaps are precisely
\[
\frac{1}{Tk},\frac{2}{Tk},\dots,\frac{N_t(u)}{Tk},
\]
so the exponent takes the explicit form
\begin{align}
R_t(u)
&=
\sum_{m=1}^{N_t(u)}\left(1-u-\frac{m}{Tk}\right) \nonumber\\
&=
N_t(u)(1-u)-\frac{N_t(u)(N_t(u)+1)}{2Tk}.
\label{eq:linear-Rt}
\end{align}
Substituting this into~\eqref{eq:oblivious-final-bound}, we get
\begin{align}
\mathbb{E}[\EH(\Tree^{(0)})]
&\le
\sum_{t=1}^{Tk}\int_0^1
\exp\!\left(
-N_t(u)(1-u)+\frac{N_t(u)(N_t(u)+1)}{2Tk}
\right)\,du.
\label{eq:linear-upper}
\end{align}

We now show that the right-hand side of~\eqref{eq:linear-upper} is \(O(\sqrt{Tk})\). The case \(t=1\) is trivial, since \(\Pr(x_1>M_0)=1\). So fix \(t\ge 2\). We split the integral into two regions according to whether \(N_t(u)=t-1\) or \(N_t(u)=\lfloor Tk(1-u)\rfloor\).

On the first region,
\[
0\le u\le 1-\frac{t-1}{Tk},
\]
we have \(N_t(u)=t-1\), and therefore
\[
R_t(u)
=
(t-1)(1-u)-\frac{t(t-1)}{2Tk}.
\]
Thus,
\begin{align}
\int_0^{1-\frac{t-1}{Tk}} e^{-R_t(u)}\,du
&=
\int_0^{1-\frac{t-1}{Tk}}
\exp\!\left(
-(t-1)(1-u)+\frac{t(t-1)}{2Tk}
\right)\,du.
\end{align}
Making the change of variables \(z=1-u\), this becomes
\begin{align}
e^{\frac{t(t-1)}{2Tk}}
\int_{\frac{t-1}{Tk}}^{1} e^{-(t-1)z}\,dz
&\le
e^{\frac{t(t-1)}{2Tk}}
\int_{\frac{t-1}{Tk}}^{\infty} e^{-(t-1)z}\,dz \nonumber\\
&=
\frac{1}{t-1}
\exp\!\left(
-\frac{(t-1)(t-2)}{2Tk}
\right).
\label{eq:linear-upper-part1}
\end{align}

On the second region,
\[
1-\frac{t-1}{Tk}\le u\le 1,
\]
we have
\[
N_t(u)=\lfloor Tk(1-u)\rfloor.
\]
Letting \(v=1-u\), we obtain \(0\le v\le \frac{t-1}{Tk}\) and
\[
N_t(u)=\lfloor Tkv\rfloor.
\]
Using~\eqref{eq:linear-Rt}, we can write
\[
R_t(u)
=
\lfloor Tkv\rfloor\,v
-
\frac{\lfloor Tkv\rfloor(\lfloor Tkv\rfloor+1)}{2Tk}.
\]
Since \(\lfloor Tkv\rfloor \ge Tkv-1\), it follows that
\[
R_t(u)\ge \frac{Tk}{2}v^2-v.
\]
Therefore,
\begin{align}
\int_{1-\frac{t-1}{Tk}}^1 e^{-R_t(u)}\,du
&\le
\int_0^1 \exp\!\left(-\frac{Tk}{2}v^2+v\right)\,dv \nonumber\\
&\le
e\int_0^1 e^{-\frac{Tk}{2}v^2}\,dv
=
O\!\left(\frac{1}{\sqrt{Tk}}\right).
\label{eq:linear-upper-part2}
\end{align}

Combining the two regions, we conclude that for every \(t\ge 2\),
\[
\int_0^1 e^{-R_t(u)}\,du
\le
\frac{1}{t-1}
\exp\!\left(
-\frac{(t-1)(t-2)}{2Tk}
\right)
+
O\!\left(\frac{1}{\sqrt{Tk}}\right).
\]
Summing this bound over \(t=2,\dots,Tk\), the first term contributes at most \(O(\log Tk)\), while the second contributes \(O(\sqrt{Tk})\). Together with the \(t=1\) term, this proves
\[
\mathbb{E}[\EH(\Tree^{(0)})]=O(\sqrt{Tk}).
\]

We next show that this order is attained by the same example. Returning to the exact formula~\eqref{eq:oblivious-min-form}, fix any \(t\ge 2\sqrt{Tk}\) and restrict the integral to the interval
\[
u\in\left[1-\frac{1}{\sqrt{Tk}},\,1\right].
\]
If \(t'<t-\sqrt{Tk}\), then
\[
u+\sigma(a_t-a_{t'})
=
u+\frac{t-t'}{Tk}
\ge
1-\frac{1}{\sqrt{Tk}}+\frac{\sqrt{Tk}}{Tk}
=1,
\]
so the corresponding factor is exactly \(1\). For the remaining at most \(\sqrt{Tk}\) indices, each factor is at least
\[
u\ge 1-\frac{1}{\sqrt{Tk}}.
\]
Hence,
\begin{align}
\Pr(x_t>M_{t-1})
&=
\int_0^1
\prod_{t'=1}^{t-1}
\min\!\left\{1,u+\frac{t-t'}{Tk}\right\}\,du \nonumber\\
&\ge
\int_{1-1/\sqrt{Tk}}^1
\left(1-\frac{1}{\sqrt{Tk}}\right)^{\sqrt{Tk}}\,du \nonumber\\
&=
\frac{1}{\sqrt{Tk}}
\left(1-\frac{1}{\sqrt{Tk}}\right)^{\sqrt{Tk}}
=
\Omega\!\left(\frac{1}{\sqrt{Tk}}\right).
\end{align}
Summing this lower bound over all \(t=2\sqrt{Tk},\dots,Tk\), we obtain
\[
\mathbb{E}[\EH(\Tree^{(0)})]=\Omega(\sqrt{Tk}).
\]

We have therefore shown that for the example
\[
a_t=\frac{t-1}{\sigma Tk},
\]
the expected height satisfies
\[
\mathbb{E}[\EH(\Tree^{(0)})]=\Theta(\sqrt{Tk}).
\]
Since oblivious adversaries are a subclass of adaptive adversaries, the adaptive upper bound \(O(\sqrt{Tk})\) also applies to the oblivious optimum. Combined with the above construction, this shows that the optimal oblivious \(r^*\)-averse adversary also achieves \(\Theta(\sqrt{Tk})\).

\subsection{Oblivious Random-Order Adversary}\label{sec:randord}
This follows from the fact that hyperplanes \(\hp_t\) arrive in uniformly random order within each direction \(\ax_i \in \Ax\), and layers evolve independently. The bound is implied by the classical \emph{Binary Search Tree} result:
\begin{theorem}[\cite{pittel1984growing}]
Let \(x_1, \dots, x_m\) be i.i.d.\ samples inserted into an initially empty Binary Search Tree. The expected height is \(O(\log m)\), with variance \(O(1)\).
\end{theorem}
\section{Warm-Up: Example Insertion and Sampling in $2$-Dimensions (\Cref{section:ds-qhandling})} \label{section:ds-qhandlingapp} 
To illustrate the use of the structure, consider the warm-up case where two hyperplanes arrive at time \(t\):  
\(\hp_{1,t} : x^{(0)} = \bp^{(0)}_t\) and \(\hp_{2,t} : x^{(1)} = \bp^{(1)}_t\).  
These partition the domain \([0,1)^2\) into four rectangular regions \(\rg_{i,t}\), each assigned a $\FP$ \(\fu_t(\rg_{i,t})\). The regions are:

\[
\begin{array}{ll}
\rg_{1,t}: [0{:}\bp^{(0)}_t][\bp^{(1)}_t{:}1],\; \fu_t(\rg_{1,t}) \quad & 
\rg_{2,t}: [\bp^{(0)}_t{:}1][\bp^{(1)}_t{:}1],\; \fu_t(\rg_{2,t}) \\
\rg_{3,t}: [0{:}\bp^{(0)}_t][0{:}\bp^{(1)}_t],\; \fu_t(\rg_{3,t}) \quad & 
\rg_{4,t}: [\bp^{(0)}_t{:}1][0{:}\bp^{(1)}_t],\; \fu_t(\rg_{4,t})
\end{array}
\]

Each region \(\rg_{i,t}\) is inserted by identifying all \(v_0 \in \Tree^{(0)}\) whose horizontal range \(\RA(v_0)\) intersects the horizontal interval of \(\rg_{i,t}\). Depending on the nature of the intersection:
\begin{itemize}
    \item [-] If \(\RA(v_0)\) is \emph{intersects}, the vertical interval is inserted into \(\Tree^{(1)}_{v_0}\) (regular insertion tree).
    \item [-] If \(\RA(v_0)\) is \emph{fully contained} in the region's horizontal interval, the update is deferred to \(\LZL^{(1)}_{v_0}\) or \(\LZR^{(1)}_{v_0}\), depending on the subtree alignment.
\end{itemize}

In all cases, the cumulative reward vector \(\fu_t(\rg_{i,t})\) is added to the corresponding node in the second-level tree, ensuring each rectangular region contributes correctly to later queries.

\paragraph{Handling \(\Ins(\rg_{1,t})\) Query:}
We insert the region \(\rg_{1,t} = [0{:}\bp^{(0)}_t][\bp^{(1)}_t{:}1]\) following the hierarchical structure of \(\Tree\). First, we refine the first-level tree \(\Tree^{(0)}\) to align with the endpoints \(0\) and \(\bp^{(0)}_t\), by locating the leaves containing each and splitting the corresponding leaves if needed. Let \(l_0\) and \(u_0\) denote the resulting leaves. We then define the upward paths \(p_{l_0}\) and \(p_{u_0}\) to the root and compute the lowest common ancestor \(v_{l_0,u_0} = \FCA(l_0, u_0)\). All nodes in \(p_{l_0} \cup p_{u_0}\) have horizontal ranges intersecting \([0{:}\bp^{(0)}_t]\), and the subtree of \(v_{l_0,u_0}\) contains all nodes fully contained in the interval.

We insert the vertical interval \([\bp^{(1)}_t{:}1]\) into each second-level tree according to the position of the corresponding \(v_0 \in \Tree^{(0)}\). For all \(v_0 \in p_{l_0} \cup p_{u_0}\), the insertion is regular into \(\Tree^{(1)}_{v_0}\) because the $\RA(v_0)$ intersects the horizontal interval of the inserted region. To identify lazy insertions, we consider nodes \(v_0 \in p_{l_0} \setminus p_{v_{l_0,u_0}}\) whose right child is not on that path, and nodes \(v_0 \in p_{u_0} \setminus p_{v_{l_0,u_0}}\) whose left child is not on that path. These are the roots of sibling subtrees entirely contained in \([0{:}\bp^{(0)}_t]\) and thus suitable for lazy updates.

Formally, we insert into \(\LZR^{(1)}_{v_0}\) if
\[
v_0 \in p_{l_0} \setminus p_{v_{l_0,u_0}} \quad \text{and} \quad \RC(v_0) \notin p_{l_0} \setminus p_{v_{l_0,u_0}},
\]
and into \(\LZL^{(1)}_{v_0}\) if
\[
v_0 \in p_{u_0} \setminus p_{v_{l_0,u_0}} \quad \text{and} \quad \LC(v_0) \notin p_{u_0} \setminus p_{v_{l_0,u_0}}.
\]

Lazy insertions defer propagation of contributions to query time, allowing efficient updates without redundant insertions across multiple second-level structures. This process naturally extends to higher dimensions by recursively inserting projections along each axis across the tree layers.

Specifically, when inserting \([\bp^{(1)}_t{:}1]\) into a regular insertion tree \(\Tree^{(1)}_{v_0}\), we first locate the leaves that contain the endpoints \(\bp^{(1)}_t\) and \(1\), and refine them by splitting their intervals. This introduces two new child nodes per leaf, converting each into an internal node and designating the new children as leaves. We refer to the resulting leaves containing \(\bp^{(1)}_t\) and \(1\) as \(l_1\) and \(u_1\), respectively. We then define the upward paths \(p_{l_1}\) and \(p_{u_1}\) from \(l_1\) and \(u_1\) to the root of \(\Tree^{(1)}_{v_0}\), and identify their first common ancestor \(v_{l_1,u_1} = \FCA(l_1, u_1)\). We traverse $p_{l_1} \cup p_{u_1}$ bottom-up to update the cumulative reward vectors, as described next. The procedure for inserting into lazy insertion data structures follows the same logic, with one important distinction: the update to the cumulative reward vectors differs slightly as detailed in the next section.

\paragraph{Handling \(\Update(\rg_{1,t})\) Query:}  
Once \(\rg_{1,t}\) is inserted, we update the cumulative reward vectors in all second-level structures where its vertical interval \([\bp^{(1)}_t{:}1]\) was inserted—either regularly or lazily—using with the $\FP$ vector \(\fu_t(\rg_{1,t})\). This update follows the standard \emph{regular} and \emph{lazy} weight update mechanisms for interval trees \citep{cohen2017online}, extended to vector-valued cumulative rewards as \emph{regular} and \emph{lazy \(\FP\) updates}. Unlike lazy insertions, which defer interval insertions, lazy \(\FP\) updates apply to reward vectors in second-level trees. While the update routine is structurally similar across regular and lazy insertion trees with minor implementation differences depending on the data structure type.

\begin{itemize}
    \item \textbf{Updating Regular Insertion $\IT$s:} 
For each node \(v_0 \in p_{l_0} \cup p_{u_0}\) where regular insertion occurred, we update the cumulative reward vector \(\vec{\CR}_{v_0,v_1}\) for every node \(v_1 \in \Tree^{(1)}_{v_0}\) whose vertical range \(\RA(v_1)\) intersects the interval \([\bp^{(1)}_t{:}1]\). The update adds the contribution of the region $\rg_{1,t}$ with its $\FP$ \(\fu_t(\rg_{1,t})\) over the intersected subregion 
\(
\rg' = [\RA(v_0) \cap [0{:}\bp^{(0)}_t]][\RA(v_1) \cap [\bp^{(1)}_t{:}1]],
\)
which captures the portion of \(\rg_{1,t}\) aligned with the rectangular region defined by \(v_0\) and \(v_1\).

To implement this, we first locate the leaves \(l_1\) and \(u_1\) in \(\Tree^{(1)}_{v_0}\) that contain the endpoints \(\bp^{(1)}_t\) and \(1\), respectively. We then traverse the paths \(p_{l_1}\) and \(p_{u_1}\) to the root, and for each node \(v_1\) along \(p_{l_1} \cup p_{u_1}\), we perform the update:
\[
\vec{\CR}_{v_0,v_1} \gets \vec{\CR}_{v_0,v_1} + \vec{\SC}_{\rg'},
\]
where \(\vec{\SC}_{\rg'}\) is the vector representation of the cumulative reward over \(\rg'\),i.e $\vec{\CR}(\rg')$ with the $\fu_t(\rg_{1,t})$ as the coefficient vector.

Now, similar to our general strategy for efficient updates, we apply a lazy propagation scheme to vertical intervals in \(\Tree^{(1)}_{v_0}\), specifically targeting those subtrees whose vertical range \(\RA(v_1)\) is fully contained within the vertical interval of the region \(\rg_{1,t}\). This follows the same logic as the lazy insertion mechanism.

To implement this, we traverse the truncated paths \(p_{l_1} \setminus p_{v_{l_1,u_1}}\) and \(p_{u_1} \setminus p_{v_{l_1,u_1}}\). During each traversal, we propagate the update message \(\fu_t(\rg_{1,t})\), referred to as the \emph{lazy \(\FP\) update message}, to specific children of nodes along the path, provided those children are not themselves on the path. The cumulative lazy updates received by a node \(v_1\) from its parent in \(\Tree^{(1)}_{v_0}\) are stored in \(\LZM_{v_0,v_1}\).

In particular, for each node \(v_1 \in p_{l_1} \setminus p_{v_{l_1,u_1}}\), we send a lazy \(\FP\) update message to its right child \(\RC(v_1)\) if
\(
\RC(v_1) \notin p_{l_1} \setminus p_{v_{l_1,u_1}},
\)
by setting
\[
\LZM_{v_0, \RC(v_1)} \gets \LZM_{v_0, \RC(v_1)} + \fu_t(\rg_{1,t}).
\]

Likewise, for each node \(v_1 \in p_{u_1} \setminus p_{v_{l_1,u_1}}\), we send a lazy \(\FP\) update message to its left child \(\LC(v_1)\) if
\(
\LC(v_1) \notin p_{u_1} \setminus p_{v_{l_1,u_1}},
\)
by setting
\[
\LZM_{v_0, \LC(v_1)} \gets \LZM_{v_0, \LC(v_1)} + \fu_t(\rg_{1,t}).
\]

These lazy update messages accumulated at each node encode the \(\FP\)s of all regions whose contributions affect the entire subtree rooted at either the left or right child. Since these contributions are additive, we can aggregate the corresponding vectors and defer integration until propagation, ensuring correctness and efficiency in handling lazy updates.

\item \textbf{Updating Lazy Insertion $\IT$s:} 

For lazy insertions, the update routine is identical for both \(\LZL^{(1)}_{v_0}\) and \(\LZR^{(1)}_{v_0}\). As in the case of regular insertion trees, updating a lazy insertion tree involves locating the leaves \(l_1\) and \(u_1\), and traversing the nodes \(w_1 \in p_{l_1} \cup p_{u_1}\) to update their cumulative reward vectors by integrating \(\fu_t(\rg_{1,t})\) over appropriate rectangles \(\rg''\). Specifically, for each such node we perform the update:
\[
\vec{\CR}_{v_0,w_1} \gets \vec{\CR}_{v_0,w_1} + \vec{\SC}_{\rg''},
\]
where \(\rg'' = [\RA(v_0)][\RA(w_1) \cap [\bp^{(1)}_t{:}1]]\) denotes the subregion over which lazy integration is performed. Unlike the regular insertion case, the horizontal component of the integration is not restricted to the intersection with the horizontal interval of \(\rg_{1,t}\); rather, it is fixed to \(\RA(v_0)\), since all affected subregions are fully contained within this range. This ensures that the contributions stored in lazy structures are correctly computed and remain scalable for future query processing.

Lazy $\FP$ update messages are also handled in the same manner as in regular insertion trees. That is, we send \(\fu_t(\rg_{1,t})\) as a lazy message \(\LZM\) to those nodes whose vertical range is fully contained within the vertical interval of \(\rg_{1,t}\).

\end{itemize}

In both regular and lazy insertion structures, lazy update messages are propagated to their corresponding subtrees during leaf location in each second-level interval tree. Specifically, while locating a leaf node \(v_1\) in a second-level tree associated with a node \(v_0 \in \Tree^{(0)}\), the algorithm retrieves the corresponding lazy message \(\LZM_{v_0,v_1}\), which encodes an \(\FP\) vector. It then integrates this vector over the appropriate region and propagates the resulting update downward through the subtree rooted at \(v_1\). The full procedure is detailed in Appendix~\ref{sec:update}.

Next we discuss the structural properties of our data structure that allows efficient sampling.

\subsection{Structural Properties of $\Tree$}\label{sec:structure}

\paragraph{Main Structure:}
The following observation clarifies the properties maintained during insertion, specifically how a region \(\rg\) is handled by the data structure upon insertion.
\begin{observation}\label{obs:ins}
After inserting the region \(\rg = [x^{(0)}_1{:}x^{(0)}_2][x^{(1)}_1{:}x^{(1)}_2]\), the vertical interval \([x^{(1)}_1{:}x^{(1)}_2]\) is inserted into the second-level data structures as follows:
\begin{enumerate}
    \item It is \emph{regularly inserted} into each regular insertion tree \(\Tree^{(1)}_{v_0}\), where \(v_0 \in p_{l_0} \cup p_{u_0}\) and the horizontal range \(\RA(v_0)\) intersects \([x^{(0)}_1{:}x^{(0)}_2]\).

    \item It is \emph{lazily inserted} into the corresponding lazy insertion tree \(\LZL^{(1)}_{v_0}\) or \(\LZR^{(1)}_{v_0}\) for each node \(v_0 \in (p_{l_0} \cup p_{u_0}) \setminus p_{v_{l_0,u_0}}\), such that the range of its respective child—\(\RA(\RC(v_0))\) or \(\RA(\LC(v_0))\)—is fully contained in \([x^{(0)}_1{:}x^{(0)}_2]\).

    \item No other node in \(\Tree^{(0)}\) maintains cumulative reward information for any region intersecting \(\rg\) except those that are regularly or lazily affected during the insertion of \(\rg\).
\end{enumerate}
\end{observation}

In the following, we describe how our hierarchical data structure supports online learning algorithms in drawing an action at each round with the correct probability. This is achieved by first identifying an atomic region with probability proportional to its cumulative reward, as maintained in the data structure.

\paragraph{Uniqueness and Completeness of the Information in the Data Structures Along the Paths:}

Let \(v_0 \in \Tree^{(0)}\) be a node. We construct the set \(\mathcal{S}_{v_0}^{(1)}\), which includes the regular insertion tree \(\Tree^{(1)}_{v_0}\) as well as the corresponding lazy data structures of all ancestors of \(v_0\) along the path \(p_{v_0}\) toward the root. Specifically, for each ancestor \(v_0'\) of \(v_0\), we include \(\LZL^{(1)}_{v_0'}\) in \(\mathcal{S}_{v_0}^{(1)}\) if \(v_0\) lies in the left subtree of \(v_0'\), and \(\LZR^{(1)}_{v_0'}\) otherwise. This construction ensures that, for each node \(v_0' \in p_{v_0}\), exactly one second-level data structure—either regular or lazy—is included in \(\mathcal{S}_{v_0}^{(1)}\), which we denote by \(S^{(1)}_{v_0'}\).

The following Lemmas~\ref{lemma:unique} and~\ref{lemma:complete} together guarantee that: (i) every region intersecting \(\RA(v_0)\) and influencing it is covered by some structure in \(\mathcal{S}_{v_0}^{(1)}\); and (ii) each such region is stored in exactly one of these structures.

\begin{lemma}[Uniqueness of Endpoints Across \(\mathcal{S}^{(1)}_{v_0}\)]\label{lemma:unique}
Let \(v_0 \in \Tree^{(0)}\) be a fixed node, and let \(\mathcal{S}^{(1)}_{v_0}\) denote the set of second-level data structures constructed along the path \(p_{v_0}\) from \(v_0\) to the root, as defined in Algorithm~\ref{alg:crr1}. Then, for any region whose vertical interval has been inserted into one of these structures, this interval appears in exactly one data structure in \(\mathcal{S}^{(1)}_{v_0}\). That is, no two data structures \(S_{v_0'}, S_{v_0''} \in \mathcal{S}^{(1)}_{v_0}\) maintain vertical intervals at their leaves that share a common endpoint.
\end{lemma}

\begin{lemma}[Completeness of Endpoints Across \(\mathcal{S}^{(1)}_{v_0}\)]\label{lemma:complete}
For a given node \(v_0 \in \Tree^{(0)}\), let \(\mathcal{S}^{(1)}_{v_0}\) denote the set of data structures constructed along the path \(p_{v_0}\) from the root \(\ROOT(\Tree^{(0)})\) to \(v_0\), as defined in Algorithm~\ref{alg:crr1}. Then, for every region inserted up to round \(t\), if its horizontal interval intersects the horizontal range \(\RA(v_0)\), its vertical interval is contained in at least one data structure \(S_{v_0'} \in \mathcal{S}^{(1)}_{v_0}\). In other words, the union of all data structures in \(\mathcal{S}^{(1)}_{v_0}\) covers all vertical intervals relevant to \(v_0\) that have been inserted so far.
\end{lemma}

We now describe how our drawing algorithm operates, assuming all required structural properties established thus far.

\subsection{Drawing \(x_t\) via \(\Draw(t)\)}\label{sec:exampdraw}

As described above, the \(\Draw(t)\) query samples an atomic region \(\rg = [\bp^{(0)}_{t_1}:\bp^{(0)}_{t_2}][\bp^{(1)}_{t_1}:\bp^{(1)}_{t_2}] \in \Fu_t\) with probability proportional to its cumulative reward:

\[
\Pr(\rg \gets \Draw(t)) = \frac{\CR(\rg)}{\sum_{\rg' \in \Fu_t} \CR(\rg')},
\]
where $\Pr(\rg \gets \Draw(t))$ is the probability of the drawing algorithm to return $\rg$.

The sampling process consists of two stages: \textit{Horizontal Sampling} followed by \textit{Vertical Sampling}. In the horizontal sampling stage, the drawing algorithm selects an atomic horizontal interval \([\bp^{(0)}_{t_1} : \bp^{(0)}_{t_2}]\) by randomly traversing a root-to-leaf path in the auxiliary interval tree \(\ATree^{(0)}\), which encodes the horizontal structure of the action space. This traversal identifies a rectangular region whose horizontal interval is atomic and whose vertical interval spans the full range $[0:1]$ of the second coordinate of the action space. In the subsequent vertical sampling stage, the algorithm selects an atomic vertical interval \([\bp^{(1)}_{t_1} : \bp^{(1)}_{t_2}]\) within this subspace by traversing the corresponding auxiliary tree \(\ATree^{(1)}\). The combination of the selected atomic horizontal and vertical intervals uniquely determines an atomic region. At each stage, a cumulative reward retrieval query is used to compute the transition probabilities for navigating left or right within the auxiliary trees to get to the atomic intervals stored at the leaves at each stage.
After sampling region $\rg$ with the correct probability, this algorithm samples $x_t$ within that region proportional to its cumulative reward using \emph{Hit-and-Run} sampling algorithm.

\paragraph{Cumulative Reward Retrieval $\Wei(.)$ Query:}
This query is responsible for retrieving the cumulative reward of the queried region while drawing. Although the \(\Wei\) routine can handle arbitrary rectangular queries, the \(\Draw\) algorithm only invokes it for two specific types of structured regions, denoted by \(\Wei^{(0)}\) and \(\Wei^{(1)}\). The query \(\Wei^{(0)}(\rg_1: [x^{(0)}_1:x^{(0)}_2][0:1])\), used during horizontal sampling, retrieves the cumulative reward over a rectangular region with horizontal interval \([x^{(0)}_1:x^{(0)}_2]\) and full vertical range \([0:1]\). The endpoints of this horizontal interval are assumed to belong to the set \(\bp^{(0)}_{t'}\) for some \(t' \in [t]\), as defined by the adversary up to round \(t\). In contrast, the query \(\Wei^{(1)}(\rg_2: [x'^{(0)}_1:x'^{(0)}_2][x^{(1)}_1:x^{(1)}_2])\), used during vertical sampling, retrieves the cumulative reward over a region whose horizontal interval \([x'^{(0)}_1:x'^{(0)}_2]\) corresponds to an atomic segment (i.e., a leaf in \(\Tree^{(0)}\)), and whose vertical interval \([x^{(1)}_1:x^{(1)}_2]\) has endpoints in \(\bp^{(1)}_{t'}\) for \(t' \in [t]\). Both types of queries are efficiently supported by the hierarchical interval tree \(\Tree\), including its regular and lazy insertion structures.

To compute \(\Wei^{(0)}(\rg_1: [x^{(0)}_1:x^{(0)}_2][0:1])\), the algorithm first identifies the leaves \(l_0\) and \(u_0\) in \(\Tree^{(0)}\) that contain \(x^{(0)}_1\) and \(x^{(0)}_2\), respectively, and determines their first common ancestor \(v_{l_0,u_0}\). By construction, the horizontal interval \([x^{(0)}_1:x^{(0)}_2]\) is fully contained in \(\RA(v_{l_0,u_0})\). The algorithm then retrieves the cumulative reward over the rectangular region \([\RA(v_{l_0,u_0})][0:1]\) and prunes away contributions from subregions that fall outside the query region \(\rg_1\).

Thanks to the structural guarantees provided by Lemmas~\ref{lemma:unique} and~\ref{lemma:complete}, the cumulative reward over \([\RA(v_{l_0,u_0})][0:1]\) is distributed across multiple second-level data structures, including both regular and lazy insertion trees associated with nodes whose horizontal ranges intersect or contain \(\RA(v_{l_0,u_0})\). By Lemma~\ref{lemma:unique}, each contributing region is stored in exactly one such data structure, allowing the total to be computed by aggregating disjoint contributions from the set \(\mathcal{S}^{(1)}_{v_{l_0,u_0}}\), which includes:

\begin{itemize}
    \item The cumulative reward vector \(\vec{\CR}_{v_{l_0,u_0}, \ROOT(\Tree^{(1)}_{v_{l_0,u_0}})}\), stored at the root of the regular insertion tree \(\Tree^{(1)}_{v_{l_0,u_0}}\), which encodes the contribution of all regions whose vertical intervals intersect \(\RA(v_{l_0,u_0})\);
    \item Scaled cumulative reward vectors stored at the roots of the lazy insertion trees \(\LZL^{(1)}_{v_0}\) or \(\LZR^{(1)}_{v_0}\), where each \(v_0 \in p_{v_{l_0,u_0}}\) is an ancestor whose horizontal range \(\RA(v_0)\) fully contains \(\RA(v_{l_0,u_0})\), and whose lazy subtree therefore contributes to the query region.
\end{itemize}

The algorithm then prunes the cumulative reward contributions corresponding to portions of the region \([\RA(v_{l_0,u_0})][0:1]\) that fall outside the query rectangle \(\rg: [x^{(0)}_1:x^{(0)}_2][0:1]\). This pruning is performed in two symmetric passes. First, it traverses the nodes \(v_0 \in p_{l_0}\). For each such node, if the horizontal range of its left child \(\RA(\LC(v_0))\) lies entirely outside the interval \([x^{(0)}_1:x^{(0)}_2]\), the algorithm subtracts the cumulative reward over the region \([\RA(\LC(v_0))][0:1]\). This quantity is retrieved via a recursive call to \(\Wei^{(0)}([\RA(\LC(v_0))][0:1])\), where no further pruning is necessary. A symmetric pruning pass is applied along the path \(p_{u_0}\). For each node \(v_0 \in p_{u_0}\), if the horizontal range of its right child \(\RA(\RC(v_0))\) does not intersect \([x^{(0)}_1:x^{(0)}_2]\), the cumulative reward over the region \([\RA(\RC(v_0))][0:1]\) is subtracted. This procedure ensures that only the contribution from the queried region is retained in the final result. The complete implementation of the cumulative reward query \(\Wei^{(0)}\) is detailed in Section~\ref{Sec:Wei0}.

The query \(\Wei^{(1)}(\rg_2:[x'^{(0)}_1:x'^{(0)}_2][x^{(1)}_1:x^{(1)}_2])\) is handled in a manner similar to \(\Wei^{(0)}\), with the main distinction arising from the structure of the queried region \(\rg_2\), whose horizontal interval \([x'^{(0)}_1:x'^{(0)}_2]\) is assumed to be atomic. Consequently, both endpoints of the horizontal interval are contained within a single leaf \(l_0 \in \Tree^{(0)}\), such that \(\RA(l_0) = [x'^{(0)}_1:x'^{(0)}_2]\).

Given this atomic interval, the algorithm constructs the set \(\mathcal{S}^{(1)}_{l_0}\), consisting of the regular insertion tree \(\Tree^{(1)}_{l_0}\) and the lazy insertion trees associated with the ancestors of \(l_0\). Together, these structures maintain the vertical subdivision necessary for handling the query region \(\rg_2\). 

Unlike the \(\Wei^{(0)}\) case, however, the vertical interval \([x^{(1)}_1:x^{(1)}_2]\) does not span the full range \([0:1]\), and its endpoints may not align with any node boundaries in the second-level trees. This is due to the deferred insertion strategy used in maintaining the lazy insertion data structures: vertical intervals are incorporated only through scaling, and their endpoints may be distributed across multiple second-level trees. 

As a result, each data structure in \(\mathcal{S}^{(1)}_{l_0}\) must be traversed to locate the leaves \(l_1\) and \(u_1\) containing \(x^{(1)}_1\) and \(x^{(1)}_2\), respectively. The algorithm then computes their first common ancestor—denoted \(v_{l_1,u_1}\) in regular insertion trees and \(w_{l_1,u_1}\) in lazy ones. It retrieves the cumulative reward over the region \([\RA(l_0)][\RA(v_{l_1,u_1})]\) (or \([\RA(l_0)][\RA(w_{l_1,u_1})]\)) and prunes it to match the exact query rectangle \(\rg_2\).

Since the endpoints \(x^{(1)}_1\) and \(x^{(1)}_2\) may lie in the interior of the ranges \(\RA(l_1)\) and \(\RA(u_1)\), additional leaf-level pruning and scaling is required. In lazy insertion trees, this is handled correctly by construction, as guaranteed by Corollary~\ref{corr:lazleaf}. For the regular insertion tree \(\Tree^{(1)}_{l_0}\), correctness is ensured by Lemma~\ref{lemma:regleaf}, which applies specifically because \(l_0\) is a leaf in \(\Tree^{(0)}\). In essence, once we are assured that \(l_0\) is a leaf, our update mechanism—particularly the way lazy \(\FP\) update messages are propagated—allows us to treat \(\Tree^{(1)}_{l_0}\) similarly to a lazy insertion data structure. As a result, the same scalability and pruning guarantees hold. The full procedure for answering the cumulative reward query \(\Wei^{(1)}\) is described in Section~\ref{sec:Wei1}. 

This structured decomposition, combined with the scalability of cumulative reward vectors and efficient detachment, enables \(\Wei^{(1)}\) to support fast and accurate evaluation of cumulative rewards over sub-rectangles, as required by the \(\Draw(t)\) routine. We now describe how the drawing algorithm employs these two subclasses of cumulative reward retrieval queries along  with horizontal and vertical sampling to locate an atomic region with probability proportional to its cumulative reward.

\paragraph{Horizontal Sampling.}
By the horizontal sampling step, the \(\Draw(t)\) algorithm samples a region \(\rg_1 = [\bp^{(0)}_{t_1}:\bp^{(0)}_{t_2}][0:1]\) with probability proportional to \(\CR(\rg_1)/\CR([0:1][0:1])\), where the horizontal interval \([\bp^{(0)}_{t_1}:\bp^{(0)}_{t_2}]\) is atomic. To this end, the algorithm performs a randomized root-to-leaf traversal of the first-level auxiliary interval tree \(\ATree^{(0)}\). The leaf reached by this traversal corresponds to a rectangular region with atomic horizontal interval and full vertical range.

The traversal begins at the root node \(u_0 = \ROOT(\ATree^{(0)})\). Then, while \(u_0\) remains an internal node, the algorithm sets \(u_0\) to its right child with probability
\[
\Pr(u_0 \gets \RC(u_0)) = \frac{\CR([\RA(\RC(u_0))][0:1])}{\CR([\RA(u_0)][0:1])},
\]
and to its left child with the complementary probability. The values \(\CR([\RA(\RC(u_0))][0:1])\) are computed using the query \(\Wei^{(0)}([\RA(\RC(u_0))][0:1])\). The traversal continues until reaching a leaf \(l_0\), where \(\RA(l_0) = [\bp^{(0)}_{t_1}:\bp^{(0)}_{t_2}]\).

\paragraph{Vertical Sampling.}
Given the selected horizontal interval \(\RA(l_0)\), we sample a vertical atomic interval to complete the region \(\rg_2\). This is done by traversing the auxiliary tree \(\ATree^{(1)}\) top-down. At each internal node \(u_1\), the next move is selected via:
\[
\Pr(u_1 \gets \RC(u_1)) = \frac{\CR([\RA(l_0)][\RA(\RC(u_1))])}{\CR([\RA(l_0)][\RA(u_1)])},
\]
and \(\Pr(u_1 \gets \LC(u_1)) = 1 - \Pr(u_1 \gets \RC(u_1))\). The cumulative reward values are computed using \(\Wei^{(1)}\). This continues until reaching a leaf \(l_1\), yielding \(\RA(l_1) = [\bp^{(1)}_{t_1}:\bp^{(1)}_{t_2}]\), and completing the sampled region \(\rg_2 = [\RA(l_0)][\RA(l_1)]\). This region $\rg_2$ then will be returned as the drawn region $\rg$. The detailed queries and their pseudocode implementations are provided in Appendices~\ref{section:quereis}.

\paragraph{Sampling \(x_t\) within Region \(\rg\).}
After sampling a region \(\rg\) with probability proportional to its cumulative reward, we now sample a point \(x_t \in \rg\) according to a distribution proportional to the cumulative reward within \(\rg\). Specifically, we aim to sample \(x_t\) with probability
\[
x_t \sim \frac{\CR(x)}{\CR(\rg)} = \frac{\MF(x, \Fu_t(x))}{\int_{\rg}\MF(x, \Fu_t(x))dx}.
\]
To accomplish this, we invoke the following result regarding the \emph{Hit-and-Run}~\citep{lovasz2007geometry} algorithm:

\begin{theorem}[\citep{lovasz2007geometry}]\label{them:vempa}
Let \( f : \mathbb{R}^d \gets \mathbb{R}_{> 0} \) be a log-concave function supported on a convex body \( Q \subset \mathbb{R}^d \). Suppose we have access to a membership oracle for \( Q \) and an evaluation oracle for \( f \), and that the initial distribution is \( Y \)-warm with respect to the target distribution \( \pi(x) \propto f(x) \). Then, after
\[
O\left(d^3 \log\left( \frac{Y}{\varepsilon} \right)\right)
\]
steps, the distribution of the Hit-and-Run Markov chain is within total variation distance \(\varepsilon\) of \(\pi\).
\end{theorem}

We now verify that the conditions of Theorem~\ref{them:vempa} are satisfied in our setting. First, we assume that the mapping \(\MF(x, \cdot)\) is a log-concave polynomial and that the vector \(\fu_t\) is positive at every round, ensuring that the resulting function \(\MF(x, \Fu_t(x))\) is log-concave. Second, since our action space \(\AS\) is a rectangular region, it is convex, and a membership oracle can be implemented efficiently via simple geometric checks. Therefore, we can apply the Hit-and-Run algorithm to sample \(x_t \in \rg\) in \(O(d^3)\) steps with high accuracy from the target distribution.

\subsection{Correctness and Generalization}
The correctness of all core operations—including \(\Ins\), \(\Update\), \(\Wei\), and \(\Draw\)—is formally and in detale established in Appendix~\ref{section:correctness}. Appendix~\ref{section:generalpar} then presents the generalization of our data structure to \(d\)-dimensions. This involves extending the hierarchical design so that each node at every level is augmented with two next-level structures corresponding to regular and lazy insertions. The structural properties proven in Section~\ref{sec:structure} carry over to this generalized setting.
\section{Query Details and Implementations in $2$-Dimensions (\Cref{section:ds-qhandling})}\label{section:quereis}
We now describe how to initialize the data structure and handle the queries introduced in Section~\ref{sec:que}. We begin with the insertion and update routines for an arbitrary rectangular region \(\rg = [x^{(0)}_1:x^{(0)}_2][x^{(1)}_1:x^{(1)}_2]\), which form the basis for maintaining \(\Tree\) at each time step. We then illustrate these procedures using the region \(\rg_{1,t} \in \Rg_t\) as a representative example. The same steps apply to all other regions in \(\Rg_t\) and, more generally, to any rectangular region.

\subsection{Initialization:}

We initialize the data structure \(\Tree\) by creating a two-dimensional \(\IT\). This initialization involves constructing \(\Tree^{(0)}\) as a balanced \(\IT\) that initially contains the single horizontal interval \([0,1)\), represented by a single node \(\ROOT(\Tree^{(0)})\). Simultaneously, the associated second-level data structures \(\Tree^{(1)}\), \(\LZL^{(1)}\), and \(\LZR^{(1)}\) are initialized as \(\IT\)s over the vertical interval \([0,1)\). Every time a new node is inserted into \(\Tree^{(0)}\), we follow the same initialization steps by creating its associated second-level lazy and regular insertion \(\IT\)s.

\subsection{Locating an Endpoint \(x^{(i)}\)}
Many of the queries described below rely on a common subroutine: locating the leaf node in an \(\Tree^{(i)}\) whose interval contains a given endpoint \(x^{(i)}\). The search starts at the root \(\ROOT(\Tree^{(i)})\) and recursively moves to the left or right child, depending on which subtree contains \(x^{(i)}\), until reaching a leaf representing the atomic interval that includes \(x^{(i)}\). 

\begin{algorithm}[H]
    \begin{algorithmic}[1]
    \caption{$\Loc(x^{(i)})$}
    \label{alg:locate}
    \Input{Endpoint \(x^{(i)}\)}
         \State $v_i \gets \ROOT(\Tree^{(i)})$ 
         \While{$v_i$ is not a leaf}
         \If{\(x^{(i)} \in \RC(v_i)\)}
         \State $v_i \gets \RC(v_i) $
         \Else
         \State $v_i \gets \LC(v_i) $
         \EndIf
         \EndWhile
         \State Return $v_i$
    \end{algorithmic}
\end{algorithm}

\subsection{Locating the Smallest Interval Containing \([x^{(i)}_1{:}x^{(i)}_2]\)}
Several queries require finding the node in \(\Tree^{(i)}\) whose interval is the smallest one that fully contains \([x^{(i)}_1{:}x^{(i)}_2]\). This node is the \emph{first common ancestor} of the leaves \(l_i\) and \(h_i\) that contain the endpoints \(x^{(i)}_1\) and \(x^{(i)}_2\), respectively. We denote this node by \(\FCA(l_i, h_i)\), and trace the paths from \(l_i\) and \(h_i\) to the root as \(p_{l_i}\) and \(p_{h_i}\).

\begin{algorithm}[H]
    \begin{algorithmic}[1]
    \caption{$\FCA(l_i,h_i)$}
    \label{alg:fca}
    \Input{Leaves $l_i,h_i$}
        \State Set $p_{l_i} = l_i$, $p_{h_i} = h_i$, $l = l_i$, and $h=h_i$
         \While{$p_{l_i} \cap p_{h_i} = \emptyset$}
         \State $l\gets \PAR(l)$
         \State $h\gets \PAR(h)$
         \State $p_{l_i}\gets p_{l_i} \cup l$
         \State $p_{h_i}\gets p_{h_i} \cup h$
         \EndWhile
         \State Return $p_{l_i} \cap p_{h_i}$
    \end{algorithmic}
\end{algorithm}

\subsection{Insert Query \(\Ins(\rg:[x^{(0)}_1:x^{(0)}_2][x^{(1)}_1:x^{(1)}_2])\):} 

To insert a region \(\rg = [x^{(0)}_1{:}x^{(0)}_2][x^{(1)}_1{:}x^{(1)}_2]\) chosen by the adversary, we follow a three-step process. First, we insert the horizontal interval \([x^{(0)}_1{:}x^{(0)}_2]\) into \(\Tree^{(0)}\). Next, we insert the vertical interval \([x^{(1)}_1{:}x^{(1)}_2]\) into the appropriate second-level trees---either regular or lazy structures---associated with nodes \(v_0 \in \Tree^{(0)}\) whose horizontal range \(\RA(v_0)\) intersects \([x^{(0)}_1{:}x^{(0)}_2]\). Finally, we update the auxiliary trees by inserting \([x^{(0)}_1{:}x^{(0)}_2]\) into \(\ATree^{(0)}\) and \([x^{(1)}_1{:}x^{(1)}_2]\) into \(\ATree^{(1)}\). We now describe each step in detail.

\paragraph{Insertion of the Horizontal Interval $[x^{(0)}_1:x^{(0)}_2]$ :}
If the horizontal interval \([x^{(0)}_1{:}x^{(0)}_2]\) does not exist in \(\Tree^{(0)}\), we insert it using the standard interval tree procedure. Once inserted, we identify the set of nodes \(v_0 \in \Tree^{(0)}\) whose second-level trees require either regular or lazy insertion of the vertical interval \([x^{(1)}_1{:}x^{(1)}_2]\). These are nodes such that \(\RA(v_0)\) intersects or is fully contained by \([x^{(0)}_1{:}x^{(0)}_2]\).

To determine the affected nodes, we locate the leaves \(l_0\) and \(u_0\) containing \(x^{(0)}_1\) and \(x^{(0)}_2\), and define the paths \(p_{l_0}\), \(p_{u_0}\), and \(p_{v_{l_0,u_0}}\) as before, where \(v_{l_0,u_0} = \FCA(l_0, u_0)\). We then identify the nodes along these paths for which \([x^{(1)}_1{:}x^{(1)}_2]\) should be regularly or lazily inserted, as detailed in the following subsections.

\begin{algorithm}[H]
    \begin{algorithmic}[1]
    \caption{$\Ins$ an interval}
    \label{alg:ins1}
    \Input{Horizontal Interval $[x^{(0)}_1:x^{(0)}_2]$}
         \State $l_0 \gets \Loc(x^{(0)}_1)$ 
         \State $u_0 \gets \Loc(x^{(0)}_2)$ 
         \State $[a_{l_0},b_{l_0}] \gets \RA(l_0)$ \State $[a_{u_0},b_{u_0}] \gets \RA(u_0)$
         \State Make new nodes $v_0$ and $v'_0$
         \State $\RA(v_0) \gets [a_{l_0},x^{(0)}_1]$, $\RA(v'_0) \gets [x^{(0)}_1,b_{l_0}]$
         \State $\LC(l_0) \gets v_0$, $\PAR(v_0)\gets l_0$
         \State $\RC(l_0) \gets v'_0$, $\PAR(v'_0)\gets l_0$
         \State $\LZR^{(1)}_{l_0} \gets \Tree^{(1)}_{l_0}$
         \State $\LZL^{(1)}_{l_0} \gets \Tree^{(1)}_{l_0}$
         
         \State Make new nodes $v_1$ and $v'_1$
         \State $\RA(v_1) \gets [a_{u_0},x^{(0)}_2]$, $\RA(v'_1) \gets [x^{(0)}_2,b_{u_0}]$
         \State $\RC(u_0) \gets v_1$, $\PAR(v_1)\gets u_0$
         \State $\LC(u_0) \gets v'_1$, $\PAR(v'_1)\gets u_0$
         \State $\LZL^{(1)}_{u_0} \gets \Tree^{(1)}_{u_0}$
         \State $\LZR^{(1)}_{u_0} \gets \Tree^{(1)}_{u_0}$
    \end{algorithmic}
\end{algorithm}

\paragraph{Regular Insertion of the Vertical Interval \([x^{(1)}_1:x^{(1)}_2]\):} As discussed earlier, regular insertion of the vertical interval \([x^{(1)}_1:x^{(1)}_2]\) is required for all nodes \(v_0 \in \Tree^{(0)}\) such that \(\RA(v_0)\) intersects but is not fully contained by \([x^{(0)}_1:x^{(0)}_2]\). These nodes lie along the union of paths \(p_{l_0} \cup p_{u_0}\).

Thus, we traverse each of these paths bottom-up and insert \([x^{(1)}_1:x^{(1)}_2]\) with its corresponding $\FP$ \(\fu_t(\rg_{1,t})\) into the regular second-level trees \(\Tree^{(1)}_{v_0}\). This insertion follows the standard interval tree procedure used for horizontal intervals.

\begin{algorithm}[H]
    \begin{algorithmic}[1]
    \caption{Regular Insertion of an Interval}
    \label{alg:insreg}
    \Input{an interval \([x^{(1)}_1:x^{(1)}_2]\), paths $p_{l_0},p_{u_0}$}
         \For{All nodes in $v_0 \in p_{l_0} \cup p_{u_0}$ }
         \State Insert the vertical interval $[x^{(1)}_1:x^{(1)}_2]$ in $\Tree^{(1)}_{v_0}$
         \EndFor  
    \end{algorithmic}
\end{algorithm}

\paragraph{Lazy Insertion of the Vertical Interval \([x^{(1)}_1:x^{(1)}_2]\):}
By definition, lazy insertion of the vertical interval \([x^{(1)}_1:x^{(1)}_2]\) applies to nodes \(v_0 \in (p_{l_0} \cup p_{u_0}) \setminus p_{v_{l_0,u_0}}\) whose child subtrees are fully contained by the horizontal interval \([x^{(0)}_1:x^{(0)}_2]\). Specifically:

\begin{itemize}
    \item For $v_0 \in p_{l_0}\setminus p_{v_{l_0,u_0}}$: If \(\RA(\LC(v_0))\) is fully contained by \([x^{(0)}_1:x^{(0)}_2]\), we lazily insert \([x^{(1)}_1:x^{(1)}_2]\) into \(\LZL^{(1)}_{v_0}\),
    \item For $v_0 \in p_{u_0}\setminus p_{v_{l_0,u_0}}$:If \(\RA(\RC(v_0))\) is fully contained by \([x^{(0)}_1:x^{(0)}_2]\), we insert it into \(\LZR^{(1)}_{v_0}\).
\end{itemize}

To perform the insertion, we traverse the nodes in \((p_{l_0} \cup p_{u_0}) \setminus p_{v_{l_0,u_0}}\) bottom-up and apply the standard interval tree insertion routine to each lazy structure. These insertions defer updates to the affected subtrees, preserving efficiency while ensuring that the contributions from \([x^{(1)}_1:x^{(1)}_2]\) are later applied via scaling during queries.

\begin{algorithm}[H]
    \begin{algorithmic}[1]
    \caption{Lazy Insertion of an Interval}
    \label{alg:inslaz}
    \Input{an interval \([x^{(1)}_1:x^{(1)}_2]\), paths $p_{l_0},p_{u_0},p_{v_{l_0,u_0}}$}
         \For{All nodes in $v_0 \in p_{l_0} \setminus p_{v_{l_0,u_0}}$ }
         \If{$\LC(v_0) \in p_{l_0}$}
         \State Insert \([x^{(1)}_1:x^{(1)}_2]\) into $\LZR^{(1)}_{v_0}$
         \EndIf
         \EndFor  
         \For{All nodes in $v_0 \in p_{u_0} \setminus p_{v_{l_0,u_0}}$ }
         \If{$\RC(v_0) \in p_{u_0}$}
         \State Insert \([x^{(1)}_1:x^{(1)}_2]\) into $\LZL^{(1)}_{v_0}$
         \EndIf
         \EndFor  
    \end{algorithmic}
\end{algorithm}

\paragraph{Update \(\Tree\) with \(\fu_t(\rg)\):}
Once the horizontal interval of \(\rg\) is inserted and its vertical interval is lazily or regularly inserted into the appropriate second-level trees, we invoke the \(\Update(\rg)\) query. This step aggregates the associated \(\FP\) of \(\rg\) into all second-level structures where its vertical interval resides, updating the cumulative profile from \(\Fu_{t-1}\) to \(\Fu_t\). This ensures that the drawing algorithm samples actions correctly and that both lazy and regular insertions reflect the contribution of \(\fu_t(\rg)\).

\subsection{Update Query \(\Update(\rg:[x^{(0)}_1:x^{(0)}_2][x^{(1)}_1:x^{(1)}_2])\):} \label{sec:update}
This query updates the data structure by incorporating the \( \FP \) value \( \fu_t(\rg) \) of the inserted region \( \rg \) into the existing structure that currently maintains \( \Fu_{t-1} \), thereby producing \( \Fu_t \). To perform this update, we modify the second-level interval trees \( \Tree^{(1)}_{v_0} \), \( \LZL^{(1)}_{v_0} \), and \( \LZR^{(1)}_{v_0} \) associated with the nodes \( v_0 \in \Tree^{(0)} \), ensuring that the vertical interval \( [x^{(1)}_1:x^{(1)}_2] \) of the region \( \rg \) is incorporated into one of these second-level structures.

This update follows standard \emph{regular} and \emph{lazy weight update} techniques for interval trees (see \cite{cohen2017online}), which we adapt to our vector-valued cumulative rewards as \emph{regular} and \emph{lazy \( \FP \) updates}. The lazy \( \FP \) updates are distinct from lazy insertions. While lazy insertions are used to insert new intervals, lazy \( \FP \) updates are specifically employed to update the cumulative reward vectors in the second-level trees. Although the routine for handling lazy \( \FP \) updates remains consistent across all second-level data structures—whether lazy or regular insertion trees—the updates are applied only to the second-level structures of the affected nodes \( v_0 \in \Tree^{(0)} \) that are impacted by the insertion of \( \rg \).

Although the update process is mostly similar in both lazy and regular insertion data structures, there are minor differences when updating the cumulative reward vectors in each. For simplicity, we describe the process for the regular insertion data structure \( \Tree^{(1)}_{v_0} \); however, we also highlight the differences with the lazy insertion data structures \( \LZL^{(1)}_{v_0} \) and \( \LZR^{(1)}_{v_0} \) as we go through.

\paragraph{Locate a Node \(v_0 \in \Tree\) with Affected \(\Tree^{(1)}_{v_0}\):}
We locate the leaves \(l_0\) and \(u_0\) in \(\Tree^{(0)}\) whose horizontal ranges \(\RA(l_0)\) and \(\RA(u_0)\) have \(x^{(0)}_1\) and \(x^{(0)}_2\) as their left and right endpoints, respectively. Let \(p_{l_0}\) and \(p_{u_0}\) denote their respective paths to the root, and let \(v_{l_0,u_0} = \FCA(l_0, u_0)\) be their first common ancestor, with path \(p_{v_{l_0,u_0}}\) to the root.

\begin{algorithm}[H]
    \begin{algorithmic}[1]
    \caption{Locate the Horizontal interval}
    \label{alg:upd1}
    \Input{Horizontal Interval $[x^{(0)}_1:x^{(0)}_2]$}
         \State $l_0 \gets \Loc(x^{(0)}_1)$ 
         \State $u_0 \gets \Loc(x^{(0)}_2)$ 
         \State Set $v_{l_0,u_0} \gets \FCA(l_0,u_0)$
    \end{algorithmic}
\end{algorithm}

We now update all second-level trees corresponding to nodes along \(p_{l_0}\), \(p_{u_0}\), and \(p_{v_{l_0,u_0}}\) where the vertical interval \([x^{(1)}_1:x^{(1)}_2]\) has been inserted. For each such node \(v_0 \in p_{l_0} \cup p_{u_0}\), we apply the update procedure to its associated \(\Tree^{(1)}_{v_0}\) as follows:

    \paragraph{Locate the Vertical Interval \([x^{(1)}_1:x^{(1)}_2]\) in \(\Tree^{(1)}_{v_0}\):} 
   At this stage, we locate the vertical interval \([x^{(1)}_1:x^{(1)}_2]\) within the relevant second-level tree \(\Tree^{(1)}_{v_0}\). Let \(l_1\) and \(u_1\) be the leaves containing the left and right endpoints, and let \(p_{l_1}\), \(p_{u_1}\), and \(p_{v_{l_1,u_1}}\) denote the paths from these leaves and their first common ancestor \(v_{l_1,u_1} = \FCA(l_1, u_1)\) to the root, respectively.
 \begin{algorithm}[H]
    \begin{algorithmic}[1]
    \caption{Locate the Vertical Interval}
    \label{alg:upd2}
    \Input{Vertical Interval $[x^{(1)}_1:x^{(1)}_2]$, Regular Insertion Tree $\Tree^{(1)}_{v_0}$}
         \State $l_1 \gets \Loc(x^{(1)}_1)$ 
         \State $u_1 \gets \Loc(x^{(1)}_2)$ 
         \State Set $v_{l_1,u_1} \gets \FCA(l_1,u_1)$
    \end{algorithmic}
\end{algorithm}

\paragraph{Regular Reward Function Parameters Update in \(\Tree^{(1)}_{v_0},\) $\LZL^{(1)}_{v_0},$ and $\LZR^{(1)}_{v_0}$:}
We update the cumulative reward vectors \( \vec{\CR}_{v_0,v_1} \) for all \( v_1 \in p_{l_1} \cup p_{u_1} \) using the following update rule:
\begin{align}
    \vec{\CR}_{v_0,v_1} &\leftarrow \vec{\CR}_{v_0,v_1} + \vec{\SC}_{\rg'}, \label{alg:lupdl'}
\end{align}
where \( \rg' = [\RA(v_0) \cap [x^{(0)}_1:x^{(0)}_2]][\RA(v_1) \cap [x^{(1)}_1:x^{(1)}_2]] \), and \( \vec{\SC}_{\rg'} \) is computed by integrating the mapping function \( \MF \) over \( \rg' \) with coefficient vector \( \fu_t(\rg) \), as described in Section~\ref{section:scaling}.

 \begin{algorithm}[H]
    \begin{algorithmic}[1]
    \caption{Regular Reward Function Parameters Update in $\Tree^{(1)}_{v_0}$}
    \label{alg:upd3}
    \Input{$l_1$, $u_1$, and $\Tree^{(1)}_{v_0}$}
        \For{$v_1 \in p_{l_1} \cup p_{u_1}$}
        \State Set \( \rg' = [\RA(v_0) \cap [x^{(0)}_1:x^{(0)}_2]][\RA(v_1) \cap [x^{(1)}_1:x^{(1)}_2]] \)
        \State Set  $\vec{\CR}_{v_0,v_1} \gets \vec{\CR}_{v_0,v_1} + \vec{\SC}_{\rg'}$
        \EndFor
    \end{algorithmic}
\end{algorithm}

In contrast, for the lazy insertion structures \( \LZL^{(1)}_{v_0} \) and \( \LZR^{(1)}_{v_0} \) for the nodes $w_1 \in p_{l_1} \cup p_{u_1} \setminus p_{v_{l_1,u_1}} $, which defer updates to be scaled down during queries, the update is performed as follows:
\begin{align}
    \vec{\CR}_{v_0,w_1} &\leftarrow \vec{\CR}_{v_0,w_1} + \vec{\SC}_{\rg''}, \label{alg:lupdl'1}
\end{align}
where \( \rg'' = [\RA(v_0)][\RA(w_1) \cap [x^{(1)}_1:x^{(1)}_2]] \) spans the full horizontal range of the subtree rooted at \( v_0 \). To simplify the scaling process when the lazy insertion structures are later accessed during query time, we treat the range \( \RA(v_0) \) as fully contained by the horizontal interval \( [x^{(0)}_1:x^{(0)}_2] \). This treatment facilitates the downward scaling of updates into the subtrees, ensuring that the update remains compatible with the information stored in \( \Tree^{(1)}_{v_0} \), without affecting the correctness of the algorithm.

\begin{algorithm}[H]
    \begin{algorithmic}[1]
    \caption{Regular Reward Function Parameters Update in $\LZL^{(1)}_{v_0}$ or $\LZR^{(1)}_{v_0}$ }
    \label{alg:upd4}
    \Input{$l_1$, $u_1$, and $\LZL^{(1)}_{v_0}$ or $\LZR^{(1)}_{v_0}$}
        \For{$w_1 \in p_{l_1} \cup p_{u_1}\setminus p_{v_{l_1,u_1}}$}
        \State Set \( \rg'' = [\RA(v_0)][\RA(w_1) \cap [x^{(1)}_1:x^{(1)}_2]] \)
        \State Set  $\vec{\CR}_{v_0,w_1} \gets \vec{\CR}_{v_0,w_1} + \vec{\SC}_{\rg'}$
        \EndFor
    \end{algorithmic}
\end{algorithm}

\paragraph{Lazy Reward Function Parameters Update in $\Tree^{(1)}_{v_0}$, $\LZL^{(1)}_{v_0}$, and $\LZR^{(1)}_{v_0}$:}
The lazy \( \FP \) update proceeds by traversing the paths \( p_{l_1} \setminus p_{v_{l_1,u_1}} \) and \( p_{u_1} \setminus p_{v_{l_1,u_1}} \) in a bottom-up manner. At each node \( v_1 \) along these paths, a lazy update message, either \( \LZM_{v_1, \RC(v_1)} \) or \( \LZM_{v_1, \LC(v_1)} \), is sent to the child whose vertical interval is strictly contained in \( [x^{(1)}_1:x^{(1)}_2] \). In this context, \( \LZM_{v'_0,v'_1} \) represents the lazy \( \FP \) update message sent by node \( \PAR(v'_1) \) to its child \( v'_1 \) in the second-level \( \IT \) associated with \( v'_0 \in \Tree^{(0)} \). Each lazy update message carries a cumulative \( \FP \) vector that has not yet been propagated. Since these vectors are intended to uniformly influence the entire subtree rooted at the recipient, they are aggregated and integrated during propagation. The final aggregated results are then applied during query-time integration.

We begin with the nodes \( v_1 \) along the path \( p_{l_1} \setminus p_{v_{l_1,u_1}} \). If the left child \( \LC(v_1) \) lies on the path, we send a lazy \( \FP \) update message to the right child \( \RC(v_1) \) as follows:
\begin{align}
    \LZM_{v_0,\RC(v_1)}  &\leftarrow \LZM_{v_0,\RC(v_1)} + \fu_t(\rg). \label{alg:lupd2}
\end{align}

Next, we apply the same update process to the nodes \( v_1 \in p_{u_1} \setminus p_{v_{l_1,u_1}} \), with a slight variation. If the right child \( \RC(v_1) \) also belongs to the path \( p_{u_1} \setminus p_{v_{l_1,u_1}} \), we send a lazy update message to the left child \( \LC(v_1) \), as follows:
\begin{align}
    \LZM_{v_0,\LC(v_1)}  &\leftarrow \LZM_{v_0,\LC(v_1)} + \fu_t(\rg). \label{alg:lupd3}
\end{align}

\begin{algorithm}[H]
    \begin{algorithmic}[1]
    \caption{Lazy Reward Function Parameters Update }
    \label{alg:upd5}
    \Input{$l_1$, $u_1$, $\Tree^{(1)}_{v_0}$ and $\LZL^{(1)}_{v_0}$ or $\LZR^{(1)}_{v_0}$}
        \For{$v_1 \in p_{l_1}\setminus p_{u_{l_1,u_1}}$}
        \If{$\LC(v_1) \in p_{l_1}\setminus p_{u_{l_1,u_1}}$}
        \State $\LZM_{v_0,\RC(v_1)} \gets \LZM_{v_0,\RC(v_1)} + \fu_t(\rg)$
        \EndIf
        \EndFor
        \For{$v_1 \in p_{u_1}\setminus p_{u_{l_1,u_1}}$}
        \If{$\RC(v_1) \in p_{u_1}\setminus p_{u_{l_1,u_1}}$}
        \State $\LZM_{v_0,\LC(v_1)} \gets \LZM_{v_0,\LC(v_1)} + \fu_t(\rg)$
        \EndIf
        \EndFor

    \end{algorithmic}
\end{algorithm}

\paragraph{Lazy Reward Function Update Messages Propagation}
We now describe how lazy \( \FP \) update messages are propagated and cleared in both regular and lazy insertion data structures. These propagations are triggered during insertions, updates, or query operations (e.g., a \( \Draw \) query), typically when locating a leaf. For such operations, the algorithm identifies the relevant leaves \( l_1 \) and \( u_1 \) in \( \Tree^{(1)}_{v_0} \), and applies any pending lazy update messages ( \( \LZM \) ) along the corresponding root-to-leaf paths, \( p_{l_1} \) or \( p_{u_1} \).

The mechanism is uniform across both the regular insertion structure \( \Tree^{(1)}_{v_0} \) and the lazy insertion structures \( \LZL^{(1)}_{v_0} \) and \( \LZR^{(1)}_{v_0} \). In all cases, the algorithm traverses the path top-down and applies the lazy messages just before descending to a child.

For instance, consider the downward traversal of \( p_{l_1} \), starting from \( v_1 = \ROOT(\Tree^{(1)}_{v_0}) \). Before descending to the right child \( \RC(v_1) \), the algorithm updates the corresponding cumulative reward vector as:
\begin{align}
    \vec{\CR}_{v_0,\RC(v_1)} \leftarrow \vec{\CR}_{v_0,\RC(v_1)} + \vec{\SC}_{\rg'}, \label{eq:lz1}
\end{align}
where \( \rg' :[\RA(v_0)][\RA(\RC(v_1))] \) in both the regular and lazy insertion structures. The coefficient vector used to compute \( \vec{\SC}_{\rg'} \) is \( \LZM_{v_0,\RC(v_1)} \). This guarantees that deferred updates are consistently applied and that contributions remain correctly scaled.

Then, as the algorithm proceeds to the child \( \RC(v_1) \), it pushes the lazy message further down to both its children. Let \( v'_1 = \LC(\RC(v_1)) \) and \( v''_1 = \RC(\RC(v_1)) \). The propagation is done by:
\begin{align}
    \LZM_{v_0,v'_1} &\leftarrow \LZM_{v_0,v'_1} + \LZM_{v_0,\RC(v_1)}, \label{eq:lz2} \\
    \LZM_{v_0,v''_1} &\leftarrow \LZM_{v_0,v''_1} + \LZM_{v_0,\RC(v_1)}, \label{eq:lz3}
\end{align}
and the original message is then cleared:
\[
    \LZM_{v_0,\RC(v_1)} \leftarrow \vec{0}.
\]

The same procedure applies symmetrically if the traversal goes through \( \LC(v_1) \). This unified handling across both regular and lazy data structures ensures that all lazy updates are correctly propagated and scaled, maintaining the correctness and efficiency of cumulative reward computations.

\begin{algorithm}[H]
    \begin{algorithmic}[1]
    \caption{Lazy Reward Function Parameters Update Message Propagation}
    \label{alg:upd6}
    \Input{$l_1$, $p_{l_1} =\{l_1,\PAR(l_1),\PAR(\PAR(l_1)),\cdots,\ROOT(\Tree^{(1)}_{v_0})\}$, $\Tree^{(1)}_{v_0}$($\LZL^{(1)}_{v_0}$ or $\LZR^{(1)}_{v_0}$)}
    \State Traverse $p_{l_1}$ downwards by setting $v_1 \gets \ROOT(\Tree^{(1)}_{v_0})$
     \State $v'_1 \gets \RC(v_1)]$
      \State $v''_1 \gets \LC(v_1)]$
        \While{$v_1$ is not a leaf}
        \If{$\RC(v_1) \in p_{l_1}$}
        \State $\rg' \gets [\RA(v_0)][\RA(\RC(v_1))]$
        \State $v'_1 \gets \RC(\RC(v_1))]$
        \State $v''_1 \gets \LC(\RC(v_1))]$
        \State $ \vec{\CR}_{v_0,\RC(v_1)} \gets \vec{\CR}_{v_0,\RC(v_1)} + \vec{\SC}_{\rg'}$
        \State $\LZM_{v_0,v'_1} \gets \LZM_{v_0,v'_1} +$ $\LZM_{v_0,\RC(v_1)}$
          \State $\LZM_{v_0,v''_1} \gets \LZM_{v_0,v''_1} +$ $\LZM_{v_0,\RC(v_1)}$
        \State $\LZM_{v_0,\RC(v_1)} \leftarrow \vec{0}$

        \Else
        \State $\rg' \gets [\RA(v_0)][\RA(\LC(v_1))]$
        \State $v'_1 \gets \RC(\LC(v_1))]$
        \State $v''_1 \gets \LC(\LC(v_1))]$
        \State $ \vec{\CR}_{v_0,\LC(v_1)} \gets \vec{\CR}_{v_0,\LC(v_1)} + \vec{\SC}_{\rg'}$
        \State $\LZM_{v_0,v'_1} \gets \LZM_{v_0,v'_1} +$ $\LZM_{v_0,\LC(v_1)}$
        \State $\LZM_{v_0,v''_1} \gets \LZM_{v_0,v''_1} +$ $\LZM_{v_0,\LC(v_1)}$
        \State $\LZM_{v_0,\LC(v_1)} \leftarrow \vec{0}$
        \EndIf
        \EndWhile

    \end{algorithmic}
\end{algorithm}

\subsection{Cumulative Reward Retrieval Query $\Wei(\rg:[x^{(0)}_1:x^{(0)}_2][x^{(1)}_1:x^{(1)}_2])$} \label{Sec:Weigeneral}
This query returns the cumulative rewards associated with a given rectangular region \( \rg \) by aggregating all relevant vectors \( \vec{\CR}_{v_0,v_1} \) stored across the data structure and map it to the real valued cumulative reward. Specifically, it gathers contributions from nodes \( v_0 \in \Tree^{(0)} \) and nodes \( v_1 \) in the corresponding second-level structures—\( \Tree^{(1)}_{v_0} \), \( \LZL^{(1)}_{v_0} \), or \( \LZR^{(1)}_{v_0} \)—such that:
\begin{itemize}
    \item the horizontal range \( \RA(v_0) \) intersects with or fully contains the horizontal interval \( [x^{(0)}_1:x^{(0)}_2] \), and
    \item the vertical range \( \RA(v_1) \) intersects with or fully contains the vertical interval \( [x^{(1)}_1:x^{(1)}_2] \).
\end{itemize}

In our multi-stage drawing algorithm (described in Section \ref{section:draw}), we work with locating atomic intervals along each coordinate in a sequential manner, following the hierarchy of \( \Tree \). To clarify the notation, we define atomic intervals as those that are indivisible along their respective coordinates. Specifically, we use primes in the intervals \( [x'^{(0)}_1:x'^{(0)}_2] \) and \( [x'^{(1)}_1:x'^{(1)}_2] \) to indicate that these intervals are atomic with respect to the corresponding coordinate. Here, \( [x'^{(0)}_1:x'^{(0)}_2] \) represents an atomic horizontal interval in the \( x^{(0)} \)-coordinate, and \( [x'^{(1)}_1:x'^{(1)}_2] \) represents an atomic vertical interval in the \( x^{(1)} \)-coordinate.

Based on this structure, we divide the cumulative reward retrieval query into two subroutines. The algorithm first locates an atomic horizontal interval \( [x'^{(0)}_1:x'^{(0)}_2] \) by querying regions of the form \( \rg' : [x^{(0)}_1:x^{(0)}_2][0:1] \), handled by the subroutine \( \Wei^{(0)}(\rg') \). Once the horizontal atomic interval is identified, the algorithm then locates the atomic vertical interval \( [x'^{(1)}_1:x'^{(1)}_2] \) by querying over \( \rg'': [x'^{(0)}_1:x'^{(0)}_2][x^{(1)}_1:x^{(1)}_2] \), using the subroutine \( \Wei^{(1)}(\rg'') \).

We therefore distinguish between the following two classes of cumulative reward queries:
\begin{itemize}
    \item \(\Wei^{(0)}(\rg')\): queries of the form \([x^{(0)}_1:x^{(0)}_2][0:1]\), which return cumulative rewards over general horizontal intervals and the full vertical range.
    \item \(\Wei^{(1)}(\rg'')\): queries of the form \([x'^{(0)}_1:x'^{(0)}_2][x^{(1)}_1:x^{(1)}_2]\), where the horizontal interval is atomic and the vertical interval may be arbitrary.
\end{itemize}

Here, \( \Wei^{(i)}(\rg) \) retrieves the cumulative reward for a region \( \rg \) where the first \( i-1 \) coordinate intervals are atomic and fixed by previous stages, while the \( i \)-th coordinate spans a general interval defined by the endpoints inserted by the adversary. The remaining coordinates correspond to open subspaces, represented by the interval \( [0:1] \).

This separation is crucial: for instance, in the two-dimensional case, all horizontal intervals are stored in \( \Tree^{(0)} \), while vertical intervals are distributed across second-level data structures attached to individual nodes in \( \Tree^{(0)} \). As a result, once the horizontal interval is fixed, retrieving cumulative rewards for vertical atomic intervals must be handled differently and locally across the second-level structures.

Handling both types of queries involves first identifying the set of second-level data structures \( \mathcal{S}^{(1)} \) in \( \Tree \) whose stored intervals correspond to all the regions that either intersect with or fully contain the query region \( \rg \). We then retrieve the relevant portions of the cumulative rewards from these structures and aggregate them to compute the total cumulative reward \( \CR(\rg) \). A crucial requirement for the correctness of this aggregation is that the data structures in \( \mathcal{S}^{(1)} \) store disjoint information with respect to the \( \FP \) of any region intersecting \( \rg \). In other words, if the vertical intervals stored in these data structures do not share any common endpoints, there will be no region that intersects \( \rg \) and is counted twice. This disjointness ensures that there is no overlap or double-counting when summing the contributions across different subtrees, which follows directly from the recursive and non-overlapping design of the lazy and regular structures in \( \Tree \).

\subsubsection{\(\Wei^{(0)}(\rg:[x^{(0)}_1:x^{(0)}_2][0:1])\)} \label{Sec:Wei0}

To process this query, we first identify the leaves \(l_0\) and \(u_0\) in \(\Tree^{(0)}\) that contain the endpoints of the horizontal interval \([x^{(0)}_1:x^{(0)}_2]\), and compute their first common ancestor \(v_{l_0,u_0}\). Let \(p_{l_0}\), \(p_{u_0}\), and \(p_{v_{l_0,u_0}}\) denote the paths from \(l_0\), \(u_0\), and \(v_{l_0,u_0}\) to the root of \(\Tree^{(0)}\), respectively.

Since \(\RA(v_{l_0,u_0})\) contains \([x^{(0)}_1:x^{(0)}_2]\), we distinguish two cases: either \(\RA(v_{l_0,u_0}) = [x^{(0)}_1:x^{(0)}_2]\), or \(\RA(v_{l_0,u_0})\) strictly contains the horizontal interval \([x^{(0)}_1:x^{(0)}_2]\).

\begin{enumerate}
    \item \textbf{Case \(\RA(v_{l_0,u_0}) = [x^{(0)}_1:x^{(0)}_2]\):} 
    
    In this case, we aim to retrieve and return the cumulative reward of the region \( \rg \) as \( \CR([\RA(v_{l_0,u_0})][0,1]) \). This value consists of two contributing components:

\begin{itemize}
    \item The first component is the portion of the cumulative rewards from regions whose horizontal intervals intersect, but do not fully contain, the range \( \RA(v_{l_0,u_0}) \) maintained at \( v_{l_0,u_0} \). These regions contribute to the cumulative rewards of the region \( \rg \) due to their intersection. By definition, the aggregated cumulative \( \FP \) vectors of such regions over the region \( \rg : [\RA(v_{l_0,u_0})][0,1] \) can be directly retrieved from the root of the regular insertion tree \( \Tree^{(1)}_{v_{l_0,u_0}} \), which aggregates parts of the cumulative rewards of these regions.

    \item The second component is the portion of the cumulative rewards from regions whose horizontal intervals fully contain \( \RA(v_{l_0,u_0}) \). These regions are stored in the lazy insertion trees \( \LZL^{(1)}_{v_0} \) and \( \LZR^{(1)}_{v_0} \) for all nodes \( v_0 \) along the path \( p_{v_{l_0,u_0}} \), i.e., ancestors of \( v_{l_0,u_0} \). If \( v_{l_0,u_0} \) is in the left (resp.\ right) subtree of \( v_0 \), we retrieve the relevant data from \( \LZL^{(1)}_{v_0} \) (resp.\ \( \LZR^{(1)}_{v_0} \)) and apply the appropriate scaling.

\end{itemize}

We then collect all the second-level data structures into a set \( \mathcal{S}^{(1)} \), as previously mentioned, and query each structure to retrieve the corresponding part of the cumulative reward for the region \( \rg \). To formalize this process, we define the set \( \mathcal{S}^{(1)}_{v_{l_0,u_0}} \) to include all second-level structures relevant for computing the cumulative reward \( \CR(\rg) \). Initially, this set contains the regular insertion tree associated with $v_{l_0,u_0}$:

\[
\mathcal{S}^{(1)}_{v_{l_0,u_0}} \gets \left\{ \Tree^{(1)}_{v_{l_0,u_0}} \right\}.
\]
We then traverse the path \(v_0 \in p_{v_{l_0,u_0}}\) and augment this set as follows:
\begin{align*}
    \mathcal{S}^{(1)}_{v_{l_0,u_0}} &\gets \mathcal{S}^{(1)}_{v_{l_0,u_0}} \cup \LZL^{(1)}_{v_0} \quad \text{if } \LC(v_0) \in p_{v_{l_0,u_0}},\\
    \mathcal{S}^{(1)}_{v_{l_0,u_0}} &\gets \mathcal{S}^{(1)}_{v_{l_0,u_0}} \cup \LZR^{(1)}_{v_0} \quad \text{if } \RC(v_0) \in p_{v_{l_0,u_0}}.
\end{align*}
\begin{algorithm}[H]
    \begin{algorithmic}[1]
    \caption{ $\texttt{Construct}( \mathcal{S}^{(1)}_{v_{l_0,u_0}})$ }
    \label{alg:crr1}
    \Input{Region $\rg:[x^{(0)}_1:x^{(0)}_2][0:1])$}
        \State $l_0 \gets \Loc(x^{(0)}_1)$
        \State $u_0 \gets \Loc(x^{(0)}_2)$
        \State $v_{l_0,u_0} \gets \FCA(l_0,u_0)$
         \State $\mathcal{S}^{(1)}_{v_{l_0,u_0}} \gets \{\Tree^{(1)}_{v_{l_0,u_0}}\} $
         \For{$v_0 \in p_{v_{l_0,u_0}} $}
         \If{$\LC(v_0) \in  p_{v_{l_0,u_0}}$}
         \State $\mathcal{S}^{(1)}_{v_{l_0,u_0}}\gets \mathcal{S}^{(1)}_{v_{l_0,u_0}} \cup \LZL^{(1)}_{v_0}$
         \Else
         \State $\mathcal{S}^{(1)}_{v_{l_0,u_0}}\gets \mathcal{S}^{(1)}_{v_{l_0,u_0}} \cup \LZR^{(1)}_{v_0}$
         \EndIf
         \EndFor
    \end{algorithmic}
\end{algorithm}
We index the second-level data structures in the set \( \mathcal{S}^{(1)}_{v_{l_0,u_0}} \) by \( S^{(1)}_{v_0} \), where each structure corresponds to the second-level data associated with the node \( v_0 \in p_{v_{l_0,u_0}} \) and is inserted into \( \mathcal{S}^{(1)}_{v_{l_0,u_0}} \). In particular, the only regular insertion second-level tree in this set, \( \Tree^{(1)}_{v_{l_0,u_0}} \), is responsible for regions whose horizontal intervals intersect with, but do not fully contain, the horizontal interval \( \RA(v_{l_0,u_0}) \) of the region \( \rg \). The remaining second-level data structures in \( \mathcal{S}^{(1)}_{v_{l_0,u_0}} \) are lazy insertion trees associated with ancestor nodes \( v_0 \) for which \( \RA(v_0) \) fully contains \( \RA(v_{l_0,u_0}) \). In summary, \( S^{(1)}_{v_{l_0,u_0}} = \Tree^{(1)}_{v_{l_0,u_0}} \) and the rest are lazy structures. Thus, the cumulative reward for \( \rg \) is given by:

\begin{align}
    \Wei^{(0)}(\rg) = \underbrace{\CR_{S^{(1)}_{v_{l_0,u_0}}}(\rg)}_{\text{cumulative rewards of the regular insertion tree}} + \underbrace{\sum_{S^{(1)}_{v_0}\in \mathcal{S}^{(1)}_{v_{l_0,u_0}}\setminus S^{(1)}_{v_{l_0,u_0}}} \CR_{S^{(1)}_{v_0}}(\rg)}_{\text{cumulative rewards of the lazy insertion trees}}, \label{crew:detach}
\end{align}

where \( \CR_{S^{(1)}_{v_0}}(\rg) \) represents the cumulative reward of the region \( \rg \) with respect to the regions whose intervals are stored in the data structure \( S^{(1)}_{v_0} \).

For each term of the equation \eqref{crew:detach}, we retrieve the cumulative reward vectors from each data structure, apply scaling where needed, and aggregate the results to return \( \CR(\rg:[\RA(v_{l_0,u_0})][0:1]) \). For the regular insertion tree \( S^{(1)}_{v_{l_0,u_0}} \) in \( \mathcal{S}^{(1)}_{v_{l_0,u_0}} \), we retrieve the vectors without scaling, as they are computed by integrating over regions with the appropriate range on the first coordinate. For the lazy insertion trees in \( \mathcal{S}^{(1)}_{v_{l_0,u_0}} \), we scale each cumulative reward vector to the region \( \rg = [\RA(v_{l_0,u_0})][0:1] \), which is a subregion of \( [\RA(v_0)][0:1] \).

To compute the first term of equation \ref{crew:detach}, which is directly accessible from the cumulative reward vector stored at the root of the regular insertion tree \( S^{(1)}_{v_{l_0,u_0}} \), we have:
\begin{align}
    \CR_{S^{(1)}_{v_{l_0,u_0}}}(\rg) = \sum_{j \in [A]} \vec{\CR}^{(j)}_{v_{l_0,u_0},\ROOT(S^{(1)}_{v_0})},\label{crew:detach1} 
\end{align}
where \( A \) is the number of basis monomials in the mapping function \( \MF \), and \( \alpha^j_i \) is the exponent of the \( i \)-th coordinate variable in the \( j \)-th monomial. The term \( \vec{\CR}^{(j)}_{v_{l_0,u_0},\ROOT(S^{(1)}_{v_0})} \) represents the \( j \)-th entry of the cumulative reward vector. These basis terms are fixed and detailed in Section~\ref{section:scaling}.

For the lazy data structures appearing in the second term of equation \eqref{crew:detach}, we follow the same procedure as for the regular insertion trees. However, as previously mentioned, the scaling procedure is particularly important for the cumulative reward vectors at the roots of these data structures, such as \( \vec{\CR}_{v_0,\ROOT(S^{(1)}_{v_0})} \), since these vectors are computed over regions with a horizontal interval equal to \( \RA(v_0) \). We iterate over each data structure \( S^{(1)}_{v_0} \in \mathcal{S}^{(1)}_{v_{l_0,u_0}} \setminus S^{(1)}_{v_{l_0,u_0}} \), extract the cumulative reward vector at the root \( \vec{\CR}_{v_0,\ROOT(S^{(1)}_{v_0})} \), and scale it to \( \rg \). Let \( \RA(v_0) = [a^{(0)}_1:a^{(0)}_2] \), \( \RA(\ROOT(S^{(1)}_{v_0})) = [a^{(1)}_1 = 0 : a^{(1)}_2 = 1] \), and \( \RA(v_{l_0,u_0}) = [x^{(0)}_1:x^{(0)}_2] \), with \( [x^{(1)}_1 = 0 : x^{(1)}_2 = 1] \). We then apply the scaling rule from Equation~\eqref{eq:scalingmain}, yielding:

{\small\begin{align}
   \footnotesize{
   \sum_{S^{(1)}_{v_0} \in \mathcal{S}^{(1)}_{v_{l_0,u_0}} \setminus S^{(1)}_{v_{l_0,u_0}}} \CR_{S^{(1)}_{v_0}}(\rg) = \sum_{S^{(1)}_{v_0} \in \mathcal{S}^{(1)}_{v_{l_0,u_0}} \setminus S^{(1)}_{v_{l_0,u_0}}} \sum_{j \in [A]} \vec{\CR}^{(j)}_{v_0, \ROOT(S^{(1)}_{v_0})} \cdot \prod_{i=0}^{1} \frac{(x^{(i)}_2)^{\alpha^j_i + 1} - (x^{(i)}_1)^{\alpha^j_i + 1}}{(a^{(i)}_2)^{\alpha^j_i + 1} - (a^{(i)}_1)^{\alpha^j_i + 1}},}\label{crew:detach2}   
\end{align}}
where \( A \) is the number of basis monomials in the mapping function \( \MF \), and \( \alpha^j_i \) is the exponent of the \( i \)-th coordinate variable in the \( j \)-th monomial. The term \( \vec{\CR}^{(j)}_{v_0,\ROOT(S^{(1)}_{v_0})} \) represents the \( j \)-th entry of the cumulative reward vector. These basis terms are fixed and detailed in Section~\ref{section:scaling}.

Finally we would have,
\begin{align}
    \Wei^{(0)}(\rg) = \text{\eqref{crew:detach1}} + \text{\eqref{crew:detach2}}.
\end{align}

\item \textbf{Case \([x^{(0)}_1:x^{(0)}_2] \subset\RA(v_{l_0,u_0})\):}

For the second case, where \( [x^{(0)}_1:x^{(0)}_2] \subset \RA(v_{l_0,u_0}) \), we begin by computing \( \CR([\RA(v_{l_0,u_0})][0:1]) \) using the method described for the exact-match case via the query \( \Wei^{(0)}([\RA(v_{l_0,u_0})][0:1]) \). We then refine this result by subtracting contributions from subregions outside \( [x^{(0)}_1:x^{(0)}_2] \) but still within \( \RA(v_{l_0,u_0}) \).

To do this, we traverse the paths \( p_{l_0} \setminus p_{v_{l_0,u_0}} \) and \( p_{u_0} \setminus p_{v_{l_0,u_0}} \) in a bottom-up manner and prune subtrees that fall entirely outside the query interval. Specifically, for each \( v'_0 \in p_{l_0} \setminus p_{v_{l_0,u_0}} \), if \( \LC(v'_0) \) is not on the path, we subtract \( \CR([\RA(\LC(v'_0))][0:1]) \). Similarly, for each \( v''_0 \in p_{u_0} \setminus p_{v_{l_0,u_0}} \), if \( \RC(v''_0) \) is not on the path, we subtract \( \CR([\RA(\RC(v''_0))][0:1]) \). These values are computed recursively using \( \Wei^{(0)}(\cdot) \), ensuring modularity and correctness.

This pruning ensures that the final result corresponds exactly to the cumulative reward over \( [x^{(0)}_1:x^{(0)}_2][0:1] \), excluding any contributions from outside this interval.

\begin{algorithm}[H]
    \begin{algorithmic}[1]
    \caption{$\Wei^{(0)}(\rg)$}
    \label{alg:crr2}
    \Input{Region $\rg:[x^{(0)}_1:x^{(0)}_2][0:1])$, $l_0$, $u_0$, $v_{l_0,u_0}$}
        
        \If{$\RA(v_{l_0,u_0}) = [x^{(0)}_1:x^{(0)}_2] $} 
        \State  $ \CR_{S^{(1)}_{v_{l_0,u_0}}}(\rg) \gets \sum_{j \in [A]} \vec{\CR}^{(j)}_{v_{l_0,u_0},\ROOT(S^{(1)}_{v_0})}$
        \State $\CR_1 \gets 0$
        \For{$S^{(1)}_{v_0}\in \mathcal{S}^{(1)}_{v_{l_0,u_0}}\setminus S^{(1)}_{v_{l_0,u_0}}$}
        \State $[a^{(0)}_1:a^{(0)}_2] \gets \RA(v_0)$
        \State $[a^{(1)}_1: a^{(1)}_2 ]\gets [0:1]$
        \State $[x^{(1)}_1: x^{(1)}_2 ]\gets [0:1]$
        \State $\CR_1 \gets \CR_1 +\sum_{j \in [A]} \vec{\CR}^{(j)}_{v_0, \ROOT(S^{(1)}_{v_0})} \cdot \prod_{i=0}^{1} \frac{(x^{(i)}_2)^{\alpha^j_i + 1} - (x^{(i)}_1)^{\alpha^j_i + 1}}{(a^{(i)}_2)^{\alpha^j_i + 1} - (a^{(i)}_1)^{\alpha^j_i + 1}}$
        \EndFor
        \State Return $\CR_{S^{(1)}_{v_{l_0,u_0}}}(\rg) + \CR_1$
        \Else
        \State $\CR(\rg) \gets \Wei([\RA(v_{l_0,u_0})][0:1])$
        \For{$v_0 \in p_{l_0} \setminus p_{v_{l_0,u_0}}$}
        \If{$\RC(v_0) \in p_{l_0} \setminus p_{v_{l_0,u_0}}$}
            \State $\CR(\rg)\gets \CR(\rg) -\Wei^{(0)}([\RA(\LC(v_0))][0:1])$
        \EndIf
        \EndFor
        \For{$v_0 \in p_{u_0} \setminus p_{v_{l_0,u_0}}$}
        \If{$\LC(v_0) \in p_{u_0} \setminus p_{v_{l_0,u_0}}$}
            \State $\CR(\rg)\gets \CR(\rg) -\Wei^{(0)}([\RA(\RC(v_0))][0:1])$
        \EndIf
        \EndFor
        \State Return $\CR(\rg)$
        \EndIf
    \end{algorithmic}
\end{algorithm}

\end{enumerate}

\subsubsection{\(\Wei^{(1)}\left(\rg:[x'^{(0)}_1:x'^{(0)}_2][x^{(1)}_1:x^{(1)}_2]\right)\)} \label{sec:Wei1}
Given the definition of \( \Wei^{(1)} \), assume the horizontal interval \( [x'^{(0)}_1:x'^{(0)}_2] \) is atomic and corresponds to a leaf \( l_0 \in \Tree^{(0)} \), i.e., \( \RA(l_0) = [x'^{(0)}_1:x'^{(0)}_2] \). Let \( p_{l_0} \) denote the path from \( l_0 \) to the root. To compute \( \Wei^{(1)}(\rg) \), we collect all second-level data structures that affect \( \RA(l_0) \) into the set \( \mathcal{S}^{(1)}_{l_0} \), following the same method as for \( \Wei^{(0)} \). We similarly would have all data structures in \( \mathcal{S}^{(1)}_{l_0} \) to be lazy insertion data structures except \( S^{(1)}_{l_0} \), which is \( \Tree^{(1)}_{l_0} \). Thus, similar to equation \eqref{crew:detach}, we can write:

\begin{align}
    \Wei^{(1)}(\rg) = \underbrace{\CR_{S^{(1)}_{l_0}}(\rg)}_{\text{cumulative rewards of the regular insertion tree}} + \underbrace{\sum_{S^{(1)}_{v_0} \in \mathcal{S}^{(1)}_{l_0} \setminus S^{(1)}_{l_0}} \CR_{S^{(1)}_{v_0}}(\rg)}_{\text{cumulative rewards of the lazy insertion trees}}, \label{crew:detachplus}
\end{align}

For the first term of equation \eqref{crew:detachplus}, we retrieve \( \CR_{S^{(1)}_{l_0}}(\rg) \) from the data structure \( S^{(1)}_{l_0} \) without scaling. Then, we traverse the remaining lazy insertion data structures in \( \mathcal{S}^{(1)}_{l_0} \), retrieve the cumulative reward vectors for the region \( \rg':[\RA(v_0)][x^{(1)}_1:x^{(1)}_2] \) in each \( S^{(1)}_{v_0} \in \mathcal{S}^{(1)}_{l_0} \), and scale them to \( \rg \). Although the cumulative rewards are retrieved for regions with similar vertical intervals in each \( S^{(1)}_{v_0} \), the disjointness of the data structures in \( \mathcal{S}^{(1)}_{l_0} \) prevents double-counting, ensuring accurate aggregation.

The main distinction from the previous subclass of queries is that the cumulative reward vectors specifying the vertical interval \( [x^{(1)}_1:x^{(1)}_2] \) are not necessarily stored at the root of the data structures in \( \mathcal{S}^{(1)}_{l_0} \). Additionally, the way vertical intervals are maintained in the second-level data structures introduces another challenge. Specifically, the vertical interval endpoints are not stored within a single \( \IT \); rather, they are distributed across different data structures associated with various nodes in \( \Tree^{(0)} \). To clarify, when handling queries of the form \( \Wei^{(0)}(\cdot) \), we assume the vertical interval of the queried region is \( [0:1] \), which is guaranteed to be stored at the root of each data structure. However, in this subclass of queries, it is not guaranteed that there exists a leaf whose range exactly matches \( [x^{(1)}_1:x^{(1)}_2] \). As a result, when locating the endpoints of \( [x^{(1)}_1:x^{(1)}_2] \) in each data structure, we may end up at a leaf where neither endpoint corresponds to \( x^{(1)}_1 \) or \( x^{(1)}_2 \). We detail the process in the following section.

We first identify the leaves \( l_1 \) and \( u_1 \) in \( S^{(1)}_{v_0} \), where \( v_0 \in p_{l_0} \), that contain \( x^{(1)}_1 \) and \( x^{(1)}_2 \), respectively. Let \( \RA(l_1) = [a_1^{(1)}:a_2^{(1)}] \) and \( \RA(u_1) = [b_1^{(1)}:b_2^{(1)}] \), and define their first common ancestor \( v_{l_1,u_1} \) such that \( [a_1^{(1)}:b_2^{(1)}] \subseteq \RA(v_{l_1,u_1}) \).

For a given data structure \( S^{(1)}_{v_0} \), which can be either a regular or lazy insertion tree, if both endpoints of the vertical interval \( [x^{(1)}_1:x^{(1)}_2] \) are stored in \( S^{(1)}_{v_0} \), then \( a^{(1)}_1 = x^{(1)}_1 \) and \( b^{(1)}_2 = x^{(1)}_2 \). In this case, for both the first and second terms of equation \eqref{crew:detachplus}, to compute \( \CR_{S^{(1)}_{l_0}}(\rg) \), we first retrieve the cumulative reward for \( \rg' = [\RA(v_0)][a^{(1)}_1:b^{(1)}_2] \) in the same way as the \( \Wei^{(0)} \) case. This involves retrieving the cumulative reward over the larger region \( \rg'' = [\RA(v_0)][\RA(v_{l_1,u_1})] \), pruning the excess contributions outside \( [\RA(v_0)][a^{(1)}_1:b^{(1)}_2] \), and scaling the result over \( \RA(l_0) = [x'^{(0)}_1:x'^{(0)}_2] \) to match the target region \( [\RA(l_0)][a^{(1)}_1:b^{(1)}_2] \). 

As noted earlier, unlike \( \Tree^{(0)} \), \( S^{(1)}_{v_0} \) does not store all the endpoints of the vertical intervals. Therefore, pruning the unnecessary portions of \( \rg'' = [\RA(v_0)][\RA(v_{l_1,u_1})] \) requires handling partial coverage at the boundaries, specifically at \( \RA(l_1) \) and \( \RA(u_1) \). If \( x^{(1)}_1 \) is not stored in \( S^{(1)}_{v_0} \), then \( a^{(1)}_1 < x^{(1)}_1 < a^{(1)}_2 \), and similarly, if \( x^{(1)}_2 \) is missing from \( S^{(1)}_{v_0} \), then \( b^{(1)}_1 < x^{(1)}_2 < b^{(1)}_2 \). In these cases, we exclude \( \CR_{S^{(1)}_{v_0}}([\RA(v_0)][a^{(1)}_1:x^{(1)}_1]) \) and \( \CR_{S^{(1)}_{v_0}}([\RA(v_0)][x^{(1)}_2:b^{(1)}_2]) \) from \( \CR_{S^{(1)}_{v_0}}(\rg') \). While handling boundaries for both lazy and regular data structures is largely similar, a scaling step is necessary when handling boundaries in lazy data structures, i.e., in \( \mathcal{S}^{(1)}_{l_0} \setminus S^{(1)}_{l_0} \). We detail both cases below.

\begin{enumerate}
    \item \textbf{Compute $\CR_{S^{(1)}_{l_0}}(\rg)$:}
    
    For the regular insertion data structure in the first term of equation \eqref{crew:detachplus}, the cumulative reward retrieval for the region \( \rg \), denoted \( \CR_{S^{(1)}_{l_0}}(\rg) \), is computed as follows:

\begin{align}
    \CR_{S^{(1)}_{l_0}}(\rg) = \CR_{S^{(1)}_{l_0}}(\rg') - u_{l_1} - u_{u_1}, \label{crew1:reg}
\end{align}
where \( \rg' = [\RA(l_0)][a^{(1)}_1:b^{(1)}_2] \) is the region from which the cumulative reward is computed. Specifically, \( u_{l_1} = \CR_{S^{(1)}_{l_0}}([\RA(l_0)][a^{(1)}_1:x^{(1)}_1]) \) and \( u_{u_1} = \CR_{S^{(1)}_{l_0}}([\RA(l_0)][x^{(1)}_2:b^{(1)}_2]) \) represent the cumulative rewards of the excess contributions from the leaves \( l_1 \) and \( u_1 \), which must be excluded.

To compute \( u_{l_1} \) and \( u_{u_1} \), we exploit the fact that \( l_0 \), \( l_1 \), and \( u_1 \) are all leaves, and the regions \( [\RA(l_0)][\RA(l_1)] \) and \( [\RA(l_0)][\RA(u_1)] \) are atomic regions. Thus, the cumulative \( \FP \) vectors \( \vec{\CR}_{l_0,l_1} \) and \( \vec{\CR}_{l_0,u_1} \) can be scaled over the second coordinate to accurately compute  \( u_{l_1} = \CR_{S^{(1)}_{l_0}}([\RA(l_0)][a^{(1)}_1:x^{(1)}_1]) \) and \( u_{u_1} = \CR_{S^{(1)}_{l_0}}([\RA(l_0)][x^{(1)}_2:b^{(1)}_2]) \), as follows:

We compute the excess contribution from the lower boundary by scaling the vector \( \vec{\CR}_{l_0,l_1} \) to the region \( [\RA(l_0)][a^{(1)}_1:x^{(1)}_1] \). The resulting contribution is given by:

\begin{align}
    u_{l_1} = \sum_{j \in [A]} \vec{\CR}^{(j)}_{l_0,l_1} \cdot \frac{(x^{(1)}_1)^{\alpha^j_1 + 1} - (a^{(1)}_1)^{\alpha^j_1 + 1}}{(a^{(1)}_2)^{\alpha^j_1 + 1} - (a^{(1)}_1)^{\alpha^j_1 + 1}}, \label{crew1:reg1}
\end{align}

where no scaling is applied to the first coordinate because \( l_0 \) is a leaf, and its regular insertion tree \( S^{(1)}_{l_0} = \Tree^{(1)}_{l_0} \) contains the \( \FP \) values of regions whose horizontal intervals contain \( \RA(l_0) \). Therefore, the integration within this tree is performed over the entire \( \RA(l_0) \), and no scaling is necessary for that coordinate.

Similarly, the contribution from the upper boundary is computed by scaling the vector \( \vec{\CR}_{v_0,u_1} \) to the region \( [\RA(l_0)][x^{(1)}_2:b^{(1)}_2] \), as:

\begin{align}
    u_{u_1} = \sum_{j \in [A]} \vec{\CR}^{(j)}_{l_0,u_1} \cdot \frac{(b^{(1)}_2)^{\alpha^j_1 + 1} - (x^{(1)}_2)^{\alpha^j_1 + 1}}{(b^{(1)}_2)^{\alpha^j_1 + 1} - (b^{(1)}_1)^{\alpha^j_1 + 1}}. \label{crew1:reg2}
\end{align}

According to \eqref{crew1:reg}, subtracting \(u_{l_1} + u_{u_1}\) from the full-interval contribution \(\CR_{S^{(1)}_{l_0}}([\RA(l_0)][a^{(1)}_1:b^{(1)}_2])\), and applying the vertical scaling step, yields the desired cumulative reward \(\CR_{S^{(1)}_{l_0}}\) for the region \([x'^{(0)}_1:x'^{(0)}_2][x^{(1)}_1:x^{(1)}_2]\).

\item \textbf{Computing $\sum_{S^{(1)}_{v_0} \in \mathcal{S}^{(1)}_{l_0} \setminus S^{(1)}_{l_0}} \CR_{S^{(1)}_{v_0}}(\rg)$:}

To retrieve the cumulative reward for the region \( \rg \) in each lazy insertion data structure \( S^{(1)}_{v_0} \in \mathcal{S}^{(1)}_{l_0} \), we apply the same procedure as for the regular data structure \( S^{(1)}_{l_0} \). However, there is a slight difference in the scaling step because the horizontal range of the nodes \( \RA(v_0) = [x''^{(0)}_1:x''^{(0)}_2] \) does not match \( \RA(l_0) \), but instead fully contains it. Therefore, we define the vector representation of the equation \eqref{crew1:reg} and compute each term using the cumulative reward vectors in \( S^{(1)}_{v_0} \). After handling the boundaries, we scale the entire cumulative reward vector over the horizontal range \( \RA(l_0) \). Thus, we write:

\begin{align}
    \vec{\CR}_{S^{(1)}_{v_0}}(\rg) = \vec{\CR}_{S^{(1)}_{v_0}}(\rg') - \vec{u}_{l_1} - \vec{u}_{u_1}, \label{crew1:laz}
\end{align}
where \( \rg' = [\RA(v_0)][a^{(1)}_1:b^{(1)}_2] \). Specifically, \( \vec{u}_{l_1} = \vec{\CR}_{S^{(1)}_{v_0}}([\RA(v_0)][a^{(1)}_1:x^{(1)}_1]) \) and \( \vec{u}_{u_1} = \vec{\CR}_{S^{(1)}_{v_0}}([\RA(v_0)][x^{(1)}_2:b^{(1)}_2]) \) represent the cumulative reward vectors of the excess contributions from the leaves \( l_1 \) and \( u_1 \), which must be excluded.

After retrieving each vector in equation \eqref{crew1:laz}, we scale each vector separately to the target region \( \rg \). The reason for scaling the vectors \( \vec{\CR}_{S^{(1)}_{v_0}}(\rg') \), \( \vec{u}_{l_1} \), and \( \vec{u}_{u_1} \) separately is as follows: \( \vec{\CR}_{S^{(1)}_{v_0}}(\rg') \) is retrieved through intervals whose endpoints are directly stored in \( S^{(1)}_{v_0} \). During integration, we consider the entire horizontal range \( \RA(v_0) \), so the corresponding cumulative reward vector is only scaled over the first coordinate to match \( \RA(l_0) \). On the other hand, \( \vec{u}_{l_1} \) and \( \vec{u}_{u_1} \) represent the cumulative rewards of the excess regions, where the endpoints \( x^{(1)}_1 \) and \( x^{(1)}_2 \) are not stored in the data structure. Therefore, scaling must be applied over both the first and second coordinates.

To retrieve \( \vec{\CR}_{S^{(1)}_{v_0}}(\rg') \), we first retrieve the cumulative reward vector \( \vec{\CR}_{v_0,v_{l_1,u_1}} \), and then prune the contributions from the excess regions. Since all cumulative reward vectors in the lazy insertion data structure \( S^{(1)}_{l_0} \) are computed by integrating over the range \( \RA(v_0) \), the linear combination of these vectors remains valid. Thus, we retrieve the cumulative reward vector \( \vec{\CR}_{v_0,v_{l_1,u_1}} \) and then scale it over the horizontal interval \( \RA(l_0) \) as follows:

\begin{align}
    \sum_{j \in [A]} \vec{\CR}_{S^{(1)}_{v_0}}(\rg') \cdot \frac{(x'^{(0)}_2)^{\alpha^j_0 + 1} - (x'^{(0)}_1)^{\alpha^j_0 + 1}}{(x''^{(0)}_2)^{\alpha^j_0 + 1} - (x''^{(0)}_1)^{\alpha^j_0 + 1}}. \label{crew1:laz1}
\end{align}

Next, we retrieve \( \vec{u}_{l_1} \) and \( \vec{u}_{u_1} \), and scale them over both the vertical and horizontal coordinates to compute their contribution, which will be subtracted. Similar to the previous case, \( l_1 \) and \( u_1 \) are both leaves with \( \RA(l_1) = [a_1^{(1)}:a_2^{(1)}] \) and \( \RA(u_1) = [b_1^{(1)}:b_2^{(1)}] \). The cumulative reward vectors for the regions \( [\RA(v_0)][a_1^{(1)}:a_2^{(1)}] \) and \( [\RA(v_0)][b_1^{(1)}:b_2^{(1)}] \) are stored directly in \( \vec{\CR}_{v_0,l_1} \) and \( \vec{\CR}_{v_0,u_1} \), respectively. We then scale each vector over both the first and second coordinates to match the regions \( [\RA(v_0)][a^{(1)}_1:x^{(1)}_1] \) and \( [\RA(v_0)][x^{(1)}_2:b^{(1)}_2] \), which will be subtracted later.

We compute the excess contribution from the lower boundary by first scaling the vector \( \vec{\CR}_{v_0,l_1} \) to the region \( [\RA(v_0)][a^{(1)}_1:x^{(1)}_1] \), and then scaling it again to \( [\RA(l_0)][a^{(1)}_1:x^{(1)}_1] \). The contribution is computed as:

\begin{align}
    u_{l_1} = \sum_{j \in [A]} \vec{\CR}^{(j)}_{v_0,l_1} 
    \cdot 
    \frac{(x'^{(0)}_2)^{\alpha^j_0 + 1} - (x'^{(0)}_1)^{\alpha^j_0 + 1}}{(x''^{(0)}_2)^{\alpha^j_0 + 1} - (x''^{(0)}_1)^{\alpha^j_0 + 1}}
    \cdot \frac{(x^{(1)}_1)^{\alpha^j_1 + 1} - (a^{(1)}_1)^{\alpha^j_1 + 1}}{(a^{(1)}_2)^{\alpha^j_1 + 1} - (a^{(1)}_1)^{\alpha^j_1 + 1}}. \label{crew1:laz2}
\end{align}

Similarly, the contribution from the upper boundary is computed by scaling the vector \( \vec{\CR}_{v_0,u_1} \) to the region \( [\RA(l_0)][x^{(1)}_2:b^{(1)}_2] \), as follows:

\begin{align}
    u_{u_1} = \sum_{j \in [A]} \vec{\CR}^{(j)}_{v_0,u_1} 
    \cdot 
    \frac{(x'^{(0)}_2)^{\alpha^j_0 + 1} - (x'^{(0)}_1)^{\alpha^j_0 + 1}}{(x''^{(0)}_2)^{\alpha^j_0 + 1} - (x''^{(0)}_1)^{\alpha^j_0 + 1}}
    \cdot \frac{(b^{(1)}_2)^{\alpha^j_1 + 1} - (x^{(1)}_2)^{\alpha^j_1 + 1}}{(b^{(1)}_2)^{\alpha^j_1 + 1} - (b^{(1)}_1)^{\alpha^j_1 + 1}}. \label{crew1:laz3}
\end{align}

Finally, we compute \( \CR_{S^{(1)}_{v_0}}(\rg) \) according to equations \eqref{crew1:laz1}, \eqref{crew1:laz2}, and \eqref{crew1:laz3}, and aggregate over all the lazy insertion data structures in \( \mathcal{S}^{(1)}_{v_0} \) to get $\sum_{S^{(1)}_{v_0} \in \mathcal{S}^{(1)}_{l_0} \setminus S^{(1)}_{l_0}} \CR_{S^{(1)}_{v_0}}(\rg)$.

\end{enumerate}

\begin{algorithm}[H]
    \begin{algorithmic}[1]
    \caption{$\Wei^{(1)}(\rg)$}
    \label{alg:crr3}
    \Input{Region $\rg:[x'^{(0)}_1:x'^{(0)}_2][x^{(1)}_1:x^{(1)}_2])$, $l_0$}
        \State $r \gets 0$
        \For{$S^{(1)}_{v_0} \in \mathcal{S}^{(1)}_{l_0}$}
        \State $l_1 \gets \Loc(x^{(1)}_1)$
        \State $u_1 \gets \Loc(x^{(1)}_1)$
        \State $[a^{(1)}_1:a^{(1)}_2] \gets \RA(l_1)$
        \State $[b^{(1)}_1:b^{(1)}_2] \gets \RA(u_1)$
        \State $\rg' \gets [\RA(l_0)][a^{(1)}_1:b^{(1)}_2]$
        \State $u_{l_1} \gets \CR_{S^{(1)}_{l_0}}([\RA(l_0)][a^{(1)}_1:x^{(1)}_1])$ 
        \State \( u_{u_1} \gets \CR_{S^{(1)}_{l_0}}([\RA(l_0)][x^{(1)}_2:b^{(1)}_2]) \)
        \State $r \gets r + \CR(\rg') - u_{l_1} - u_{u_1}$
        \EndFor
        \State Return $r$
    \end{algorithmic}
\end{algorithm}

\subsection{Drawing Query \(\Draw(t)\)}\label{section:draw}
Now that we have all the necessary components, we describe the algorithm for handling the \(\Draw(t)\) query. The goal of this query is to sample an atomic region \(\rg = [x'^{(0)}_1{:}x'^{(0)}_2][x'^{(1)}_1{:}x'^{(1)}_2]\) with probability proportional to its cumulative reward:
\[
\frac{\CR(\rg)}{\sum_{\rg'} \CR(\rg')},
\]
where the summation is over all atomic regions \(\rg'\) formed over the action space \(\AS\).

The algorithm proceeds in two stages:

\begin{itemize}
    \item \textbf{Stage 1: Horizontal Sampling.} The algorithm first samples a region of the form \(\rg_1 = [x'^{(0)}_1{:}x'^{(0)}_2][0{:}1]\), where the horizontal interval \([x'^{(0)}_1{:}x'^{(0)}_2]\) is selected with probability
    \[
    \frac{\CR(\rg_1)}{\sum_{\rg' \in \Fu_t} \CR(\rg')}.
    \]
    This is implemented by traversing the top-level sampling tree \(\ATree^{(0)}\) from the root to a leaf \(l_0\) such that \(\RA(l_0) = [x'^{(0)}_1{:}x'^{(0)}_2]\), querying cumulative rewards of regions of the form \([x^{(0)}_1{:}x^{(0)}_2][0{:}1]\) at each node \(v_0\) along the path \(p_{l_0}\).
    
    \item \textbf{Stage 2: Vertical Sampling.} Conditional on the selected region \(\rg_1\), the algorithm samples a vertical sub-region \([x'^{(1)}_1{:}x'^{(1)}_2]\) such that the full atomic region \(\rg = [x'^{(0)}_1{:}x'^{(0)}_2][x'^{(1)}_1{:}x'^{(1)}_2]\) is sampled with probability
    \[
    \frac{\CR(\rg)}{\CR(\rg_1)}.
    \]
    This step is implemented by traversing the second-level sampling tree \(\ATree^{(1)}_{l_0}\), rooted at the node corresponding to the chosen leaf \(l_0\), down to a leaf \(l_1\) such that \(\RA(l_1) = [x'^{(1)}_1{:}x'^{(1)}_2]\), while querying cumulative rewards of regions of the form \([x'^{(0)}_1{:}x'^{(0)}_2][x^{(1)}_1{:}x^{(1)}_2]\).

    This two-stage procedure ensures that the overall sampling probability satisfies the desired marginal:
    \[
    \Pr(\rg \gets \Draw(t)) = \frac{\CR(\rg)}{\sum_{\rg'} \CR(\rg')}.
    \]
\end{itemize}

\subsubsection{Horizontal Sampling Stage}

The algorithm begins at the root node \(v_0 = \ROOT(\ATree^{(0)})\) and recursively traverses downward until reaching a leaf. At each internal node \(v_0 \in \ATree^{(0)}\), it decides whether to proceed to the left child \(\LC(v_0)\) or the right child \(\RC(v_0)\), based on the cumulative rewards of the corresponding regions.

Specifically, it moves to \(\RC(v_0)\) with probability
\begin{align}
 \Pr(v_0 \gets \RC(v_0)) = 
 \frac{\CR([\RA(\RC(v_0))][0{:}1])}
      {\CR([\RA(\RC(v_0))][0{:}1]) + \CR([\RA(\LC(v_0))][0{:}1])}, \label{eq:probright}  
\end{align}
and to \(\LC(v_0)\) with the complementary probability \(1 - \Pr(v_0 \to \RC(v_0))\).

To evaluate the required cumulative rewards, the algorithm uses the \(\Wei^{(0)}\) query:
\[
\Wei^{(0)}([\RA(\RC(v_0))][0{:}1]) \quad \text{and} \quad \Wei^{(0)}([\RA(\LC(v_0))][0{:}1]).
\]

This recursive process continues until a leaf node \(v_0 = l_0\) is reached, determining the atomic horizontal interval \(\RA(l_0)\) of the sampled region.

\begin{algorithm}[H]
    \begin{algorithmic}[1]
    \caption{Horizontal Sampling}
    \label{alg:draw1}
         \label{alg:DWR}
    \Input{$\ATree^{(0)}$ }
         \State $v_0\gets \ROOT(\ATree^{(0)})$ 
         \While {$v_0$ is not a leaf}
         \State Toss a Coin with Head probability $\frac{\Wei^{(0)}([\RA(\RC(v))][0:1])}{\Wei^{(0)}([\RA(\LC(v))][0:1]) + \Wei^{(0)}([\RA(\RC(v))][0:1])}$
         \If{Head}
         \State $v_0 \gets \RC(v_0)$
         \Else 
         \State $v_0 \gets \LC(v_0)$
         \EndIf
         \EndWhile
         \State Return $\rg_1:[\RA(v_0)][0:1]$ 
    \end{algorithmic}
\end{algorithm}

\subsubsection{Vertical Sampling Stage}

To select the vertical interval, the algorithm proceeds similarly using the tree \(\ATree^{(1)}_{l_0}\), which maintains the vertical atomic intervals corresponding to the horizontal leaf \(l_0\). Starting from the root \(v_1 = \ROOT(\ATree^{(1)}_{l_0})\), the algorithm recursively descends to a leaf.

At each internal node \(v_1\), the decision to move left or right is based on the cumulative reward values conditioned on the already selected horizontal interval \(\RA(l_0)\). The probability of moving to the right child \(\RC(v_1)\) is given by:
{\footnotesize\begin{align}
 \Pr(v_1 \gets \RC(v_1)) = 
 \frac{\CR([\RA(l_0)][\RA(\RC(v_1))])}
      {\CR([\RA(l_0)][\RA(\RC(v_1))]) + \CR([\RA(l_0)][\RA(\LC(v_1))])}, \label{eq:probrighty}
\end{align}
}
and the algorithm moves to \(\LC(v_1)\) with the remaining probability.

These values are computed using the \(\Wei^{(1)}\) query:
{\footnotesize\[
\Wei^{(1)}([\RA(l_0)][\RA(\RC(v_1))]) \quad \text{and} \quad \Wei^{(1)}([\RA(l_0)][\RA(\LC(v_1))]).
\]}
These queries are valid since \(\RA(l_0)\) is an atomic horizontal interval.

The process continues until a leaf node \(v_1 = l_1\) is reached in \(\ATree^{(1)}_{l_0}\), determining the atomic vertical interval \(\RA(l_1)\). The final sampled region is then \(\rg = [\RA(l_0)][\RA(l_1)]\), drawn with probability proportional to its cumulative reward.

\begin{algorithm}[H]
    \begin{algorithmic}[1]
    \caption{Vertical Sampling}
    \label{alg:draw2}
         \label{alg:DWR}
    \Input{$\ATree^{(1)}$, $\rg_1:[x'^{(0)}_1{:}x'^{(0)}_2][0{:}1]$ }
         \State $v_1\gets \ROOT(\ATree^{(1)})$ 
         \While {$v_1$ is not a leaf}
         \State Toss a Coin with {\footnotesize$\frac{\Wei^{(1)}([x'^{(0)}_1{:}x'^{(0)}_2][\RA(\RC(v_1))])}{\Wei^{(1)}([x'^{(0)}_1{:}x'^{(0)}_2][\RA(\LC(v_1))]) + \Wei^{(1)}([x'^{(0)}_1{:}x'^{(0)}_2][\RA(\RC(v_1))])}$}
         \If{Head}
         \State $v_1 \gets \RC(v_1)$
         \Else 
         \State $v_1 \gets \LC(v_1)$
         \EndIf
         \EndWhile
         \State Return $\rg:[x'^{(0)}_1{:}x'^{(0)}_2][\RA(v_1)]$ 
    \end{algorithmic}
\end{algorithm}

\subsubsection{Hit-and-Run~\citep{lovasz2007geometry} in the Sampled Atomic Region $\rg$}

After sampling an atomic region $\rg = 
[x^{(0)}_1:x^{(0)}_2][x^{(1)}_1:x^{(1)}_2]$ via \Cref{alg:draw2}, we sample $x_t \in \rg$ 
proportionally to $\MF(x, \Fu_t(x))$ using the Hit-and-Run algorithm. 
Since $\MF(x, \Fu_t(x))$ is log-concave over $\rg$ and each atomic 
region is a convex subset of $\AS$, the conditions of \Cref{them:vempa} 
are satisfied, guaranteeing efficient mixing within $O(d^3)$ steps.

\begin{algorithm}[H]
    \begin{algorithmic}[1]
    \caption{Hit-and-Run Sampling within Atomic Region $\rg$}
    \label{alg:hitandrun}
    \Input{Atomic region $\rg = \prod_{i=0}^{d-1}[x^{(i)}_1:x^{(i)}_2]$, 
    cumulative reward function $H_t$, accuracy $\varepsilon$}
    \State Initialize $x \gets$ any point in $\rg$
    \Repeat
        \State Sample a uniformly random direction 
        $\theta \in \mathbb{S}^{d-1}$
        \State Compute the chord $\ell = \{x + s\theta : s \in \mathbb{R}\} 
        \cap \rg$, giving interval $[s_{\min}, s_{\max}]$
        \State Sample $s^* \sim \frac{H_t(x + s\theta)}
        {\int_{s_{\min}}^{s_{\max}} H_t(x + s\theta)\,ds}$ 
        over $[s_{\min}, s_{\max}]$
        \State Set $x \gets x + s^*\theta$
    \Until{mixing criterion is met~(\Cref{them:vempa})}
    \State \Return $x_t \gets x$
    \end{algorithmic}
\end{algorithm}
\section{Correctness of the Data Structure Operations (\Cref{section:ds-qhandling})}\label{section:correctness}
In this section, we prove the effectiveness of our data structure in efficiently retrieving the cumulative reward of an atomic region queried by the online learning algorithm—either \(\Exp\) or \(\Band\)—during the action selection step at round \(t\).

\subsection{$\Ins$, $\Update$, and Structural Properties of $\Tree$}
We begin by establishing the correctness of \(\Ins\) and \(\Update\) queries at each round. We then analyze the resulting structure to prove key properties—such as size and update complexity—that are essential for the efficiency of our data structure.

\paragraph{Correctness of $\Ins$ Query:}
To insert a region \(\rg\) into our data structure, we begin by locating the atomic horizontal intervals at the leaves \(l_0\) and \(h_0\) of \(\Tree^{(0)}\) that contain the endpoints of \(\rg\)'s horizontal interval. We then split these leaves to insert the endpoints, following the standard \(\IT\) insertion routine. This process identifies all nodes \(v_0 \in \Tree^{(0)}\) whose horizontal ranges intersect the horizontal span of \(\rg\), by traversing the root-to-leaf paths \(p_{l_0}\) and \(p_{h_0}\). As a result, the set of nodes affected by the insertion consists of two parts: (i) the nodes along the union \(p_{l_0} \cup p_{h_0}\), which are directly and regularly updated, and (ii) the nodes in the subtree rooted at the lowest common ancestor \(v_{l_0,h_0}\) that are not on these paths but are lazily affected, as their horizontal ranges lie strictly within the horizontal interval of \(\rg\). These nodes are then used to update the cumulative reward vectors post insertion to the corresponding second-level \(\IT\)s.

Once these nodes \(v_0\) are located, we insert the vertical interval of \(\rg\) into the appropriate second-level structure—either the regular or lazy insertion \(\IT\)—associated with each \(v_0\), as dictated by the type of insertion. 

In summary, the following Observation~\ref{obs:ins} guarantees that the vertical interval of the region \(\rg\) is inserted into the appropriate regular and lazy insertion \(\IT\)s associated with the nodes in \(\Tree^{(0)}\), thereby covering all nodes affected by the insertion of \(\rg\).

\paragraph{Correctness of $\Update$ Query:}
After identifying and inserting the region \(\rg\), we proceed to update the cumulative reward information in all affected \(\IT\)s identified during the insertion step. This ensures that each tree correctly reflects the cumulative rewards over its associated subregions. The updating routine, in either lazy or regular insertion data structures, consists of regular and lazy $\FP$ updates and lazy update messages propagation which induce slightly different properties in each though they are handled mostly similar. 

We begin by characterizing the structural properties and cumulative reward updates in the regular insertion trees affected by the insertion of region \(\rg\). As noted above, all such trees \(\Tree^{(1)}_{v_0}\) correspond to nodes \(v_0 \in p_{l_0} \cup p_{h_0}\). We analyze the structure and update process by focusing on an arbitrary such node \(v_0\), as the argument applies uniformly to all affected nodes. However, a crucial distinction arises between leaf and non-leaf nodes in \(\Tree^{(0)}\). Our drawing algorithm treats these cases differently, and we correspondingly state the structural properties of \(\Tree^{(1)}_{v_0}\) through separate lemmas and observations for each type. In particular, for leaf nodes—namely \(v_0 = l_0\) or \(v_0 = h_0\)—we describe the update process through lemmas stated for \(\Tree^{(1)}_{l_0}\), which then naturally extend to \(\Tree^{(1)}_{h_0}\) and the associated regular insertion data structure of all other leaves in \(\Tree^{(0)}\).

\begin{lemma}\label{lemma:regroot}
For any node \( v_0 \in p_{l_0} \cup p_{h_0} \) (whether a leaf or internal node) in the first-level interval tree \(\Tree^{(0)}\), the cumulative reward vector \(\vec{\CR}_{v_0, \ROOT(\Tree^{(1)}_{v_0})}\) maintained at the root of its associated second-level tree \(\Tree^{(1)}_{v_0}\) correctly represents the cumulative rewards over the region \([\RA(v_0)][0{:}1]\), where \(\RA(v_0)\) denotes the horizontal range of node \(v_0\) with respect to the regions such that their vertical interval is regularly inserted in $\Tree^{(1)}_{v_0}$.
\end{lemma}

\begin{proof}
The correctness of this lemma follows from how the algorithm performs the \(\Update(\rg)\) operation immediately after inserting a region \(\rg = [x^{(0)}_1 : x^{(0)}_2][x^{(1)}_1 : x^{(1)}_2]\). By construction, every node \(v_0 \in p_{l_0} \cup p_{h_0}\) satisfies \(\RA(v_0) \cap [x^{(0)}_1 : x^{(0)}_2] \neq \emptyset\), and hence the vertical interval of \(\rg\) is inserted into the associated regular insertion tree \(\Tree^{(1)}_{v_0}\).

The insertion of the vertical interval \([x^{(1)}_1 : x^{(1)}_2]\) into \(\Tree^{(1)}_{v_0}\) begins by locating the leaf nodes \(l_1\) and \(h_1\) that contain the interval's endpoints. The algorithm then traverses the paths \(p_{l_1}\) and \(p_{h_1}\) from these leaves up to the root of \(\Tree^{(1)}_{v_0}\), applying the update rule~\eqref{alg:lupdl'} at every node \(v_1\) along \(p_{l_1} \cup p_{h_1}\). Since the root is always included in these paths, its cumulative reward vector \(\vec{\CR}_{v_0, \ROOT(\Tree^{(1)}_{v_0})}\) is guaranteed to be updated.

Each update at a node \(v_1\) adds the cumulative reward vector \(\vec{\SC}_{\rg'}\), where the region \(\rg' = [\RA(v_0) \cap [x^{(0)}_1 : x^{(0)}_2]][\RA(v_1) \cap [x^{(1)}_1 : x^{(1)}_2]]\)\ corresponds to the portion of \(\rg\) that lies within the subregion induced by the node pair \((v_0, v_1)\). In particular, for the root node \(v_1 = \ROOT(\Tree^{(1)}_{v_0})\), the vertical range \(\RA(v_1) = [0 : 1]\) spans the entire vertical axis, and thus \(\rg' = [\RA(v_0) \cap [x^{(0)}_1 : x^{(0)}_2]] [x^{(1)}_1 : x^{(1)}_2]\), which is exactly the portion of \(\rg\) that intersects the rectangular region \([\RA(v_0)] [0 : 1]\). Therefore, this update ensures that the vector \(\vec{\CR}_{v_0, \ROOT(\Tree^{(1)}_{v_0})}\) accumulates the contribution of \(\rg\) over the intersection \([\RA(v_0)] [0 : 1]\). The same argument applies to all previously inserted regions that intersect this rectangular area. Hence, the root of each regular insertion tree correctly maintains the total cumulative reward over its corresponding subregion, as stated.

\end{proof}

The following corollary extends Lemma~\ref{lemma:regroot} to the case of lazy insertion interval trees. The result is derived similarly, based on the lazy update rule~\eqref{alg:lupdl'1}. Due to the design of lazy insertion trees, the cumulative reward vector maintained at the root remains unscaled at the time of insertion. This vector is scaled appropriately during the execution of the $\Draw$ algorithm when it is needed for computing rewards.

\begin{corollary}\label{corr:lazroot}
Let \( v_0 \in (p_{l_0} \cup p_{h_0}) \setminus p_{v_{l_0,h_0}} \) be a node in the horizontal interval tree \(\Tree^{(0)}\), such that the horizontal range of either its left or right child is fully contained in the horizontal interval of the inserted region \(\rg\). Then, the cumulative reward vector maintained at the root of the corresponding lazy insertion tree—either \(\vec{\CR}_{v_0, \ROOT(\LZL^{(1)}_{v_0})}\) or \(\vec{\CR}_{v_0, \ROOT(\LZR^{(1)}_{v_0})}\)—correctly accumulates the \textbf{unscaled} reward contributions of all regions that lazily affect the subtree rooted at \(\LC(v_0)\) or \(\RC(v_0)\), respectively, over the region \([\RA(v_0)] [0{:}1]\).
\end{corollary}

We now present additional structural properties that distinguish the regular insertion trees associated with leaf versus non-leaf nodes \(v_0 \in \Tree^{(0)}\), as well as highlight key differences between the regular and lazy insertion data structures. The following Lemma~\ref{lemma:regleaf} establishes a key structural property of the regular insertion data structures \(\Tree^{(1)}_{v_0}\) associated with the leaf nodes \(v_0 \in \Tree^{(0)}\). The proof hinges on how lazy \(\FP\) updates are applied and cleared within these structures. In particular, the fact that \(v_0\) is a leaf guarantees that all updates are correctly propagated, ensuring that cumulative reward vectors stored throughout the tree are accurate and up-to-date.

\begin{lemma}\label{lemma:regleaf}
Let \(v_0 \in \Tree^{(0)}\) be a leaf node (e.g., \(v_0 = l_0\)). Then, for every node \(v_1 \in \Tree^{(1)}_{l_0}\), the cumulative reward vector \(\vec{\CR}_{l_0, v_1}\) correctly represents the cumulative rewards over the region \([\RA(l_0)] [\RA(v_1)]\).
\end{lemma}

\begin{proof}
The key observation underlying Lemma~\ref{lemma:regleaf} is that when \(v_0 = l_0\) is a leaf, its horizontal range \(\RA(l_0)\) is an atomic interval—i.e., it contains no other endpoints. This implies that any region whose vertical interval is inserted into \(\Tree^{(1)}_{l_0}\) must have a horizontal interval that fully contains \(\RA(l_0)\); otherwise, the existence of an endpoint within \(\RA(l_0)\) would contradict the assumption that \(l_0\) is a leaf. As a result, all cumulative rewards in \(\Tree^{(1)}_{l_0}\) are over regions that span the full horizontal range \(\RA(l_0)\), ensuring correctness at every node \(v_1 \in \Tree^{(1)}_{l_0}\).

Therefore, during the regular $\FP$ update  while inserting the vertical interval of $\rg$,$[x^{(1)}_1 : x^{(1)}_2]$, in the \(\Tree^{(1)}_{l_0}\), while traversing the paths $p_{l_1}$ and $p_{h_1}$ bottom-up, the cumulative reward vectors for each node $v_1 \in p_{l_1} \cup p_{h_1}$ is added with the cumulative reward vector \(\vec{\SC}_{\rg'}\) of the region $\rg' = [\RA(l_0) \cap [x^{(0)}_1 : x^{(0)}_2]][\RA(v_1) \cap [x^{(1)}_1 : x^{(1)}_2]]$ which is $[\RA(l_0) ][\RA(v_1) \cap [x^{(1)}_1 : x^{(1)}_2]]$ cause $l_0$ is a leaf. Thus, the regular $\FP$ update is correctly done for such node. 

An additional component in the proof involves verifying the correctness of the cumulative reward vectors at nodes within the subtree rooted at the first common ancestor \(v_{l_1,h_1}\), which receive lazy \(\FP\) update messages during propagation. This follows from the lazy update rule~\eqref{eq:lz1}, which uses the accumulated \(\FP\) message stored in \(\LZM_{v_1}\) at each node \(v_1\) to update its cumulative reward vector by adding \(\vec{\SC}_{\rg'}\), where \(\rg' = [\RA(l_0)][\RA(v_1)]\) and the coefficients vector is the $\FP$ vectors in \(\LZM_{v_1}\). The subsequent propagation steps~\eqref{eq:lz2} and~\eqref{eq:lz3} pass this lazy message to both children of \(v_1\). Notably, since all such updates are computed over the fixed horizontal range \(\RA(l_0)\), which matches the full horizontal span of every affected region in \(\Tree^{(1)}_{l_0}\), these lazy updates correctly contribute to the cumulative rewards and naturally align with the regular updates—ensuring correctness throughout the subtree.
\end{proof}
We emphasize that Lemma~\ref{lemma:regleaf} does not extend to regular insertion trees \(\Tree^{(1)}_{v_0}\) associated with non-leaf nodes \(v_0\). The key issue is that, in these cases, the lazy update messages originate from regions whose horizontal intervals may intersect only a subinterval of the node's horizontal range \(\RA(v_0)\), rather than fully containing it. Consequently, the cumulative lazy \(\FP\) updates are no longer additive across all such regions, as the integration intervals vary and do not align uniformly with \(\RA(v_0)\). Preserving the additivity property in this setting would require explicitly tracking the horizontal intersection of every inserted region with \(\RA(v_0)\), significantly increasing the complexity and inefficiency of the data structure.

\begin{corollary}\label{corr:lazleaf}
For any lazy insertion data structure, either \(\LZL^{(1)}_{v_0}\) or \(\LZR^{(1)}_{v_0}\), the cumulative reward vector \(\vec{\CR}_{v_0, w_1}\) maintained at any node \(w_1\) correctly represents the scalable cumulative reward over the region \([\RA(v_0)][\RA(w_1)]\), for all \(v_0 \in \Tree^{(0)}\).
\end{corollary}

The intuition behind the Corollary~\ref{corr:lazleaf} lies in the structure of the lazy insertion trees: both regular and lazy \(\FP\) updates are computed by integrating over regions whose horizontal interval exactly matches \(\RA(v_0)\). As a result, all contributions are aligned over a fixed base interval, ensuring that the cumulative reward vectors across nodes in the tree remain additive. This structural uniformity guarantees that the cumulative reward at any internal node can be decomposed into the sum of rewards over its subregions, preserving scalability for downstream computations.

Beyond these structural properties, which follow directly from the design of our \(\Ins\) and \(\Update\) operations, we now present two key lemmas—Lemma~\ref{lemma:unique} and Lemma~\ref{lemma:complete}—that are fundamental to ensuring the correctness of cumulative reward retrieval. These lemmas formalize how our data structure supports exact and complete access to cumulative rewards queried by the algorithm, as detailed in the next section.

\subsection{Correctness of $\Draw$}
In this section, we prove that the \(\Draw\) algorithm samples an atomic region with probability proportional to its cumulative reward. As discussed earlier, \(\Draw\) proceeds in two stages—\emph{horizontal} and \emph{vertical sampling}—each relying on a variant of the cumulative reward retrieval subroutine, denoted by \(\Wei\). The correctness of \(\Draw\) thus fundamentally depends on the correctness of these \(\Wei\) procedures. We first analyze the two variants, \(\Wei^{(0)}\) and \(\Wei^{(1)}\), in order. Prior to that, we establish two key structural lemmas—Lemmas~\ref{lemma:unique} and~\ref{lemma:complete}—which underpin the correctness and simplify the subsequent analysis.

\paragraph{Uniqueness and Completeness of the $\IT$s }
We now analyze two fundamental structural properties—\emph{uniqueness} and \emph{completeness}—that are critical to the correctness of our cumulative reward retrieval procedures. When handling a \(\Wei\) query, either \(\Wei^{(0)}\) or \(\Wei^{(1)}\), for a rectangular region \(\rg : [x^{(0)}_1:x^{(0)}_2][x^{(1)}_1:x^{(1)}_2]\), the algorithm identifies all data structures that contain cumulative reward information relevant to regions intersecting \(\rg\). It then extracts the corresponding contributions—either directly or via scaling—and aggregates them to compute the cumulative reward over \(\rg\). To ensure correctness, two conditions must be met: First, \textbf{uniqueness}—each intersecting region must contribute through exactly one of the collected data structures to avoid double-counting. Second, \textbf{completeness}—all regions that intersect \(\rg\) must be represented within the union of the collected data structures; omitting any such region would lead to inaccurate estimates. The following lemmas formalize and prove these two properties.

To retrieve the cumulative reward of a region \(\rg : [x^{(0)}_1:x^{(0)}_2][x^{(1)}_1:x^{(1)}_2]\), our algorithm first identifies the node \(v_0 \in \Tree^{(0)}\) whose horizontal range minimally contains \([x^{(0)}_1:x^{(0)}_2]\). It then constructs the set \(\mathcal{S}^{(1)}_{v_0}\) of second-level interval trees relevant for querying the vertical interval. This set includes the regular insertion tree \(\Tree^{(1)}_{v_0}\), as well as lazy insertion trees associated with the ancestors \(v_0'\) of \(v_0\) along the path to the root: if \(v_0\) lies in the left (resp., right) subtree of \(v_0'\), we include \(\LZL^{(1)}_{v_0'}\) (resp., \(\LZR^{(1)}_{v_0'}\)) in \(\mathcal{S}^{(1)}_{v_0}\). We refer to each element in this set by \(S_{v'_0}\), denoting the second-level data structure associated with the ancestor \(v'_0\). The complete construction of \(\mathcal{S}^{(1)}_{v_0}\) is detailed in Section~\ref{Sec:Weigeneral} and Algorithm~\ref{alg:crr1}.

Now, the following Lemmas~\ref{lemma:unique} and~\ref{lemma:complete} guarantee the uniqueness and completeness of $\mathcal{S}_{v_0}$ for any node $v_0 \in \Tree^{(0)}$.

\begin{lemma}[Uniqueness of Endpoints Across \(\mathcal{S}^{(1)}_{v_0}\)]\label{lemma:unique}
Let \(v_0 \in \Tree^{(0)}\) be a fixed node, and let \(\mathcal{S}^{(1)}_{v_0}\) denote the set of second-level data structures constructed along the path \(p_{v_0}\) from \(v_0\) to the root of \(\Tree^{(0)}\), as defined in Algorithm~\ref{alg:crr1}. Then, for any region whose vertical interval has been inserted into one of these structures, this interval appears in exactly one data structure in \(\mathcal{S}^{(1)}_{v_0}\). That is, no two data structures \(S_{v_0'}, S_{v_0''} \in \mathcal{S}^{(1)}_{v_0}\) maintain atomic vertical intervals at their leaves that share a common endpoint.
\end{lemma}

\begin{lemma}[Completeness of Endpoints Across \(\mathcal{S}^{(1)}_{v_0}\)]\label{lemma:complete}
For a given node \(v_0 \in \Tree^{(0)}\), let \(\mathcal{S}^{(1)}_{v_0}\) denote the set of data structures along the path \(p_{v_0}\) from the root of the outermost tree \(\ROOT(\Tree^{(0)})\) to the node \(v_0\), as constructed by Algorithm~\ref{alg:crr1}. Then, for every region inserted up to round \(t\), if its horizontal interval intersects the horizontal range \(\RA(v_0)\), its vertical interval is contained in at least one data structure \(S_{v_0'} \in \mathcal{S}^{(1)}_{v_0}\). In other words, the union of all data structures in \(\mathcal{S}^{(1)}_{v_0}\) includes all vertical intervals relevant to \(v_0\) that have been inserted so far.
\end{lemma}

\begin{proof}[Proof of Lemmas~\ref{lemma:unique} and~\ref{lemma:complete}]

  The proof of both lemmas highly depend on our approach in inserting regions and updating the information about the cumulative rewards after each insertion. 

  By construction, each second-level data structure associated with a node \(v_0 \in \Tree^{(0)}\)—either the regular insertion tree \(\Tree^{(1)}_{v_0}\) or a lazy insertion tree of one of its ancestors—stores the vertical intervals of regions whose horizontal intervals intersect the horizontal range \(\RA(v_0)\). In particular, if a region \(\rg\) is inserted such that its horizontal interval intersects (but does not fully contain) \(\RA(v_0)\), then its vertical interval is inserted into \(\Tree^{(1)}_{v_0}\), the regular insertion tree of \(v_0\).

    On the other hand, if the horizontal interval of \(\rg\) fully contains \(\RA(v_0)\), then by definition it is inserted into a lazy insertion data structure associated with an ancestor \(v_0'\) of \(v_0\). Our data structure ensures that vertical intervals inserted into such lazy trees are never propagated downward into the second-level structures of descendants. Instead, their contributions are correctly scaled during retrieval when needed.

Therefore, the union of the vertical intervals stored across all second-level structures in \(\mathcal{S}^{(1)}_{v_0}\)—comprising \(\Tree^{(1)}_{v_0}\) and the lazy insertion trees of \(v_0\)'s ancestors along the path to the root—contains exactly the vertical intervals of all regions whose horizontal intervals intersect \(\RA(v_0)\), including those that fully contain it. Moreover, due to the regular insertion routine, any vertical interval inserted into a second-level tree associated with a descendant of \(v_0\) must also be inserted into \(\Tree^{(1)}_{v_0}\), ensuring no relevant information is omitted. This proves that every vertical interval relevant to \(\RA(v_0)\) is captured by \(\mathcal{S}^{(1)}_{v_0}\), completing the proof of completeness.

Recall from the construction of \(\mathcal{S}^{(1)}_{v_0}\) that this set includes exactly one regular insertion tree \(\Tree^{(1)}_{v_0}\), along with a collection of lazy insertion trees associated with the ancestors of \(v_0\) along the path \(p_{v_0}\) to the root of \(\Tree^{(0)}\). By design, each lazy insertion tree \(\LZL^{(1)}_{v_0'}\) or \(\LZR^{(1)}_{v_0'}\) in \(\mathcal{S}^{(1)}_{v_0}\) stores only those vertical intervals from regions whose horizontal intervals fully contain the horizontal range of either \(\LC(v_0')\) or \(\RC(v_0')\). Moreover, our insertion procedure ensures that once a vertical interval is inserted into a lazy structure of an ancestor \(v_0'\), it is not inserted into any descendant's second-level lazy data structure on the path \(p_{v_0}\). Thus, no two lazy insertion trees in \(\mathcal{S}^{(1)}_{v_0}\) share any common vertical interval.

It remains to verify that \(\Tree^{(1)}_{v_0}\) does not share any vertical interval with the lazy insertion trees of its ancestors. Suppose a vertical interval is inserted into \(\Tree^{(1)}_{v_0}\). Then, by the regular insertion routine, it is also inserted into the second-level regular insertion trees of all nodes along the path from the corresponding leaf to the root of \(\Tree^{(0)}\), including \(v_0\) and its ancestors in \(p_{v_0}\). Crucially, our algorithm prevents such an interval from being inserted into any lazy structure affecting \(\RA(v_0)\), as lazy insertions are restricted to regions whose horizontal interval strictly contains \(\RA(v_0)\) and it is not along the path $p_{v_0}$, and these are handled exclusively at the level of the relevant ancestor. Consequently, no vertical interval can appear in both \(\Tree^{(1)}_{v_0}\) and any lazy insertion tree in \(\mathcal{S}^{(1)}_{v_0}\), which completes the proof of uniqueness.
\end{proof}

Lemmas~\ref{lemma:unique} and~\ref{lemma:complete} guarantee a disjoint and complete organization of the second-level data structures that either directly or lazily affect a node \( v \in \Tree^{(0)} \), denoted by \(\mathcal{S}_v\) in Algorithm~\ref{alg:crr1}. This structural property enables efficient and modular handling of cumulative reward queries. A key observation is that each lazy data structure in \(\mathcal{S}_v\) uniformly influences its corresponding subtree, and thus can be queried independently. Combined with the scaling technique introduced in Section~\ref{section:scaling}, which preserves correctness while mapping contributions from various structures to the target range of \(v\), this design avoids redundant insertions at every round.

With the structural properties established in Lemmas~\ref{lemma:unique} and~\ref{lemma:complete}, we are now ready to formally prove the correctness of Algorithm~\ref{alg:crr2} and~\ref{alg:crr3}, which implements the cumulative reward retrieval procedure \(\Wei\). These lemmas ensure that all relevant contributions to the cumulative reward of a given query region are uniquely and completely represented within the data structures comprising \(\mathcal{S}^{(1)}_{v_0}\), thereby enabling accurate and efficient computation via \(\Wei\).

\paragraph{Correctness of $\Wei$ Query:}
As previously discussed, the cumulative reward retrieval procedure \(\Wei\) is decomposed into two specialized subroutines, \(\Wei^{(0)}(\cdot)\) and \(\Wei^{(1)}(\cdot)\), for handling queries in the two-dimensional action space. These are detailed in Section~\ref{Sec:Weigeneral}. We now formally establish the correctness of each subroutine through the following lemmas.

\begin{lemma}\label{lemma:Wei0}
For any region \(\rg = [x^{(0)}_1:x^{(0)}_2][0:1]\), the cumulative reward retrieval algorithm (Algorithm~\ref{alg:crr2} and~\ref{alg:crr3}) correctly computes \(\CR(\rg)\), the total cumulative reward over the region \(\rg\).
\end{lemma}
\begin{proof}
As detailed in Section~\ref{Sec:Wei0}, handling a query of the form \(\Wei^{(0)}(\rg)\) begins by identifying the node \(v_{l_0,h_0}\), defined as the lowest common ancestor of the leaves \(l_0\) and \(h_0\) whose horizontal ranges contain the endpoints of the queried region \(\rg\). The algorithm then retrieves the cumulative reward over the region \([\RA(v_{l_0,h_0})][0:1]\), followed by pruning the contribution of the portions outside \(\rg\). Hence, to prove correctness, it suffices to show that we can correctly retrieve the cumulative reward over any region of the form \([\RA(v_0)] [0:1]\) for an arbitrary node \(v_0 \in \Tree^{(0)}\), as this will directly imply correctness for \(v_{l_0,h_0}\) and any pruned subregions.

This retrieval is performed by first constructing the set \(\mathcal{S}^{(1)}_{v_0}\), which contains all second-level data structures—both regular and lazy—influencing node \(v_0\). The algorithm then computes the contribution of each structure to the region \([\RA(v_0)][0:1]\) independently and aggregates the results. This aggregation follows the detachment rule described in Equation~\eqref{crew:detach}, which decomposes the total reward into additive components corresponding to disjoint sources. In particular, the detachment separates the contribution from the regular insertion tree \(\Tree^{(1)}_{v_0}\)—stored directly at its root and requiring no scaling—from the contributions of the lazy insertion trees associated with the ancestors of \(v_0\) along the path \(p_{v_0}\), each of which must be appropriately scaled before aggregation. The relevant cumulative reward vectors are maintained at the roots of these data structures, as guaranteed by Lemma~\ref{lemma:regroot} and Corollary~\ref{corr:lazroot}. The procedures for retrieving and scaling these vectors are discussed in detail in Section~\ref{Sec:Wei0}.

Here, we focus on establishing the correctness of Equation~\eqref{crew:detach}, which follows directly from the structural guarantees of \emph{uniqueness} and \emph{completeness} for the set \(\mathcal{S}^{(1)}_{v_0}\), as formalized in Lemmas~\ref{lemma:unique} and~\ref{lemma:complete}.

\end{proof}
\begin{lemma}\label{lemma:Wei1}
Let \(\rg = [x'^{(0)}_1:x'^{(0)}_2] [x^{(1)}_1:x^{(1)}_2]\) be a rectangular region such that \([x'^{(0)}_1:x'^{(0)}_2]\) is an atomic horizontal interval. Then, the vertical cumulative reward retrieval procedure \(\Wei^{(1)}(\rg)\), as implemented in Algorithm~\ref{alg:crr2}, correctly returns \(\CR(\rg)\)—the total cumulative reward over the region \(\rg\).
\end{lemma}

\begin{proof}
Before proceeding with the proof, observe that the horizontal interval of the query region \(\rg = [x'^{(0)}_1:x'^{(0)}_2][x^{(1)}_1:x^{(1)}_2]\) is atomic. Therefore, there exists a leaf node \(l_0 \in \Tree^{(0)}\) such that \(\RA(l_0) = [x'^{(0)}_1:x'^{(0)}_2]\), and we henceforth fix \(v_0 := l_0\) to reflect this property, as established in Section~\ref{sec:Wei1}.

In contrast to the \(\Wei^{(0)}\) retrieval, which accesses cumulative rewards stored at the roots of the data structures in \(\mathcal{S}^{(1)}_{l_0}\), the vertical interval \([x^{(1)}_1:x^{(1)}_2]\) in \(\rg\) may not align with the range of any single node. Consequently, cumulative rewards must be computed by aggregating contributions from nodes whose vertical ranges intersect \([x^{(1)}_1:x^{(1)}_2]\), requiring partial evaluation and pruning of leaf-level reward vectors.

To accomplish this, the algorithm constructs the set \(\mathcal{S}^{(1)}_{l_0}\) and detaches the contribution of each of its members to the region \(\rg\), as outlined in Equation~\eqref{crew:detachplus}. The correctness of this decomposition relies on the scalability of the cumulative reward vectors maintained within these data structures. Specifically, Lemma~\ref{lemma:regleaf} ensures that all nodes in the unique regular insertion tree \(\Tree^{(1)}_{l_0} \in \mathcal{S}^{(1)}_{l_0}\) store cumulative rewards over rectangles of the form \([\RA(l_0)][\RA(v_1)]\), which enables precise scaling and, consequently, accurate pruning with respect to the vertical interval. Likewise, Corollary~\ref{corr:lazleaf} guarantees that the lazy insertion trees in \(S^{(1)}_{v_0} \in \mathcal{S}^{(1)}_{l_0}\), whose cumulative reward vectors are all defined and integrated over the fixed horizontal interval \(\RA(v_0)\), support correct scaling to any vertical subinterval \([x^{(1)}_1:x^{(1)}_2]\).

Together, these properties ensure that each data structure in \(\mathcal{S}^{(1)}_{l_0}\) can be queried independently, pruned accurately, and aggregated to compute \(\CR(\rg)\) correctly. The full details of this procedure, including leaf-level pruning and scaling operations, are presented in Section~\ref{sec:Wei1}.

\end{proof}

\paragraph{Correctness of $\Draw$ Query}
In this section we prove that our drawing Algorithm~\ref{alg:draw1} and~\ref{alg:draw2} finds an atomic region with the correct probability. 
\begin{lemma}\label{lemma:draw}
    At each time step $t$, Algorithm~\ref{alg:draw1} and~\ref{alg:draw2}, which implements the $\Draw(t)$ operation, samples an atomic region $\rg$ with probability proportional to its cumulative reward.
\end{lemma}

\begin{proof}
As described in Algorithm~\ref{alg:draw1} and~\ref{alg:draw2}, the drawing process begins by selecting a region of the form \(\rg_1 = [x^{(0)}_1:x^{(0)}_2][0:1]\), where the horizontal interval \([x^{(0)}_1:x^{(0)}_2]\) corresponds to a leaf of \(\Tree^{(0)}\). This region is selected with probability proportional to the cumulative reward \(\CR(\rg_1)\) among all such horizontal slices. Then, within the selected region \(\rg_1\), the algorithm samples an atomic region \(\rg_2 = [x^{(0)}_1:x^{(0)}_2][x^{(1)}_1:x^{(1)}_2]\) with probability proportional to its cumulative reward relative to \(\rg_1\). Consequently, the final region \(\rg = \rg_2\) is drawn with probability proportional to its cumulative reward among all atomic regions in the domain \([0:1][0:1]\).

If the algorithm selects the horizontal region \(\rg_1\) in the first stage with probability
\[
\frac{\CR(\rg_1)}{\CR([0:1][0:1])},
\]
and then samples the atomic region \(\rg_2\) from within \(\rg_1\) with probability
\begin{align}
    \frac{\CR(\rg_2)}{\CR(\rg_1)}, \label{eq:versamplingmain}
\end{align}

then the marginal probability of drawing \(\rg_2\) is
\begin{align*}
   \Pr(\Draw(t) = \rg_2) &= \frac{\CR(\rg_1)}{\CR([0:1][0:1])} \cdot \frac{\CR(\rg_2)}{\CR(\rg_1)} \\
   &= \frac{\CR(\rg_2)}{\CR([0:1][0:1])}.
\end{align*}
Thus, Algorithm~\ref{alg:draw1} and~\ref{alg:draw2} samples each atomic region \(\rg_2\) with probability proportional to its cumulative reward.

During the first stage of handling the $\Draw$ query, Algorithm~\ref{alg:draw1} and~\ref{alg:draw2} follows a randomized path from the root node \( u_0 = \ROOT(\ATree^{(0)}) \) to a leaf. The probability that the drawing algorithm includes either \(\LC(u_0)\) or \(\RC(u_0)\) along this path directly influences which leaf is ultimately selected.

Initially, the algorithm starts at \( u_0 = \ROOT(\ATree^{(0)}) \), which yields
\[
\Pr \left(u_0 \gets \ROOT(\ATree^{(0)}) \right) =  \frac{\CR([\RA(\ROOT(\Tree^{(0)}))][0:1])}{\CR([0:1][0:1])} = \frac{\CR([0:1][0:1])}{\CR([0:1][0:1])} = 1.
\]

At each internal node \( u_0 \), the algorithm proceeds to the right child \(\RC(u_0)\) with probability
\[
\frac{\CR([\RA(\RC(u_0))][0:1])}{\CR([\RA(u_0)][0:1])}.
\]
Thus, the marginal probability that \(\RC(u_0)\) is included in the randomized path is
\begin{align*}
  \Pr \left(u_0 \gets \RC(u_0) \right) &= \Pr(u_0 \text{ appears on the path}) \times \frac{\CR([\RA(\RC(u_0))][0:1])}{\CR([\RA(u_0)][0:1])} \\
  &= \frac{\CR([\RA(u_0)][0:1])}{\CR([\RA(\PAR(u_0))][0:1])} \times \frac{\CR([\RA(\RC(u_0)][0:1]))}{\CR([\RA(u_0)][0:1])} \\
  &= \frac{\CR([\RA(\RC(u_0))][0:1])}{\CR([\RA(\PAR(u_0))][0:1])}.
\end{align*}

By applying this argument inductively, we conclude that for any node \( u_0 \in \ATree^{(0)} \), including the leaves,
\begin{align}
    \Pr\left(u_0 \text{ appears on the path}\right) = \frac{\CR([\RA(u_0)][0:1])}{\CR([0:1][0:1])}. \label{eq:horsampling}
\end{align}
Therefore, as long as the values \( \CR([\RA(u_0)][0:1]) \) are computed correctly, the first stage of our drawing algorithm selects the region \( \rg_1 = [\RA(u_0)][0:1] \) with the correct probability. These queried regions \( [\RA(u_0)][0:1] \) are precisely the ones for which \(\Wei^{(0)}\) is responsible. Hence, as long as \(\Wei^{(0)}([\RA(u_0)][0:1])\) retrieves the correct cumulative reward for its region—guaranteed by Lemma~\ref{lemma:Wei0}—the sampling remains correct. This completes the proof of the correctness of the  horizontal sampling.

Now it remains to prove the correctness of \(\Draw(t)\) through the vertical sampling stage.

We show that the \(\Draw(t)\) algorithm locates the atomic region \(\rg_2\) within the previously selected region \(\rg_1\) with the correct probability \(\frac{\CR(\rg_2)}{\CR(\rg_1)}\). This part of the proof closely follows the same reasoning used to establish the correctness of the horizontal sampling stage.

Let \(l_0\) denote the leaf in \(\ATree^{(0)}\) corresponding to the horizontal interval of the selected region \(\rg_1 = [\RA(l_0)][0:1]\). During the vertical sampling stage, the drawing algorithm begins at the root \(u_1 = \ROOT(\ATree^{(1)})\) and moves to its right child \(\RC(u_1)\) with probability
\[
\Pr\left(u_1 \gets \RC(u_1)\right) = \frac{\CR([\RA(l_0)][\RA(\RC(u_1))])}{\CR([\RA(l_0)][\RA(u_1)])},
\]
and to the left child with the complementary probability.

By an inductive argument similar to the one used for horizontal sampling, we conclude that for any node \(u_1 \in \ATree^{(1)}\), including the leaves, the probability that it appears on the traversal path is
\begin{align}
    \Pr\left(u_1 \text{ appears on the path}\right) = \frac{\CR([\RA(l_0)][\RA(u_1)])}{\CR([\RA(l_0)][0:1])}. \label{eq:versampling}
\end{align}

The correctness of this stage depends on the accurate computation of the cumulative rewards over vertical slabs within the region \(\rg_1\), which is precisely the responsibility of the query \(\Wei^{(1)}\). Its correctness is guaranteed by Lemma~\ref{lemma:Wei1}.

Finally, combining Equation~\eqref{eq:versampling} with the horizontal sampling 
result from Equation~\eqref{eq:horsampling}, we obtain the overall probability 
that \(\Draw(t)\) returns the region \(\rg_2 = [\RA(l_0)][\RA(l_1)]\):
\begin{align*}
    \Pr(\rg_2 \gets \Draw(t)) 
    &= \Pr(\rg_1 \gets \text{Horizontal Sampling}) 
       \times \Pr(\rg_2 \gets \text{Vertical Sampling} \mid \rg_1) \\
    &= \frac{\CR([\RA(l_0)][0:1])}{\CR([0:1][0:1])}
    \times \frac{\CR([\RA(l_0)][\RA(l_1)])}{\CR([\RA(l_0)][0:1])} \\
    &= \frac{\CR([\RA(l_0)][\RA(l_1)])}{\CR([0:1][0:1])},
\end{align*}
which matches Equation~\eqref{eq:versamplingmain}, completing the proof.
\end{proof}
\section{Generalization to $d$-dimensions (\Cref{sec:ds})} \label{section:generalpar}
 We now extend this construction recursively to a $d$-layered data structure, denoted by $\Tree$. The outermost layer is $\Tree^{(0)}$, whose nodes \(v_0 \in \Tree^{(0)}\) each maintain three second-level data structures: a regular structure $\Tree^{(1)}_{v_0}$, and two lazy structures $\LZL^{(1)}_{v_0}$ and $\LZR^{(1)}_{v_0}$. This organization is applied recursively to all levels: every node $v_1$ in the second level structures (i.e., in $\Tree^{(2)}_{v_0,v_1}$, $\LZL^{(1)}_{v_0,v_1}$, or $\LZR^{(1)}_{v_0,v_1}$) maintains its own set of third-level structures, namely $\Tree^{(2)}_{v_0,v_1}$, $\LZL^{(2)}_{v_0,v_1}$, and $\LZR^{(2)}_{v_0,v_1}$, following the same pattern and details as in the warm-up case.

This recursive structure continues until the $(d-1)$-th level. In general, for any node $v_i$ at the $i$-th level, we maintain:
\[
\Tree^{(i+1)}_{v_0, v_1, \dots, v_i}, \quad
\LZL^{(i+1)}_{v_0, v_1, \dots, v_i}, \quad
\LZR^{(i+1)}_{v_0, v_1, \dots, v_i}.
\]
At the $d$-th level, the leaf nodes store all required information regarding the region's cumulative rewards stored as vectors $\vec{\CR}$. Specifically, for each node $v_{d-1}$ in a $d$-th level tree, we store 
\[
\vec{\CR}_{v_0, \dots, v_{d-1}},
\]
which encode the cumulative rewards for their corresponding subregions and are used for scaling and lazy propagation in higher-level queries.

 Next, we describe how the operations $\Ins$, $\Update$, $\Wei$, and $\Draw$ are performed in this recursive setting for a given region $\rg$ of the form:
$$
\rg = [x^{(0)}_1:x^{(0)}_2][x^{(1)}_1:x^{(1)}_2]\cdots[x_1^{(d-1)}:x_2^{(d-1)}].
$$

 \subsection{$\Ins(\rg)$}
 We begin by locating the interval \([x^{(0)}_1:x^{(0)}_2]\) of the region \(\rg\) along the first coordinate in \(\Tree^{(0)}\). We identify the corresponding leaf nodes \(l_0\) and \(u_0\), and compute their lowest common ancestor, denoted by \(v_{l_0,u_0}\). We then follow the same insertion procedure as described in Algorithm~\ref{alg:ins1}. Specifically, we traverse the nodes \(v_0\) along the paths \(p_{l_0}\), \(p_{u_0}\), and \(p_{v_{l_0,u_0}}\) in a bottom-up manner. For each such node \(v_0 \in \Tree^{(0)}\), we insert the remaining coordinate intervals
\[
[x^{(1)}_1:x^{(1)}_2], \dots, [x^{(d-1)}_1:x^{(d-1)}_2]
\]
into one of the associated second-level data structures: \(\Tree^{(1)}_{v_0}\), \(\LZL^{(1)}_{v_0}\), or \(\LZR^{(1)}_{v_0}\), depending on whether \(v_{l_0,u_0}\) lies in the left or right subtree of \(v_0\).

We apply this process recursively at deeper levels. For each node \(v_0\), we locate the interval \([x^{(1)}_1:x^{(1)}_2]\) within its associated second-level tree, identify the corresponding leaves \(l_1\) and \(u_1\), and determine their lowest common ancestor \(v_{l_1,u_1}\). We then traverse the paths \(p_{l_1}\), \(p_{u_1}\), and \(p_{v_{l_1,u_1}}\), and insert the remaining intervals \([x^{(2)}_1:x^{(2)}_2], \dots, [x^{(d-1)}_1:x^{(d-1)}_2]\) into the appropriate third-level data structure: \(\Tree^{(2)}_{v_0,v_1}\), \(\LZL^{(2)}_{v_0,v_1}\), or \(\LZR^{(2)}_{v_0,v_1}\).

This recursive procedure continues until we reach the \((d{-}1)\)-th level. There, the final interval \([x^{(d-1)}_1:x^{(d-1)}_2]\) is inserted into one of the data structures:
\[
\Tree^{(d-1)}_{v_0, v_1, \dots, v_{d-2}},\quad \LZL^{(d-1)}_{v_0, v_1, \dots, v_{d-2}},\quad \text{or} \quad \LZR^{(d-1)}_{v_0, v_1, \dots, v_{d-2}}.
\]

 \subsection{$\Update(\rg)$}
 To handle this query, we once again locate the affected nodes in each layer---specifically, those that were identified and inserted into their respective lazy or regular data structures during the $\Ins$ query. Suppose we reach the $d$-th level tree after traversing the sequence of nodes $v_0, v_1, \dots, v_{d-2}$, arriving at $\Tree^{(d-1)}_{v_0,v_1,\dots,v_{d-2}}$. In this tree, we locate the leaves corresponding to the endpoints of the interval $[x_1^{(d-1)}:x_2^{(d-1)}]$, denoted by $l_{d-1}$ and $u_{d-1}$, respectively. Their first common ancestor is referred to as $v_{l_{d-1},u_{d-1}}$. We then traverse the nodes $v_{d-1}$ along the paths $p_{l_{d-1}}$, $p_{u_{d-1}}$, and $p_{v_{l_{d-1},u_{d-1}}}$, and at each such node, compute the following vector \(\vec{\SC}_{\rg'},
\)
for $\rg' = [\RA(v_0)],[\RA(v_1)],\dots,[\RA(v_{d-1})]$, where each vector $\vec{\SC}^{(i)}_{\rg'}$ is computed according to equation~\eqref{eq:scalingmain}, using the coefficient vector $\fu_t(\rg)$ and then aggregate it with the $\vec{\CR}_{v_0,v_1,\cdots,v_{d-1}}$ to update it. As in the warm-up case, we then update the reward vector $\vec{\CR}_{v_0,v_1,\dots,v_{d-1}}$ at each visited node with the value $\fu_t(\rg)$. The update process, including the lazy update messages, follows exactly the same rule as in the warm-up case. 

 \subsection{$\Wei(\rg)$}
Handling this query efficiently is central to extending the methods developed in the warm-up case to the general \(d\)-dimensional axis-parallel setting. As introduced earlier, our drawing algorithm decomposes a cumulative reward query into a sequence of subqueries of the form:
\begin{align*}
    &\Wei^{(0)}\left(\rg : [x^{(0)}_1:x^{(0)}_2][0:1] \cdots [0:1]\right),\\
    &\Wei^{(1)}\left(\rg : [x'^{(0)}_1:x'^{(0)}_2][x^{(1)}_1:x^{(1)}_2][0:1] \cdots [0:1]\right),\\
    &\quad \vdots\\
    &\Wei^{(d-1)}\left(\rg : [x'^{(0)}_1:x'^{(0)}_2][x'^{(1)}_1:x'^{(1)}_2] \cdots [x'^{(d-2)}_1:x'^{(d-2)}_2][x^{(d-1)}_1:x^{(d-1)}_2]\right),
\end{align*}
where each \([x'^{(i)}_1:x'^{(i)}_2]\) is an atomic interval along the \(i\)-th coordinate, fixed during the \(i\)-th stage of the drawing process. While the mechanics of each subquery follow the warm-up case, the challenge in higher dimensions lies in identifying the correct set of data structures and ensuring the cumulative reward vectors they store can be accurately scaled.

At each stage \(i+1\), a key step is to gather all relevant \((i+1)\)-th level data structures that affect the current subspace:
\[
[x'^{(0)}_1:x'^{(0)}_2],\ [x'^{(1)}_1:x'^{(1)}_2],\ \dots,\ [x'^{(i{-}1)}_1:x'^{(i{-}1)}_2],\ [0:1],\ \dots,\ [0:1].
\]
These structures are assembled into the set \(\mathcal{D}^{(i)}_{l_0, l_1, \dots, l_{i{-}1}}\), where each \(l_j\) is a leaf node in the auxiliary tree \(\ATree^{(j)}\) with \(\RA(l_j) = [x'^{(j)}_1:x'^{(j)}_2]\). This naming reflects the assumption that each \(\RA(l_j)\) is an atomic interval---that is, one stored entirely at a leaf node.

The construction of \(\mathcal{D}^{(i)}_{l_0, l_1, \dots, l_{i{-}1}}\) proceeds recursively. We begin by locating the leaf \(l'_0 \in \Tree^{(0)}\) such that \(\RA(l'_0) \subseteq \RA(l_0)\), and then traverse its ancestor path \(p_{l'_0}\) to collect all relevant lazy insertion structures, as described in the warm-up case. 

For each structure collected at the first level, we repeat the process at the second level: we locate a leaf \(l'_1\) such that \(\RA(l'_1) \subseteq \RA(l_1)\), traverse its ancestor path \(p_{l'_1}\), and gather all associated lazy insertion structures. This continues recursively: at each newly collected structure from level \(j-1\), we locate the leaf containing \(\RA(l_j)\), traverse upward, and collect relevant structures at level \(j+1\), for \(j = 2, \dots, i-1\).

The recursion terminates when all \(i\)-th level structures are collected whose root's associated region fully contains the subspace:
\[
[x'^{(0)}_1:x'^{(0)}_2],\ [x'^{(1)}_1:x'^{(1)}_2],\ \dots,\ [x'^{(i{-}1)}_1:x'^{(i{-}1)}_2],\ [0:1],\ \dots,\ [0:1].
\]
Each such structure is denoted by \(D^{(i)}_{v'_0, v'_1, \dots, v'_{i{-}1}}\), where \(v'_j\) is a node on the path \(p_{l'_j}\). The structure \(D^{(i)}_{v'_0, \dots, v'_{i{-}1}}\) is then added to the set \(\mathcal{D}^{(i)}_{l_0, l_1, \dots, l_{i{-}1}}\). Due to the recursive construction, we can show that the size of this set satisfies:
\[
\abs{\mathcal{D}^{(i)}_{l_0, l_1, \dots, l_{i{-}1}}} \in O(\log^i t).
\]

Once \(\mathcal{D}^{(i)}_{l_0, \dots, l_{i{-}1}}\) is formed, the algorithm evaluates:
\[
\Wei^{(i)}\left(\rg : [x'^{(0)}_1:x'^{(0)}_2] \cdots [x'^{(i{-}1)}_1:x'^{(i{-}1)}_2][x^{(i)}_1:x^{(i)}_2][0:1] \cdots [0:1]\right)
\]
for each data structure \(D^{(i)}_{v'_0, v'_1, \dots, v'_{i{-}1}}\) in the set, and aggregates the results to identify the next atomic interval along the \((i+1)\)-th coordinate.

To ensure correctness, we must verify that the vectors stored in the \(d{-}1\)-th level of each \(D^{(i)}_{v'_0, \dots, v'_{i{-}1}}\) can be properly scaled to the subregion defined by the first \(i\) atomic intervals:
\[
[x'^{(0)}_1:x'^{(0)}_2],\ [x'^{(1)}_1:x'^{(1)}_2],\ \dots,\ [x'^{(i{-}1)}_1:x'^{(i{-}1)}_2][0:1]\cdots[0:1].
\]
This guarantee follows inductively from the way the data structures are collected. For each coordinate index \(j < i\), the interval \(\RA(l_j)\) lies entirely within a single leaf \(l'_j\). As we construct \(\mathcal{D}^{(i)}_{l_0, \dots, l_{i{-}1}}\), we only include lazy structures along ancestor paths from these leaves. Since each lazy structure supports proper scaling over any subregion it contains, the cumulative reward vectors can be scaled correctly at every step. By applying this argument inductively over all \(j < i\), we conclude that all vectors stored in structures from \(\mathcal{D}^{(i)}_{l_0, \dots, l_{i{-}1}}\) can be scaled to the region above. This inductive scaling property allows the drawing algorithm to proceed sequentially across coordinates, correctly and efficiently identifying one atomic interval at a time.

 \subsection{$\Draw(t)$}
The drawing procedure in the \(d\)-dimensional setting extends naturally from the warm-up case. It proceeds in \(d\) sequential stages, where at each stage \(i \in \{0, 1, \dots, d{-}1\}\), the algorithm identifies an atomic interval along the \(i\)-th coordinate using the corresponding auxiliary tree \(\ATree^{(i)}\). Each tree is traversed via a probabilistic descent based on cumulative reward values, thereby selecting the appropriate atomic interval.

The process begins by locating the atomic horizontal interval \([x'^{(0)}_1:x'^{(0)}_2]\) using the query:
\[
\Wei^{(0)}\left([x^{(0)}_1:x^{(0)}_2][0:1]\cdots[0:1]\right).
\]
Conditioned on this horizontal slice, the algorithm continues to the second stage, identifying the atomic interval \([x'^{(1)}_1:x'^{(1)}_2]\) via:
\[
\Wei^{(1)}\left([x'^{(0)}_1:x'^{(0)}_2][x^{(1)}_1:x^{(1)}_2][0:1]\cdots[0:1]\right).
\]
More generally, for stage \(i\), the query has the form:
\[
\Wei^{(i)}\left([x'^{(0)}_1:x'^{(0)}_2] \cdots [x'^{(i{-}1)}_1:x'^{(i{-}1)}_2][x^{(i)}_1:x^{(i)}_2][0:1] \cdots [0:1]\right),
\]
which returns the cumulative reward necessary to sample the atomic interval \([x'^{(i)}_1:x'^{(i)}_2]\).

This routine continues recursively until the final atomic interval \([x'^{(d{-}1)}_1:x'^{(d{-}1)}_2]\) is selected, resulting in the sampled region \(\rg = [x'^{(0)}_1:x'^{(0)}_2]\cdots[x'^{(d{-}1)}_1:x'^{(d{-}1)}_2]\).
\section{Time Complexity (\Cref{section:ds-qhandling})}\label{section:complexity}

We begin by recalling the time complexity of standard multi-dimensional 
range trees (\(\RT\)) that report the set of points within a queried 
axis-aligned rectangular region. Our multi-dimensional data structure 
follows a similar hierarchical construction but maintains intervals 
instead of points. Provided that $\Tree$ remains balanced by 
\Cref{thm:balancedness}, we can adopt the classical bounds for $\Ins$ 
and $\Update$. Due to lazy insertions and the fact that not all 
intervals are explicitly stored in associated next-level trees, the 
time complexity for $\Wei$ and $\Draw$ deviates slightly; we address 
these distinctions below.

\subsection{Time complexity of $\Ins(.)$}
\begin{lemma}\label{lemma:inscomp1}
The time complexity of $\Ins(\rg_{i,t})$ for a region 
$\rg_{i,t} \in \Rg_t$ is
\[
O\!\left(\prod_{i \in [d]} \mathbb{E}[\EH(\Tree^{(i)})]\right).
\]
In particular, under an adaptive or oblivious $\sigma$-smoothed 
adversary this becomes $O\!\left((\sigma Tk)^{\frac{d}{2}}\right)$, 
and under a random-order oblivious adversary it becomes 
$O(\log^d(Tk))$.
\end{lemma}

\begin{proof}
At each level $i$ of the $d$-dimensional interval tree, $\Ins$ locates 
the leaves $l_i$ and $u_i$ containing the endpoints 
$\bp^{(i)}_1, \bp^{(i)}_2$ of the projection of $\rg_{i,t}$ onto 
coordinate $i$, and traverses upward along $p_{l_i}$ and $p_{u_i}$ 
to their first common ancestor $v_{l_i, u_i}$, inserting the next 
projection into the appropriate regular or lazy structures at each 
visited node. The locate step costs $O(\EH(\Tree^{(i)}))$ and the 
traverse step costs $O(\EH(\Tree^{(i)}))$ at level $i$, giving a 
combined cost of $O(\EH(\Tree^{(i)}))$ per level. Since $\Ins$ 
recurses across all $d$ levels, the total cost is 
$O\!\left(\prod_{i \in [d]} \EH(\Tree^{(i)})\right)$. Taking 
expectations and applying \Cref{thm:balancedness}, 
$\mathbb{E}[\EH(\Tree^{(i)})] = O(\sqrt{\sigma Tk})$ under adaptive 
and oblivious $\sigma$-smoothed adversaries and $O(\log(Tk))$ under 
the random-order oblivious adversary, yielding 
$O\!\left((\sigma Tk)^{d/2}\right)$ and $O(\log^d(Tk))$ respectively.
\end{proof}

\subsection{Time complexity of $\Update(.)$}
\begin{lemma}\label{lemma:updcomp}
The time complexity of $\Update(\rg_{i,t})$ for a region 
$\rg_{i,t} \in \Rg_t$ is
\[
O\!\left(\prod_{i \in [d]} \mathbb{E}[\EH(\Tree^{(i)})]\right).
\]
In particular, under an adaptive or oblivious $\sigma$-smoothed 
adversary this becomes $O\!\left((\sigma Tk)^{\frac{d}{2}}\right)$, 
and under a random-order oblivious adversary it becomes 
$O(\log^d(Tk))$.
\end{lemma}

\begin{proof}
The $\Update$ operation traverses the same paths identified during 
$\Ins(\rg_{i,t})$ to update the cumulative reward vectors 
$\vec{\CR}(\cdot)$ stored at the last layer of $\Tree$. The affected 
structures are identified exactly as in $\Ins$, following the same 
locate-and-traverse template at each of the $d$ levels. For each 
visited node, the algorithm computes the regular or lazy contribution 
vector $\vec{\SC}_{\rg'}$ or $\vec{\SC}_{\rg''}$ via the vector 
representation of~\Cref{eq:crew} using $\fu_t(\rg_{i,t})$ as the 
coefficient vector, which costs $O(1)$ per node by the parametric 
representation of \Cref{sec:aux_param}. Since the update traverses 
paths of expected length $O(\mathbb{E}[\EH(\Tree^{(i)})])$ at each 
of the $d$ levels, the total expected complexity is 
$O\!\left(\prod_{i \in [d]} \mathbb{E}[\EH(\Tree^{(i)})]\right)$, 
with the same instantiations as \Cref{lemma:inscomp1} under each 
adversary type by \Cref{thm:balancedness}.
\end{proof}

\subsection{Time complexity of $\Wei(\rg)$}
As described earlier, the cumulative reward query $\Wei(\rg)$ is 
decomposed into a sequence of subqueries $\Wei^{(i)}(\rg)$, each 
operating on progressively more refined regions. In each $\Wei^{(i)}$ 
subquery, the first $i$ intervals along the first $i$ coordinates are 
atomic, fixed by earlier sampling stages, while the remaining $d-i$ 
intervals span the full range $[0:1]$. The time complexity depends on 
the number of data structures in $\mathcal{S}^{(i)}_{l_0,\dots,l_{i-1}}$, 
each of which must be queried and aggregated to compute the final 
cumulative reward.

\begin{lemma}\label{lemma:weicomp}
The time complexity of a $\Wei^{(i)}(\rg)$ query of the form
\[
\rg:[x'^{(0)}_1:x'^{(0)}_2][x'^{(1)}_1:x'^{(1)}_2]\cdots
[x'^{(i-1)}_1:x'^{(i-1)}_2][x^{(i)}_1:x^{(i)}_2][0:1]\cdots[0:1]
\]
is
\[
O\!\left(\prod_{j=0}^{i}\mathbb{E}[\EH(\Tree^{(j)})]\right).
\]
In particular, under an adaptive or oblivious $\sigma$-smoothed 
adversary this becomes $O\!\left((\sigma Tk)^{\frac{i+1}{2}}\right)$, 
and under a random-order oblivious adversary it becomes 
$O(\log^{i+1}(Tk))$.
\end{lemma}

\begin{proof}
To process a $\Wei^{(i)}$ query, we first construct the set 
$\mathcal{S}^{(i)}_{l_0,\dots,l_{i-1}}$, consisting of all 
$(d{-}i)$-dimensional data structures influencing the subspace defined 
by the atomic intervals $[x'^{(0)}_1:x'^{(0)}_2], \dots, 
[x'^{(i-1)}_1:x'^{(i-1)}_2]$, where each $l_j$ is a leaf in the 
level-$j$ tree with $\RA(l_j) = [x'^{(j)}_1:x'^{(j)}_2]$.

This construction proceeds recursively: we locate $l_0$ in $\Tree^{(0)}$ 
and for each of its $O(\EH(\Tree^{(0)}))$ ancestors along $p_{l_0}$, 
collect the relevant lazy structures. Within each collected structure 
we locate $l_1$ and traverse its $O(\EH(\Tree^{(1)}))$ ancestors, 
collecting the relevant structures at the next level, and so on 
recursively through all $i$ levels. Since at each level $j$ we 
traverse $O(\EH(\Tree^{(j)}))$ ancestors per structure from the 
previous level, the total number of structures collected in 
$\mathcal{S}^{(i)}_{l_0,\dots,l_{i-1}}$ is 
$O\!\left(\prod_{j=0}^{i-1}\EH(\Tree^{(j)})\right)$.

After constructing this set, we query each of its members to compute 
the contribution over the non-atomic interval $[x^{(i)}_1:x^{(i)}_2]$. 
For each structure, this requires locating the leaves $l_i$ and $u_i$ 
containing $x^{(i)}_1$ and $x^{(i)}_2$ and traversing to their first 
common ancestor $v_{l_i,u_i}$, costing $O(\EH(\Tree^{(i)}))$ per 
structure. The total time complexity is therefore
\[
O\!\left(\prod_{j=0}^{i-1}\EH(\Tree^{(j)}) \cdot \EH(\Tree^{(i)})\right)
= O\!\left(\prod_{j=0}^{i}\EH(\Tree^{(j)})\right).
\]
Taking expectations and applying \Cref{thm:balancedness}, 
$\mathbb{E}[\EH(\Tree^{(j)})] = O(\sqrt{\sigma Tk})$ under adaptive 
and oblivious $\sigma$-smoothed adversaries and $O(\log(Tk))$ under 
the random-order oblivious adversary, giving 
$O\!\left((\sigma Tk)^{(i+1)/2}\right)$ and $O(\log^{i+1}(Tk))$ 
respectively.
\end{proof}

\subsection{Time Complexity of Hit-and-Run~\citep{lovasz2007geometry}}
By~\Cref{them:vempa}, this step takes $O\left(d^3 \log\left( \frac{Y}{\varepsilon} \right)\right).$

\subsubsection{Time Complexity of $\Draw(t)$}

\begin{theorem}\label{thm:drawcomp}
For a $d$-dimensional domain, the $\Draw(t)$ operation samples an 
action at round $t$ in time
\[
O\!\left(d 
\prod_{i=0}^d \mathbb{E}[\EH(\Tree^{(i)})] + d^3\right).
\]
In particular, under an adaptive or oblivious $\sigma$-smoothed 
adversary this becomes $O\!\left(d\,(\sigma Tk)^{\frac{d+1}{2}} + 
d^3\right)$, and under a random-order oblivious adversary it becomes 
$O(d\log^{d+1}(Tk) + d^3)$.
\end{theorem}

\begin{proof}
The $\Draw(t)$ procedure consists of two components: sampling an 
atomic region proportionally, and sampling $x_t$ within that region.

\paragraph{Atomic region sampling.}
The procedure operates in $d$ sequential stages, one per coordinate. 
At stage $i \in \{0, \dots, d-1\}$, the algorithm traverses the 
auxiliary interval tree $\ATree^{(i)}$ from root to leaf. This 
traversal visits $O(\EH(\Tree^{(i)}))$ nodes, and at each node issues 
a $\Wei^{(i)}$ query to determine the left/right probability of the 
next move. By \Cref{lemma:weicomp}, each $\Wei^{(i)}$ query costs 
$O\!\left(\prod_{j=0}^{i}\EH(\Tree^{(j)})\right)$. The total cost at 
stage $i$ is therefore
\[
O\!\left(\EH(\Tree^{(i)}) \cdot \prod_{j=0}^{i}\EH(\Tree^{(j)})\right)
= O\!\left(\EH(\Tree^{(i)})^2 \cdot \prod_{j=0}^{i-1}
\EH(\Tree^{(j)})\right).
\]
Summing over all $d$ stages and noting that the cost is dominated by 
the last stage $i = d-1$, the total cost for atomic region sampling is
\[
O\!\left(d \cdot \EH(\Tree^{(d-1)}) \cdot 
\prod_{i \in [d]} \EH(\Tree^{(i)})\right).
\]

\paragraph{Sampling within the atomic region.}
Given the sampled atomic region $\rg$, we sample $x_t \in \rg$ 
proportionally to $\MF(x, \Fu_t(x))$ using the Hit-and-Run algorithm 
(\Cref{alg:hitandrun}). Since $\MF(x, \Fu_t(x))$ is log-concave over 
the convex region $\rg$, \Cref{them:vempa} guarantees mixing within 
$O(d^3)$ steps.

\paragraph{Total cost.}
Combining both components and taking expectations via 
\Cref{thm:balancedness}: under adaptive and oblivious $\sigma$-smoothed 
adversaries, $\mathbb{E}[\EH(\Tree^{(i)})] = O(\sqrt{\sigma Tk})$ for 
all $i$, giving
\[
O\!\left(d\,(\sigma Tk)^{\frac{d+1}{2}} + d^3\right);
\]
under the random-order oblivious adversary, 
$\mathbb{E}[\EH(\Tree^{(i)})] = O(\log(Tk))$ for all $i$, giving
\[
O\!\left(d\log^{d+1}(Tk) + d^3\right).
\]
\end{proof}

\subsection{Memory Complexity}\label{sec:memory}
\begin{lemma}\label{lemma:mem}
Assuming the problem space is \( [0,1)^d \), the data structure $\Tree$ requires \( O(t^d) \) memory.
\end{lemma}
\begin{proof}
We start by noting that $\Tree^{(0)}$ contains $O(t)$ nodes. According to the structure of $\Tree$, every node $v \in \Tree^{(0)}$ is associated with a layered tree. In this structure, each layer corresponds to one of the other $d-1$ dimensions. Within each layer, there is a complete partitioning consisting of $O(t)$ endpoints for the respective coordinate. As a result, the total memory requirement for $\Tree$ is $O(t^d)$.
\end{proof}


\section{Efficient 
Adversarial Online Learning: Overview (\Cref{sec:learning})}\label{sec:learningapp}

We consider the adversarial online learning setting where at each 
round $t$, a $\sigma$-smoothed adversary selects hyperplanes $\Hp_t$ 
and assigns a piecewise-structured reward function $\RRF_t$ over the 
induced partition $\Rg_t$ of $[0,1]^d$, with $\MF$ fixed and known. 
The learner selects action $x_t$ by sampling proportionally to 
$\zeta(H_t(x)) = \exp(\eta \MF(x, \Fu_t(x)))$, where 
$\Fu_t(x) = \sum_{t' \leq t} \fu_{t'}(x)$ is the cumulative parameter 
function maintained by $\Tree$, instantiating our proportional sampling 
framework of \Cref{sec:ds} with exponential weighting. Beyond efficient 
sampling, the central objective is regret minimization: given a time 
horizon $T$ and actions $x_1, \dots, x_T \in [0,1)^d$, the regret of 
policy $\pi$ is
\begin{align}
\operatorname{Regret}(\pi, T) \doteq 
\max_{x \in [0,1]^d} \MF(x, \Fu_T(x))
- \Expect\!\left[\sum_{t=1}^T \MF(x_t, \fu_t(x_t))\right],
\end{align}

We consider two feedback models: under 
\emph{full-information} feedback the learner observes the full 
specification of $\RRF_t$ over all regions at each round, while under 
\emph{bandit} feedback only $\RRF_t(\rg_{i,t})$ for the region 
containing $x_t$ is revealed, implicitly disclosing both the structure 
and real-valued reward of that region. In the following subsections, 
we present the detailed data structure, query handling, computational 
challenges of proportional sampling under each feedback model, and 
the corresponding regret and time complexity analyses, together 
constituting the proofs of \Cref{thm:full-info} and \Cref{thm:bandit}.

\section{ Full-Information Feedback: Data Structure, Regret, and Time Complexity (\cref{sec:learning})}\label{app:full-info}

Under full-information feedback, the mapping function $\MF$ is fixed 
and known throughout, and at each round the learner receives the full 
specification of $\fu_t$ over all regions of $\Rg_t$. The learner 
then selects action $x_t$ by sampling proportionally to 
$\zeta(H_t(x)) = \exp(\eta \MF(x, \Fu_t(x)))$ over $[0,1]^d$, where 
$\Fu_t(x) = \sum_{t' \leq t} \fu_{t'}(x)$ is the cumulative parameter 
function maintained by $\Tree$. The online learning algorithm employed 
in this setup is $\Exp$, a no-regret algorithm that selects $x_t$ at 
each round with probability proportional to the exponential of its 
cumulative reward, whose pseudocode is presented in~\cref{alg:exp3}.

\begin{algorithm}[t]
    \begin{algorithmic}[1]
    \caption{$\Exp$}
    \label{alg:exp3}
    \Input{$\eta$}
         \State set $\Fu_1(x)=0$ for all $x \in [0,1]^d$
         \For {$t=1,2, \ldots, T$}
         \State Define $q_t(x)=\frac{\exp(\eta H_{t-1}(x))}{\int_{[0,1]^d} 
         \exp(\eta H_{t-1}(x))\,dx}$ 
         \State Pick $x_t \sim q_t$
         \State Observe $\fu_t$ and receive payoff $\fu_t(x_t)$
         \State Set $\Fu_{t} \gets \Fu_{t-1} + \fu_t$
         \State Set $H_{t} \gets \MF(x, \Fu_{t})$ for all 
         $x \in [0,1]^d$
         \EndFor
    \end{algorithmic}
\end{algorithm}

Handling proportional sampling for $\Exp$ under the adversarial model 
of \Cref{sec:model} introduces two main challenges: the non-linearity 
of $\exp$ breaks the additive structure of cumulative reward vectors 
that made lazy insertion and deferred scaling efficient, and no 
closed-form aggregate reward formula exists beyond degree one, making 
exact integration and regret bounding intractable for higher-degree 
rewards. These challenges require us to restrict the adversary from 
the full generality of \Cref{sec:model}, as detailed below.

\subsection{Challenges of Proportional Sampling under  with the Adversarial Smoothness of Model~\Cref{sec:model}}\label{app:challenges}

\paragraph{Non-additivity under $\exp$ compared to polynomials.}
Similar to the proportional sampling case with $\zeta(x) = x$, 
efficient sampling under $\zeta(x) = \exp(\eta x)$ requires lazy 
insertion data structures and scaling techniques to avoid an excessive 
number of insertions at each round. However, lazy insertions interact 
poorly with the exponential form of $\zeta$, introducing a fundamental 
obstacle that prevents direct application of the same approach. When 
$\zeta(x) = x$, the feedback parameter $\fu_t$ serves as the 
coefficient of monomials in $\MF$ and remains additive across rounds, 
allowing cumulative rewards $\CR$ to be efficiently aggregated over 
disjoint regions and making scaling operations both sound and efficient. 
Under $\zeta(x) = \exp(\eta x)$, however, $\FP$ appears in the 
exponent, so the cumulative reward over a union of regions is no longer 
the sum of individual rewards. This breaks the additive property 
essential to our scaling technique: any attempt to merge disjoint 
structures while deferring insertions requires identifying all atomic 
regions formed by their union, a task computationally equivalent to 
performing all insertions upfront.

\begin{example}\label{examp:full}
In the $2$-dimensional case, consider nodes \(v_0, v_0'\) along a 
root-to-leaf path in \(\Tree^{(0)}\) with \(v_0'\) an ancestor of 
\(v_0\). During \(\Draw(t)\), aggregating the reward over 
\([\RA(v_0)][0{:}1]\) requires combining \(\Tree^{(1)}_{v_0}\), storing 
vertical intervals \([0{:}x^{(1)}_2]\) and \([x^{(1)}_2{:}1]\), with 
\(\LZR^{(1)}_{v_0'}\), storing \([0{:}x^{(1)}_1]\), 
\([x^{(1)}_1{:}x^{(1)}_3]\) and \([x^{(1)}_3{:}1]\), where 
\(x^{(1)}_1 < x^{(1)}_2 < x^{(1)}_3\). Their overlay induces new 
atomic regions such as \([\RA(v_0)][x^{(1)}_1{:}x^{(1)}_2]\) and 
\([\RA(v_0)][x^{(1)}_2{:}x^{(1)}_3]\) that appear in neither structure 
alone. Under \(\zeta(x) = \exp(\eta x)\), given stored values
\begin{align*}
\int_A \exp(\eta H_t(x))\,dx \quad \text{and} \quad 
\int_B \exp(\eta H_t(x))\,dx
\end{align*}
for misaligned intervals $A$ and $B$, one cannot recover
\begin{align*}
\int_{A \cap B} \exp(\eta H_t(x))\,dx
\end{align*}
without explicit re-integration, since exponentials do not factor 
across sums in the exponent. This forces a full re-partitioning of the 
overlay, equivalent in cost to inserting all intervals into a unified 
tree from the start, growing as $O(t)$ per round as the partition 
evolves. Under $\zeta(x) = x$, however, linearity of $\CR$ in $\FP$ 
gives
\begin{align*}
\int_{A \cap B} H_t(x)\,dx = \int_A H_t(x)\,dx + \int_B H_t(x)\,dx 
- \int_{A \cup B} H_t(x)\,dx,
\end{align*}
recoverable directly from stored reward vectors by simple arithmetic, 
with no re-integration or re-partitioning required.
\end{example}

As illustrated in Example~\ref{examp:full}, this obstacle is inherent 
to any lazy insertion scheme under $\zeta(x) = \exp(\eta x)$: combining 
information from disjoint structures during sampling necessitates 
detecting the resulting atomic regions and re-integrating over them, 
effectively undoing the savings achieved by deferring insertions. To 
circumvent this, we restrict the adversary to selecting hyperplanes 
along all but one coordinate from fixed sets of at most $M$ hyperplanes, 
ensuring that the overlay structure between lazy and regular trees is 
known in advance and remains static throughout.

\paragraph{Non-existence of closed-form aggregate reward $\CR(\rg)$ 
for $\MF$ of degree higher than one.}
For polynomial $\MF$ of degree $n \geq 2$, no closed-form expression 
for $\CR(\rg) = \int_{\rg} \exp(\eta \MF(x, \Fu_t(x)))\,dx$ exists 
in general. For instance, when $\MF(x, \Fu_t(x)) = \Fu_t^{(0)} 
(x^{(0)})^2 + \Fu_t^{(1)} x^{(1)}$, the integral $\int \exp(\eta 
\Fu_t^{(0)} (x^{(0)})^2)\,dx$ is related to the Gaussian integral and 
admits no elementary antiderivative. Consequently, $\CR(\rg)$ cannot 
be computed analytically or stored exactly in $\Tree$ for $n \geq 2$, 
making both efficient data structure maintenance and regret analysis 
infeasible beyond degree one.

\paragraph{Computational inefficiency of scaling under piecewise-linear rewards $h_t$.}
Unlike high-degree polynomials, a closed-form expression for the 
aggregate reward of a region does exist for piecewise-linear $\MF$: 
for an axis-parallel rectangular region
\[
\rg = [x^{(0)}_1:x^{(0)}_2][x^{(1)}_1:x^{(1)}_2] \cdots 
[x^{(d-1)}_1:x^{(d-1)}_2],
\]
with $H_t(x)$ linear over $\rg$ and cumulative feedback vector 
$\Fu_t = \langle \Fu_t^{(0)}, \Fu_t^{(1)}, \ldots, \Fu_t^{(d-1)} 
\rangle$, the exponentially weighted cumulative reward is given by
\begin{align}
    \CR(\rg) = \int_{\rg} \exp(\eta \RRF_t(x)) \, dx = 
    \prod_{i=0}^{d-1}
    \begin{cases}
        \displaystyle\frac{\exp(\eta \Fu_t^{(i)} x^{(i)}_2) - 
        \exp(\eta \Fu_t^{(i)} x^{(i)}_1)}{\eta \Fu_t^{(i)}}, 
        & \text{if } \Fu_t^{(i)} \ne 0, \\[1.2ex]
        x^{(i)}_2 - x^{(i)}_1, & \text{if } \Fu_t^{(i)} = 0.
    \end{cases}
    \label{eq:expint}
\end{align}
However, the presence of $\Fu_t^{(i)}$ in the denominator couples all 
$\FP$ components across coordinates: to scale a stored cumulative 
reward from a parent region down to a subregion, each internal node 
would require access to $\Fu_t^{(i)}$ for every atomic subregion in 
its subtree, since $\Fu_t^{(i)}$ varies across atomic regions and the 
exponential terms cannot be factored or cancelled without knowing it 
explicitly. As the number of atomic subregions in a subtree grows as 
$O(t)$, this incurs an $O(t)$ cost per node per round, making 
efficient maintenance infeasible.

\subsection{Full-Information Setup and Adversarial Model}\label{app:full-setup}

In response to the computational challenges detailed in 
\Cref{app:challenges}, we restrict the adversary and reward structure 
for the full-information setting as follows. Due to the non-additivity 
of $\exp(\eta H_t(x))$ across lazy and regular structures, we restrict 
to piecewise-constant rewards, and due to the $O(t)$ cost of resolving 
disjoint endpoint sets across structures in $\mathcal{S}^{(1)}_{v_0}$, 
we restrict the adversary to smooth hyperplane selection along one 
designated direction only.

\paragraph{Full-Information Setup.} Concretely, at each round $t$, $k$ hyperplanes $\Hp_t = \{\hp_{1,t}, 
\dots, \hp_{k,t}\}$ are chosen with directions from $\Ax$ as in 
\Cref{sec:model}, with the following modification: for each 
$\hp_{j,t} \in \Hp_t$, if $\hp_{j,t}$ is parallel to $\ax_1 \in \Ax$, 
it is drawn smoothly from $\Dist^{(1)}_t$ over $[0,1]$, with the total 
number of such hyperplanes growing to $TK$ over time; otherwise it is 
chosen from a fixed set $\mathcal{M}^{(i)} = \{\hp^{(i)}_1, \dots, 
\hp^{(i)}_M\}$ of at most $M$ hyperplanes defined a priori for each 
direction $\ax_i \in \Ax \setminus \{\ax_1\}$. Without loss of 
generality, $\ax_1$ corresponds to the first coordinate of $\AS$; 
otherwise the coordinate axes are permuted accordingly. The hyperplanes 
$\Hp_t$ then partition $[0,1]^d$ into regions $\Rg_t$, over which the 
adversary assigns a piecewise-constant reward $\RRF_t$.

Under this restriction, only the number of hyperplanes along $\ax_1$ 
grows linearly with time, while the number of intervals along all 
remaining directions stays bounded by $M$ throughout. This allows us 
to retain the same structure of lazy and regular insertions with 
significantly lower overhead: since the boundaries along the remaining 
$d-1$ coordinates are known a priori, handling updates and scaling 
operations does not incur a worst-case cost of $O(t^d)$ per round but 
instead reduces to $O(M^d)$, where $M$ is the size of the fixed 
hyperplane sets. In particular, the overlay structure between lazy and 
regular trees is fixed in advance, eliminating the costly atomic region 
detection that arises in the general case.

\subsection{Data Structure and Query Handling for $\Exp$ under 
Full-Information Feedback}\label{app:full-ds}

We now describe how queries are handled to maintain efficiency when sampling \(x_t\) under piecewise constant reward functions. We present the construction in the two-dimensional case; the 
$d$-dimensional generalization follows the same layered approach as~\cref{section:generalpar} with 
adjusted time complexity and regret bounds.

The first-level tree $\Tree^{(0)}$ organizes the unique dynamic 
coordinate, along which the adversary introduces new hyperplanes over 
time. Each node $v_0 \in \Tree^{(0)}$ maintains three second-level 
interval trees over the second coordinate: a regular insertion tree 
$\Tree^{(1)}_{v_0}$ and two lazy insertion trees $\LZL^{(1)}_{v_0}$ 
and $\LZR^{(1)}_{v_0}$, mirroring the structure of~\Cref{sec:ds}. 
This setting, however, differs from the general construction in three 
key respects:

\begin{itemize}
    \item \textbf{Fixed-grid trees at all deeper levels.} Since the 
    adversary's hyperplanes along all remaining coordinates are drawn 
    from fixed sets $\mathcal{M}^{(1)}, \dots, \mathcal{M}^{(d-1)}$ of 
    at most $M$ hyperplanes each, all second-level trees are initialized 
    once over this fixed grid and remain static throughout, each 
    containing $O(M)$ nodes. In the $d$-dimensional case, these become 
    $(d{-}1)$-dimensional regular interval trees, each organized over 
    the fixed grid $\mathcal{M}^{(1)} \times \cdots \times 
    \mathcal{M}^{(d-1)}$ and containing $M^{d-1}$ nodes per first-level 
    node.

    \item \textbf{Lazy structures at the first level only.} Since no 
    new hyperplanes are introduced along any coordinate beyond the 
    first, no deferred insertions arise at deeper levels. Lazy 
    structures $\LZL^{(1)}_{v_0}$ and $\LZR^{(1)}_{v_0}$ are therefore 
    maintained exclusively at nodes $v_0 \in \Tree^{(0)}$, with no 
    lazy components at any level beyond $\Tree^{(0)}$. In the 
    $d$-dimensional case, this remains unchanged: lazy structures are 
    attached solely to nodes of $\Tree^{(0)}$, and all trees at levels 
    $i \geq 1$ are purely regular.

    \item \textbf{Scalar Cumulative Rewards rather than Vectorized.} 
    Since $H_t$ is piecewise-constant over each atomic region 
    $\rg = \prod_{i=0}^{d-1}[x_1^{(i)}{:}x_2^{(i)}]$ with constant 
    value $c_{\rg} \in \mathbb{R}_{\geq 0}$, the exponentially weighted 
    cumulative reward of $\rg$ takes the closed form
    \begin{align}
        \CR(\rg) 
        = \int_{\rg} \exp\!\left(\eta\, H_t(x)\right) dx 
        = \exp(\eta\, c_{\rg})
          \cdot \prod_{i=0}^{d-1}(x_2^{(i)} - x_1^{(i)}),
        \label{eq:expintconst}
    \end{align}
    a single exponential factor multiplied by a product of interval 
    lengths. Since scaling is required only along the first coordinate, 
    this separable form means that all nodes in every second-level tree 
    need only maintain a scalar value rather than a reward vector, while 
    fully preserving the scalability needed for efficient updates.
\end{itemize}

We now specify the scalar quantities maintained at each node of the 
regular and lazy second-level trees.

\paragraph{Exponentially Weighted Cumulative Reward for Regular 
Insertion Data Structures.}
Each node $v_1$ in the regular insertion tree $\Tree^{(1)}_{v_0}$ 
maintains a scalar $\CR_{v_0,v_1}$, aggregating the exponentially 
weighted cumulative reward over all regions 
$\rg = [\bp^{(0)}_1:\bp^{(0)}_2][\bp^{(1)}_1:\bp^{(1)}_2]$ whose 
horizontal projection contains at least one endpoint in $\RA(v_0)$ and 
whose vertical projection is contained in $\RA(v_1)$, that is,
\begin{align}
    \CR_{v_0,v_1} 
    = \sum_{\rg} 
      \CR\!\left([\RA(v_0) \cap [\bp^{(0)}_1:\bp^{(0)}_2]]\,
      [\RA(v_1)\cap[\bp^{(1)}_1:\bp^{(1)}_2]]\right),
    \label{eq:cr-regular}
\end{align}
where $\CR(\cdot)$ is as defined in~\eqref{eq:expintconst}.

\paragraph{Cumulative Feedback Parameters for Lazy Insertion Data 
Structures.}
Each node $w_1$ in the lazy insertion trees $\LZL^{(1)}_{v_0}$ or 
$\LZR^{(1)}_{v_0}$ maintains a scalar $\CR_{v_0,w_1}$, aggregating 
the exponentially weighted cumulative reward over all regions 
$\rg = [\bp^{(0)}_1:\bp^{(0)}_2][\bp^{(1)}_1:\bp^{(1)}_2]$ whose 
horizontal projection fully contains the range of the corresponding 
child of $v_0$ (i.e., $\RA(\LC(v_0))$ for $\LZL^{(1)}_{v_0}$ and 
$\RA(\RC(v_0))$ for $\LZR^{(1)}_{v_0}$) and whose vertical projection 
is contained in $\RA(w_1)$, that is,
\begin{align}
    \CR_{v_0,w_1} 
    = \sum_{\rg} 
      \CR\!\left([\RA(v_0)]\,
      [[\bp^{(1)}_1:\bp^{(1)}_2] \cap \RA(w_1)]\right),
    \label{eq:cr-lazy}
\end{align}
where $\CR(\cdot)$ is as defined in~\eqref{eq:expintconst}, and the sum 
ranges over all lazily inserted regions as defined 
in~\Cref{section:ds-ps}. Since all lazily inserted regions contribute 
uniformly over $\RA(v_0)$, this value scales consistently to any 
subinterval of $[\RA(v_0)]$, including $\RA(\LC(v_0))$ and 
$\RA(\RC(v_0))$, without requiring recomputation.

\paragraph{Handling $\Ins$, $\Update$, and $\Draw$.}
The $\Ins$ and $\Update$ routines follow the same locate-and-traverse 
approach as described in~\Cref{section:ds-ps} for the main structure. 
Specifically, at the first level, $\Ins(\rg)$ locates the boundary 
leaves in $\Tree^{(0)}$ and traverses upward to identify the relevant 
regular and lazy second-level structures, exactly as before. Within 
each second-level tree, updates to the scalar values $\CR_{v_0,v_1}$ 
and $\CR_{v_0,w_1}$ follow the standard lazy propagation technique 
common in the interval tree literature~\citep{cohen2017online}.

\paragraph{Different Proportional Sampling $\Draw$.}
The $\Draw(t)$ procedure follows the same two-step sampling scheme as 
in~\Cref{section:ds-qhandling}: first a horizontal atomic interval is 
sampled from $\Tree^{(0)}$, then a vertical atomic interval is sampled 
conditioned on the horizontal selection. This procedure, however, 
differs from the general case in three key respects:

\begin{itemize}
    \item \textbf{Modified $\Wei$ subroutine.} In the general 
    structure, the additivity of $\zeta(x) = x$ allows $\Wei$ to 
    aggregate cumulative reward across lazy and regular structures via 
    clean vector arithmetic. Under $\zeta(x) = \exp(\eta x)$, this 
    additivity is lost, necessitating a fundamentally different 
    approach. We exploit the fact that all second-level trees are 
    organized over a fixed set of $M$ endpoints along each remaining 
    coordinate: rather than aggregating across structures symbolically, 
    $\Wei$ iterates explicitly over all $O(M)$ leaves of the 
    second-level trees, applies the exponential scaling to each leaf 
    individually, and sums the resulting scalar values bottom-up to 
    recover the correct aggregated cumulative reward. This leaf-wise 
    scaling and summation procedure is what replaces the vector-based 
    aggregation of the general case and is detailed in the scaling 
    procedure below.

    \item \textbf{Exponential scaling.} Rather than scaling via linear 
    vector arithmetic as in the general structure, lazy contributions 
    stored in ancestor nodes along $p_{v_0}$ are propagated 
    multiplicatively as $\exp(\eta \cdot \CR_{v_0,w_1})$, correctly 
    incorporating deferred additive feedback parameter shifts into the 
    exponent of the cumulative reward. This scaling step is performed 
    as part of the $\Wei$ query and is detailed in the scaling procedure 
    below.

    \item \textbf{Uniform sampling within atomic regions.} Since $H_t$ 
    is piecewise-constant over each atomic region, no hit-and-run step 
    is required within the selected region. Instead, a point $x_t$ is 
    drawn by sampling uniformly over each coordinate of the selected 
    atomic region $[\RA(v_0)][\RA(v_1)]$, that is,
    \[
        x_t^{(i)} \sim \mathrm{Uniform}(\RA(v_i)), \qquad i = 0, 1.
    \]
    which can be handled in $O(d)$ without applying~\cref{them:vempa}.
\end{itemize}

\paragraph{Scaling and the $\Wei$ Subroutine.}
In the general structure, the additivity of $\zeta(x) = x$ allows 
$\Wei$ to aggregate cumulative reward across lazy and regular 
structures via simple vector arithmetic, without requiring explicit 
enumeration of the underlying atomic regions. Under $\zeta(x) = 
\exp(\eta x)$, this additivity is lost: deferred lazy contributions 
appear in the exponent and can no longer be recovered by linear 
combination of stored values. This necessitates a fundamentally 
different approach to both scaling and aggregation.

The key enabler is the fixed-grid structure of the second-level trees. 
Since all second-level trees, whether regular or lazy, are organized 
over the same fixed set of $M$ vertical endpoints, every leaf 
$l'_1 \in \LZL^{(1)}_{v'_0}$ has a uniquely corresponding leaf 
$l_1 \in \Tree^{(1)}_{v_0}$ with identical vertical range. This 
alignment makes it possible to propagate deferred contributions 
leaf-by-leaf, without any re-partitioning or atomic region detection. 
In the $d$-dimensional case, the same alignment holds across all 
$(d{-}1)$-dimensional second-level trees, since every level $i \geq 1$ 
is organized over the same fixed grid 
$\mathcal{M}^{(1)} \times \cdots \times \mathcal{M}^{(d-1)}$, and the 
leaf-wise correspondence extends accordingly.

Concretely, suppose the sampling algorithm invokes 
$\Wei^{(0)}([\RA(v_0)][0{:}1])$. The algorithm locates $v_0$ in 
$\Tree^{(0)}$ and applies scaling from all relevant lazy structures 
maintained at its ancestors along the path $p_{v_0}$. Consider one 
such ancestor $v'_0$ for which $v_0$ lies in its left subtree, so 
that the relevant lazy structure is $\LZL^{(1)}_{v'_0}$. For each 
leaf $l'_1 \in \LZL^{(1)}_{v'_0}$, the algorithm identifies the 
corresponding leaf $l_1 \in \Tree^{(1)}_{v_0}$ and performs the 
scaling update
\[
    \CR_{v_0,l_1} \gets \exp(\eta \cdot \CR_{v'_0,l'_1}) \cdot 
    \CR_{v_0,l_1}.
\]
To see why this is correct, suppose $\CR_{v_0,l_1}$ currently 
aggregates two disjoint subregions with feedback parameters $a_1$ 
and $b_1$:
\[
    \CR_{v_0,l_1} = \exp(\eta a_1)(x_2 - x_1)(y_2 - y_1) 
    + \exp(\eta b_1)(x_4 - x_3)(y_4 - y_3),
\]
and let $\CR_{v'_0,l'_1} = c_1$ be the deferred additive shift stored 
in the lazy structure. After scaling:
\begin{align*}
    \CR_{v_0,l_1} &\gets \exp(\eta c_1) \cdot \CR_{v_0,l_1} \\
    &= \exp(\eta(a_1+c_1))(x_2-x_1)(y_2-y_1) 
     + \exp(\eta(b_1+c_1))(x_4-x_3)(y_4-y_3),
\end{align*}
which correctly incorporates the deferred contribution $c_1$ as an 
additive shift in the exponent of every subregion within 
$[\RA(v_0)][\RA(l_1)]$, regardless of how many subregions are 
aggregated at $l_1$.

Once all leaves of $\LZL^{(1)}_{v'_0}$ have been processed, the 
updated leaf values are propagated bottom-up through $\Tree^{(1)}_{v_0}$, 
so that $\ROOT(\Tree^{(1)}_{v_0})$ reflects the correct aggregated 
exponentially weighted cumulative reward. A symmetric procedure 
applies to $\LZR^{(1)}_{v'_0}$ when $v_0$ lies in the right subtree 
of $v'_0$. This is repeated for every ancestor of $v_0$ along 
$p_{v_0}$, with the order of application immaterial since each 
deferred shift enters the exponent additively, giving 
$\exp(\eta c_1)\cdot\exp(\eta c_2) = \exp(\eta(c_1+c_2))$.

Since the second-level trees have $O(M)$ leaves, each scaling step 
costs $O(M)$, and summing bottom-up costs $O(M)$ as well. The total 
cost of scaling all ancestors along $p_{v_0}$ is therefore 
$O(M \cdot \mathbb{E}[\EH(\Tree^{(0)})])$. In the $d$-dimensional 
case, the leaf-wise scaling extends to $M^{d-1}$ leaves per 
second-level tree, giving a total scaling cost of 
$O(M^{d-1} \cdot \mathbb{E}[\EH(\Tree^{(0)})])$ per $\Wei$ query. 
Once scaling is complete, the root of $\Tree^{(1)}_{v_0}$ holds the 
correct value for $\Wei^{(0)}([\RA(v_0)][0{:}1])$, which is then 
used directly by $\Draw(t)$ for proportional sampling.

\subsection{Time Complexity}
\begin{theorem}\label{thm:expcomp}
Under the full-information feedback setting with the adversary 
$\sigma$-smoothed along $\ax_1$ and at most $M$ fixed hyperplanes 
along each remaining direction, the per-round time complexity of 
$\Exp$ is $O\!\left(dM^{d-1}\sqrt{\sigma Tk} + d\right)$ under both 
adaptive and oblivious $\sigma$-smoothed adversaries, and 
$O\!\left(dM^{d-1}\log(Tk) + d\right)$ under the random-order 
oblivious adversary.
\end{theorem}
\begin{proof}
Each round consists of $\Draw(t)$ followed by $\Update$.

\paragraph{$\Update$ cost.}
The algorithm traverses $\Tree^{(0)}$ in $O(\mathbb{E}[\EH(\Tree^{(0)})])$ 
steps to locate the projection boundaries along the first coordinate 
and identify the affected lazy and regular structures. At each visited 
node, it updates the associated $(d{-}1)$-dimensional second-level 
structure, which is organized over the static grid of size $M^{d-1}$, 
at cost $O(M^{d-1})$ per node. The total $\Update$ cost is 
$O(\mathbb{E}[\EH(\Tree^{(0)})]\cdot M^{d-1})$.

\paragraph{$\Draw$ cost.}
The $\Draw$ procedure performs $d$ sequential sampling stages. At each 
stage, the algorithm traverses $\ATree^{(0)}$ in 
$O(\mathbb{E}[\EH(\Tree^{(0)})])$ steps, and at every visited node 
scales the relevant lazy structures leaf-to-leaf over the fixed grid 
of $M^{d-1}$ leaves at $O(1)$ per leaf, costing $O(M^{d-1})$ per node. 
Each stage therefore costs $O(\mathbb{E}[\EH(\Tree^{(0)})]\cdot M^{d-1})$, 
and summing over all $d$ stages gives
\[
O\!\left(d\cdot\mathbb{E}[\EH(\Tree^{(0)})]\cdot M^{d-1}\right).
\]
Since rewards are piecewise-constant, $x_t$ is sampled uniformly 
within the selected atomic region in $O(d)$, and the total $\Draw$ 
cost is $O\!\left(d\cdot\mathbb{E}[\EH(\Tree^{(0)})]\cdot M^{d-1} + d\right)$, 
which dominates the $\Update$ cost.

\paragraph{Total per-round cost.}
Applying \Cref{thm:balancedness}, $\mathbb{E}[\EH(\Tree^{(0)})] = 
O(\sqrt{\sigma Tk})$ under adaptive and oblivious adversaries, and 
$\mathbb{E}[\EH(\Tree^{(0)})] = O(\log(Tk))$ under the random-order 
adversary, yielding the stated bounds.
\end{proof}

\subsection{Smoothness-Related Lemmas}\label{section:regretful}
In this section, we analyze the regret of the learning algorithm $\Alg$ for $d$-dimensional $\AS$ with piecewise constant $\fu_t$s. We also generalize the regret bounds from~\cite{cohen2017online}, which was achieved for the one-dimensional case with piecewise constant $\fu_t$s, to the $d$-dimensional case where $\fu_t$ can be piecewise linear or even more complex reward functions. We should highlight that the regret analysis highly depends on the degree of \emph{smoothness} the adversary is allowed to treat like. Thus, we first provide a direct generalization for the definition of smoothness to what we implicitly mentioned in Section~\ref{sec:model}.

\begin{observation}[\cite{cohen2017online}]\label{obs:ck1}
    Let $D_1,D_2, \cdots, D_m$ be independently drawn from any distributions (possibly different) whose density functions are bounded by $\sigma$. Then the probability that there exist $D_i<D_j$ such that $D_j-D_i<\epsilon$ is at most $m^2 \sigma \epsilon$.
\end{observation}
\begin{observation}[\cite{cohen2017online}]\label{obs:ck2}
    Let $0=D_0<D_1<\cdots<D_m=1$ be any points in $[0,1]$ such that $D_i-D_{i-1}>\epsilon$ for all $i$. Let $Y_1,Y_2 \cdots, Y_n$ be drawn uniformly at random from $[0,1)$. Then if $n \geq \frac{1}{\epsilon}\left(\ln \left(\epsilon^{-1}\right)+\right.$ $\ln \left(\delta^{-1}\right)$ ), the probability that there exists $i$ such that there is no $Y_j$ in the interval $\left(D_{i-1}, D_i\right)$ is at most $\delta$.
\end{observation}
\begin{observation}\label{obs:ckd}
For $d$-dimensional $\AS$, both Observations~\ref{obs:ck1} and~\ref{obs:ck2} hold for each dimension separately.
\end{observation}

\subsection{Regret Analysis}
In this section we carry out the regret analysis for the full-information 
setup defined in \Cref{app:full-setup}. While \Cref{thm:full-info} 
focuses on piecewise-constant rewards under the restricted one-direction 
smoothness assumption, we present a more general regret analysis for 
piecewise-linear rewards under smoothness over all directions, 
formalized in \Cref{thm:ddlin} for the $d$-dimensional case. The 
regret bound for piecewise-constant rewards (\Cref{thm:dpwc}) then 
follows as a special case, with \Cref{cor:fullfeed} recovering the 
bound for our restricted setup with smoothness over one direction. 
This ordering allows us to derive the piecewise-constant result by 
specializing the more general piecewise-linear analysis.

\begin{theorem}[$d$-dimensional piecewise-constant rewards]\label{thm:dpwc}
For a $d$-dimensional action space $\AS$ with piecewise-constant 
rewards $\fu_t$, the expected regret of $\Exp$ is bounded by
\[
\eta(e-2)T - \frac{d\ln(\epsilon^*)}{\eta},
\]
where $x^* \in [0,1)^d$ is the optimal action in hindsight not lying 
on any region boundary, $\rg^*$ is the largest region containing $x^*$ 
such that $\Fu_{T+1}$ remains constant and optimal within it, and 
$\epsilon^*$ is the length of the smallest interval among the 
projections of $\rg^*$ onto each of the $d$ coordinate axes.
\end{theorem}

With the following we derive the regret bound for piecewise-linear reward 
functions $\RRF_t$ in the $d$-dimensional case under axis-parallel 
hyperplanes.

\begin{theorem}[$d$-dimensional piecewise-linear rewards with 
axis-parallel $\Hp_t$]\label{thm:ddlin}
Suppose $\Fu^{(j)}_{T+1}(x^*) < \frac{2}{\eta\epsilon^*}$ for all 
$0 \leq j \leq d$. For a $d$-dimensional action space $\AS$ with 
axis-parallel decision boundaries $\Hp_t$ and piecewise-linear rewards 
$\fu_t$, the expected regret of $\Exp$ satisfies
\begin{align*}
\eta T - \frac{\sum_{i=1}^d \log\!\left(\epsilon^* - 
\frac{\eta \Fu^{(i-1)}_{T+1}(x^*)\epsilon^{*2}}{2}\right)}{\eta} 
\leq \eta T - \frac{d\log(\epsilon^*)}{\eta} + \frac{d}{\eta},
\end{align*}
where $x^* \in [0,1)^d$ is the optimal action in hindsight not 
coinciding with any discontinuity of $\fu_t$, $\rg^*$ is the largest 
region containing $x^*$ such that $\Fu_{T+1}(x^*)$ is constant and 
optimal within it, and $\epsilon^*$ is the length of the smallest 
interval among the projections of $\rg^*$ onto each coordinate axis.
\end{theorem}

\begin{proof}
Following the standard regret analysis for the $\Exp$ algorithm, we define  
\(
W_t = \int_0^1 \exp \left(\eta H_t(x)\right) dx.
\)
Let the expected reward of the $\Exp$ algorithm at time step $t$ be denoted by  
\(
Y_t = \mathbb{E}_{x \sim q_t}[\RRF_t(x)].
\)
Accordingly, we define the cumulative expected reward of the $\Exp$ algorithm as  
\(
Y^\Exp = \sum_{t=1}^T Y_t.
\)
We now express $W_{T+1}$ as a telescoping product:
\begin{align}
W_{T+1} = \frac{W_{T+1}}{W_T} \times \frac{W_T}{W_{T-1}} \times \cdots \times \frac{W_2}{W_1} \times W_1. \label{eq:LEXP3_1}
\end{align}
To analyze this expression, we bound each term $\frac{W_{t+1}}{W_t}$ separately:
\begin{align}
    \frac{W_{t+1}}{W_{t}} &= \frac{\int_0^1 \exp \left(\eta H_{T+1}(x)\right) dx}{\int_0^1 \exp \left(\eta H_{t}(x)\right) dx},\nonumber\\
    &= \frac{\int_0^1 \exp \left(\eta H_{t}(x)\right)e^{\eta\RRF_t(x)} dx}{\int_0^1 \exp \left(\eta H_{t}(x)\right) dx},\nonumber\\
    &= \int_0^1 q_t(x)e^{\eta\RRF_t(x)}dx.\nonumber
\end{align}
Applying the standard inequality $\exp (x) \leq 1+x+x^2$ for $x \leq 1$, we obtain:
\begin{align}
    \int_0^1 q_t(x)e^{\eta\RRF_t(x)}dx 
    &\leq \int_0^1 q_t(x)\left(1 + \eta\RRF_t(x) + \eta^2\RRF^2_t(x)\right)dx,\nonumber\\
    &= 1 + \eta\int_0^1 q_t(x)\RRF_t(x)dx + \eta^2\int_0^1 q_t(x)\RRF^2_t(x) dx.\nonumber
\end{align}
Next, applying the inequality $1+x \leq \exp(x)$ for all $x \in \mathbb{R}$, we obtain:
\begin{align}
    \frac{W_{t+1}}{W_{t}} &\leq \exp\left(\eta\int_0^1 q_t(x)\RRF_t(x)dx + \eta^2\int_0^1 q_t(x)\RRF^2_t(x) dx\right),\nonumber\\
    &\leq \exp\left(\eta Y_t + \eta^2\right). \label{ineq:LEXP3_2}
\end{align}
Now, substituting inequality~\eqref{ineq:LEXP3_2} into equation~\eqref{eq:LEXP3_1}, we obtain:
\begin{align}
    W_{T+1} &\leq W_1 \prod_{t=1}^T \frac{W_{t+1}}{W_t},\nonumber\\
    &\leq \exp\left(\eta\sum_{t=1}^T Y_t + \eta^2 T\right),\nonumber\\
    &\leq \exp\left(\eta Y^\Exp+\eta^2 T\right),\label{ineq:LEXP3_3}
\end{align}
where by definition we have $W_1 = 1$.

Our proof establishes bounds on $W_{T+1}$ from both sides. For the 
upper bound, we apply inequality~\eqref{ineq:LEXP3_3} directly, as it 
does not depend on the dimensionality of $\AS$. We therefore focus on 
deriving a lower bound for $W_{T+1}$. Note that axis-parallel decision 
boundaries induce a hyper-rectangular shape for $\rg^*$, which with 
piecewise-linear $\fu_t$ forces $x^*$ to lie at a corner of $\rg^*$. 
Denoting the intervals obtained by projecting $\rg^*$ onto the $i$-th 
coordinate by $[l^*_i, r^*_i]$ with $r^*_i - l^*_i \geq \epsilon^*$, 
and assuming without loss of generality that $x^* = \langle r^*_1, 
r^*_2, \dots, r^*_d \rangle$, we write
\begin{align}
    W_{T+1} &= \int_{[0,1]^d} \exp \left(\eta H_{T+1}(x)\right) dV,\nonumber\\
    &\geq \int_{\rg^*} \exp \left(\eta H_{T+1}(x)\right) dV,\nonumber\\
    &\geq \int_{\rg^*} \exp \left(\eta \Fu_{T+1}(x^*)\cdot \bm{x}\right) dV,\nonumber\\
    &\overset{(a)}\geq \int_{l^*_1}^{r^*_1}\int_{l^*_2}^{r^*_2}\cdots\int_{l^*_d}^{r^*_d}\exp\left( \eta Y^\Opt - \eta \sum_{i=1}^d \Fu^{(i-1)}_{T+1}(x^*)\left(r^*_i-x(i-1)\right)\right) dx(d-1)dx(d-2)\cdots dx(0),\label{eq:LEXP3_4}\\
    &\geq e^{\eta Y^\Opt}\prod_{i=1}^d \int_{l^*_i}^{r^*_i}\exp\left(\eta \Fu^{(i-1)}_{T+1}(x^*)\left(x(i-1)-r^*_i \right)\right) dx(i-1),\nonumber\\
    &\geq e^{\eta Y^\Opt} \prod_{i=1}^d \frac{1-\exp \left(-\eta \Fu^{(i-1)}_{T+1}(x^*)\left(r^*_i-l^*_i\right)\right)}{\eta \Fu^{(i-1)}_{T+1}(x^*)},\nonumber\\
    &\geq e^{\eta Y^\Opt} \prod_{i=1}^d \frac{1-\exp \left(-\eta \Fu^{(i-1)}_{T+1}(x^*)\epsilon^*\right)}{\eta \Fu^{(i-1)}_{T+1}(x^*)},\nonumber\\
    &\geq e^{\eta Y^\Opt} \prod_{i=1}^d \left(\epsilon^* - \frac{\eta \Fu^{(i-1)}_{T+1}(x^*)\epsilon^{*^2}}{2}\right),\label{ineq:LEXP3_6}
\end{align}  
The regret bound of \Cref{thm:ddlin} follows by combining the upper 
bound from~\eqref{ineq:LEXP3_3} with the lower bound on $W_{T+1}$ 
derived above. Substituting and taking logarithms on both sides, then 
applying $-\log(1-u) \leq u$ for $u \in [0,1)$ in step (a) and the 
assumption $\Fu^{(j)}_{T+1}(x^*) < \frac{2}{\eta\epsilon^*}$ in step 
(b), we obtain
\begin{align*}
Y^\Opt - Y^\Exp 
&\leq \eta T - \frac{\sum_{i=1}^d \log\!\left(\epsilon^* - 
\frac{\eta \Fu^{(i-1)}_{T+1}(x^*)\epsilon^{*2}}{2}\right)}{\eta} \\
&= \eta T - \frac{\log(\epsilon^*)}{\eta} - 
\frac{\log\!\left(1-\frac{\eta\epsilon^*\Fu^{(0)}_{T+1}(x^*)}{2}
\right)}{\eta} \\
&\overset{(a)}{\leq} \eta T - \frac{\log(\epsilon^*)}{\eta} + 
\frac{\epsilon^*\Fu^{(0)}_{T+1}(x^*)}{2} \\
&\overset{(b)}{\leq} \eta T - \frac{d\log(\epsilon^*)}{\eta} + 
\frac{d}{\eta},
\end{align*}
completing the proof of \Cref{thm:ddlin}.
\end{proof}

\begin{corollary}
For $\eta=\sqrt{\frac{d\log \left(k^2 T^3 \sigma\right)+1}{T}}$, the regret of Theorem~\ref{thm:ddlin} is bounded by  $O(\sqrt{dT\left(\log \left(k^2 T^3 \sigma\right)+1\right)})$.
\end{corollary}

\begin{corollary}\label{cor:fullfeed}
For $\eta = \sqrt{\frac{\log(k^2T^3\sigma) + (d-1)\log M + 1}{T}}$, 
the regret of \Cref{thm:ddlin} under the full-information setup of 
\Cref{app:full-setup}, where $TK$ hyperplanes are drawn smoothly along 
$\ax_1$ and at most $M$ fixed hyperplanes are used along each remaining 
direction, is bounded by
\[
O\!\left(\sqrt{T\left(\log\!\left(k^2T^3\sigma\right) + 
(d-1)\log M + 1\right)}\right).
\]
\end{corollary}

\section{Bandit Feedback: Data Structure, Regret, and Time Complexity (\Cref{sec:learning})}\label{app:bandit}

Under bandit feedback, the learner observes only $\fu_t(x_t)$ after 
selecting $x_t$, assumed to be piecewise-linear, rather than the full 
specification of $\fu_t$. The algorithm constructs an 
importance-weighted estimate $\hat{\fu}_t$ of the true $\FP$ value 
$\fu_t$ and updates the data structure accordingly, maintaining an 
estimate of the cumulative parameter function $\hat{\Fu}_t = 
\hat{\Fu}_{t-1} + \hat{\fu}_t$ used to inform decisions in subsequent 
rounds. The online learning algorithm employed in this setup is $\Band$, 
which employs a fixed grid to partition the action space and selects 
$x_t$ proportionally to the exponential of its estimated cumulative 
reward, whose pseudocode is presented in~\Cref{alg:band}.

\begin{algorithm}[t]
    \begin{algorithmic}[1]
    \caption{$\Band$ Algorithm}
    \label{alg:band}
    \Input{$\eta, \mu, \gamma$ all positive and satisfying 
    $1/\mu \in \mathbb{N}$, $\gamma \leq \frac{1}{2}$, 
    $\eta \leq \gamma\mu$}
    \State Set $\mathcal{I}^{(j)} = \langle[(i-1)\mu:i\mu)
    \rangle_{i=1}^{1/\mu}$ be a family of intervals
    \State Set $\hat{H}_1(x) = 0$ for all $x \in [0,1]^d$
    \For{$t = 1, 2, \dots, T$}
        \State Define $q_t(x) = (1-\gamma)\frac{\exp(\eta\hat{H}_t(x))}
        {\int_{[0,1]^d}\exp(\eta\hat{H}_t(x))\,dx} + \gamma$
        \State Pick $x_t \sim q_t$
        \State Let $I_t$ be the hypercube of $\mathcal{I}$ containing 
        $x_t$
        \State Observe $\fu_t(x_t)$ and receive payoff $\fu_t(x_t)$
        \State Set $\hat{\fu}_t(x) = \frac{\fu_t(x)}{q_t(I_t)}$ for 
        all $x \in I_t$ and $\hat{\fu}_t(x) = 0$ for all 
        $x \in [0,1]^d \setminus I_t$
        \State Set $\hat{\Fu}_{t+1} \gets \hat{\Fu}_t + \hat{\fu}_t$
        \State Set $\hat{H}_{t+1} \gets \MF(x, \hat{\Fu}_{t+1})$ for 
        all $x \in [0,1]^d$
    \EndFor
    \end{algorithmic}
\end{algorithm}

\subsection{Bandit Setup and Adversarial Model}\label{app:bandit-setup}

The $\Band$ algorithm partitions $\AS = [0,1]^d$ into a fixed grid of 
$(1/\mu)^d$ cells prior to the start of the game. Each coordinate axis 
$[0,1]$ is divided into intervals of equal length $\mu$, forming the 
interval family $\mathcal{I}^{(i)} = \langle[x^{(i)}_j:x^{(i)}_{j+1}]
\rangle_{j=0}^{1/\mu-1}$ with $x^{(i)}_j = j\mu$. The resulting 
axis-aligned grid defines a collection of atomic regions (grid cells), 
which serve as the fundamental units tracked by the data structure. 
The parameter $\mu$ is chosen carefully to ensure that the grid cells 
capture potential discontinuities in $\fu_t$, enabling semi-accurate 
updates.

Since the grid is fixed in advance, $\Tree$ is static and requires no 
insertions throughout, operating exclusively over the predefined 
grid-induced atomic regions regardless of how the adversary selects 
$\Hp_t$ at each round. This allows us to relax the smoothness 
assumption to the standard oblivious $\sigma$-smoothed adversary of 
\Cref{sec:model} without any additional structural assumptions on $\Ax$: 
since the regions are fixed in advance and remain static throughout, 
the adversary's choice of hyperplanes does not affect the data 
structure. The hierarchical data structure $\Tree$ is constructed by 
initializing independent interval trees along each coordinate axis, 
instantiating a $d$-dimensional $\IT$ with $(1/\mu)^d$ atomic regions 
without employing any lazy insertion mechanisms.

\subsection{Data Structure and Query Handling for $\Band$ under 
Bandit Feedback}\label{app:bandit-ds}

We describe the data structure for the two-dimensional case; the 
$d$-dimensional generalization follows analogously. The data structure 
comprises a first-level interval tree $\Tree^{(0)}$ built over the 
endpoints of $\mathcal{I}^{(0)}$, with each node $v_0 \in \Tree^{(0)}$ 
maintaining a second-level interval tree $\Tree^{(1)}_{v_0}$ over 
$\mathcal{I}^{(1)}$. The alignment between $\RA(v_0)$ and $\RA(v_1)$ 
of a node $v_1 \in \Tree^{(1)}_{v_0}$ defines a rectangular region 
$[\RA(v_0)][\RA(v_1)]$ in the grid, and a pair of leaf nodes $l_0 \in 
\Tree^{(0)}$ and $l_1 \in \Tree^{(1)}_{l_0}$ uniquely determines an 
atomic grid cell $[\RA(l_0)][\RA(l_1)]$.

\paragraph{Key structural differences from the proportional sampling 
and full-information data structures.}
\begin{itemize}
    \item \textbf{Static grid with no insertions.} Since the grid of 
    $(1/\mu)^d$ cells is fixed prior to the game, $\Tree$ requires no 
    insertions throughout. All structures are purely regular interval 
    trees initialized once over the fixed grid $[\mathcal{I}^{(0)}] 
    \cdots [\mathcal{I}^{(d-1)}]$, with no lazy components 
    at any level.

    \item \textbf{Reward parameter vectors at the leaves rather than 
    aggregate reward vectors.} Unlike the proportional sampling and 
    full-information settings, here we maintain the estimated cumulative 
    reward parameter vectors $\hat{\Fu}_t(\rg)$ at the leaves of the 
    last level, rather than vectorized integrals of the reward. The 
    aggregate reward $\hat{\CR}(\rg)$ is computed on demand from these 
    parameter vectors using the closed-form formula of~\Cref{eq:expint}, 
    instantiated with $\hat{\Fu}_t$ in place of $\Fu_t$. Both the 
    parameter vector $\hat{\Fu}_t(\rg)$ and the scalar $\hat{\CR}(\rg)$ 
    must be maintained simultaneously: the scalar is needed for 
    proportional sampling during $\Draw$, while the vector is needed 
    to correctly recompute the scalar after each update 
    via~\Cref{eq:expint}, since terms of the form 
    $\prod_{i=0}^{d-1}\frac{1}{\eta\hat{\Fu}_t^{(i)}}$ appear 
    explicitly in the formula.
\end{itemize}

\paragraph{Stored quantities.}
For each node $v_1 \in \Tree^{(1)}_{v_0}$, we maintain:
\begin{itemize}
    \item A vector $\vec{\hat{\CR}}_{v_0,v_1} = \sum_{\rg \subseteq 
    [\RA(v_0)][\RA(v_1)]} \hat{\Fu}_t(\rg)$, representing the estimated 
    cumulative $\FP$ over all atomic regions $\rg$ contained in 
    $[\RA(v_0)][\RA(v_1)]$;
    \item A scalar $\hat{\CR}_{v_0,v_1} = \sum_{\rg \subseteq 
    [\RA(v_0)][\RA(v_1)]} \hat{\CR}(\rg)$, denoting the total estimated 
    exponentially weighted cumulative reward over the same region.
\end{itemize}
Since atomic regions reside at the leaf level, each region has the 
form $\rg' = [\RA(l_0)][\RA(l_1)]$. The leaf $l_1 \in \Tree^{(1)}_{l_0}$ 
directly stores $\vec{\hat{\CR}}_{l_0,l_1} = \hat{\Fu}_t(\rg')$, which 
is used to compute $\hat{\CR}(\rg')$ via~\Cref{eq:expint}. As long 
as $\vec{\hat{\CR}}_{l_0,l_1} = \hat{\Fu}_t(\rg')$ holds for all pairs 
$l_0 \in \Tree^{(0)}$, $l_1 \in \Tree^{(1)}_{l_0}$, correctness of 
the sampling algorithm is guaranteed.

\paragraph{Handling $\Update$.}
After selecting $x_t$ and observing $\fu_t(x_t)$, the algorithm 
constructs $\hat{\fu}_t(x_t) = \frac{\fu_t(x_t)}{q_t(I_t)}$ for the 
atomic region $I_t$ containing $x_t$, and $\hat{\fu}_t(x) = 0$ 
elsewhere. It locates the corresponding leaves $l_0 \in \Tree^{(0)}$ 
and $l_1 \in \Tree^{(1)}_{l_0}$ by first locating $x^{(0)}_t$ in 
$\Tree^{(0)}$ and then $x^{(1)}_t$ in $\Tree^{(1)}_{l_0}$, and updates 
the parameter vector at the leaf:
\[
\vec{\hat{\CR}}_{l_0,l_1} \gets \vec{\hat{\CR}}_{l_0,l_1} + 
\hat{\fu}_t(x_t),
\]
then recomputes $\hat{\CR}_{l_0,l_1}$ via~\Cref{eq:expint}. The 
updated values are propagated bottom-up along $p_{l_1}$ in 
$\Tree^{(1)}_{l_0}$, updating $\vec{\hat{\CR}}_{l_0,v_1}$ and 
$\hat{\CR}_{l_0,v_1}$ for each $v_1 \in p_{l_1}$ using the standard 
lazy propagation technique of~\cite{cohen2017online}. The algorithm 
then traverses $p_{l_0}$ in $\Tree^{(0)}$ and repeats the same update 
procedure in $\Tree^{(1)}_{v_0}$ for each $v_0 \in p_{l_0}$, 
propagating updates across the full structure. Since the partitioning 
is fixed, no lazy insertions are required; all updates are performed 
through simple bottom-up aggregations of $\vec{\hat{\CR}}$ and 
$\hat{\CR}$ values, ensuring the structure is always up to date and 
ready for sampling $x_{t+1}$. The $d$-dimensional extension follows 
directly, with updates propagating bottom-up through all $d$ levels.

\paragraph{Handling $\Draw(t)$.}
The $\Draw(t)$ procedure follows the same two-stage sampling scheme 
as in~\Cref{section:ds-qhandling}: a root-to-leaf traversal of 
$\ATree^{(0)}$ selects an atomic horizontal interval $\RA(v_0)$ with 
probability proportional to $\hat{\CR}([\RA(v_0)][0:1])$, followed 
by a conditioned traversal of $\ATree^{(1)}$ to select an atomic 
vertical interval $\RA(v_1)$ with probability proportional to 
$\hat{\CR}([\RA(v_0)][\RA(v_1)])$, together forming the sampled atomic 
region $[\RA(v_0)][\RA(v_1)]$. Left/right probabilities at each node 
are computed directly from the stored scalar values $\hat{\CR}_{v_0,v_1}$ 
without any scaling step, since the structure contains no lazy 
components.

\paragraph{Sampling within an atomic region $\rg$.}
After $\Draw(t)$ selects an atomic region $\rg$, the learner samples 
$x_t \in \rg$ from the distribution
\[
x_t \sim \frac{\exp(\eta \hat{H}_t(x))}{\int_{x \in \rg} 
\exp(\eta \hat{H}_t(x))\,dx}.
\]
This sampling step is performed using the Hit-and-Run algorithm, which 
by \Cref{them:vempa} produces a sample within $O(d^3)$ steps. The 
density $\frac{\exp(\eta \hat{H}_t(x))}{\int_{\rg}\exp(\eta 
\hat{H}_t(x))\,dx}$ is log-concave, as it is the normalization of a 
log-concave function over a convex domain. Since each atomic region 
$\rg$ is a convex subset of $\AS$ and $\hat{H}_t(x)$ is linear over 
$\rg$ under piecewise-linear rewards, the conditions of 
\Cref{them:vempa} are satisfied, guaranteeing efficient mixing and 
accurate sampling within $O(d^3)$ steps.

\subsection{Regret Analysis}\label{section:regretband}
Here we analyze the case of bandit observations, where only the $\FP$ 
corresponding to the chosen action $x_t$, i.e., $\fu_t(x_t)$, is 
observed. The actual real-valued reward is recovered via 
$\MF(x_t, \fu_t(x_t))$. The analysis of piecewise-constant reward 
functions has been previously established in~\cite{cohen2017online} 
and extends directly to $d$ dimensions via \Cref{obs:ckd}. We therefore 
focus on piecewise-linear rewards in the general $d$-dimensional case.

\begin{theorem}[$d$-dimensional $\AS$ with piecewise-linear $\fu_t$]
\label{thm:ddblin}
Assume $\hat{\Fu}^{(j)}_{T+1}(x^*) \leq \frac{2}{\eta\mu}$ for all 
$0 \leq j \leq d-1$, $\gamma \leq \frac{1}{2}$, and $\eta \leq 
\gamma\mu$. For a $d$-dimensional $\AS$ with axis-parallel decision 
boundaries, with probability 1, the expected regret of $\Band$ is 
bounded by
\begin{align*}
2\gamma T + \frac{2\eta T}{\mu} + dk\mu\sigma T 
- \frac{\sum_{i=0}^{d-1}\log\!\left(\mu - 
\frac{\eta\mu^2\Fu^{(i)}_{T+1}(x^*)}{2}\right)}{\eta}
\leq 2\gamma T + \frac{2\eta T}{\mu} + dk\mu\sigma T 
- \frac{d\log\mu}{\eta} + \frac{d}{\eta},
\end{align*}
where $x^* \in [0,1)^d$ is the optimal action in hindsight not 
coinciding with any discontinuity of $\fu_t$, and $\rg^* = 
\prod_{i=0}^{d-1}[(l^*_i-1)\mu:l^*_i\mu)$ is the grid cell containing 
$x^*$ such that $\Fu_{T+1}(x^*)$ is constant and optimal within it.
\end{theorem}

\begin{proof}
The proof follows the same telescoping approach as in 
\Cref{section:regretful}, with $W_t$ now defined as a function of the 
estimated rewards $\hat{\fu}_t$. We define
\[
W_t = \int_{[0,1]^d} \exp\!\left(\eta \hat{H}_t(x)\right) dx,
\]
where $\hat{H}_t = \hat{H}_{t-1} + \hat{\RRF}_t$, and bound $W_{T+1}$ 
from both sides.

\paragraph{Upper bound on $W_{T+1}$.}
We bound $\frac{W_{t+1}}{W_t}$ as
\begin{align}
\frac{W_{t+1}}{W_t} 
&= \frac{\int_{[0,1]^d} \exp(\eta\hat{H}_t(x)) 
e^{\eta\hat{\RRF}_t(x)}\,dx}
{\int_{[0,1]^d} \exp(\eta\hat{H}_t(x))\,dx} \nonumber\\
&= \frac{1}{1-\gamma}\int_{[0,1]^d}(q_t(x)-\gamma)
e^{\eta\hat{\RRF}_t(x)}\,dx \nonumber\\
&\overset{(a)}{\leq} \frac{1}{1-\gamma}\int_{[0,1]^d}
(q_t(x)-\gamma)\left(1 + \eta\hat{\RRF}_t(x) + 
\eta^2\hat{\RRF}^2_t(x)\right)dx \nonumber\\
&\leq 1 + \frac{1}{1-\gamma}\int_{[0,1]^d}
q_t(x)\left(\eta\hat{\RRF}_t(x) + 
\eta^2\hat{\RRF}^2_t(x)\right)dx \nonumber\\
&\leq \exp\!\left(\frac{\eta}{1-\gamma}\int_{[0,1]^d} 
q_t(x)\hat{\RRF}_t(x)\,dx + \frac{\eta^2}{1-\gamma}
\int_{[0,1]^d} q_t(x)\hat{\RRF}^2_t(x)\,dx\right),
\label{ineq:Band_1}
\end{align}
where (a) applies $e^x \leq 1 + x + x^2$ for $x \leq 1$. 
Telescoping via~\eqref{eq:LEXP3_1} and using the importance-weighted 
estimator $\hat{\RRF}_t(x) = \frac{\RRF_t(x)}{q_t(I_t)}
\mathbf{1}[x \in I_t]$:
\begin{align}
W_{T+1} &\leq \exp\!\left(\frac{\eta}{1-\gamma}\sum_{t=1}^T
\sum_{I \in \mathcal{I}}\int_I q_t(x)\frac{\RRF_t(x)}{q_t(I)}\,dx 
+ \frac{\eta^2}{1-\gamma}\sum_{t=1}^T\int_{[0,1]^d} 
q_t(x)\hat{\RRF}^2_t(x)\,dx\right) \nonumber\\
&\leq \exp\!\left(\frac{\eta}{1-\gamma}\sum_{t=1}^T
\sum_{I \in \mathcal{I}}\int_I q_t(x)\frac{\RRF_t(x)}{q_t(x)}\,dx 
+ \frac{\eta^2}{1-\gamma}\sum_{t=1}^T\int_{[0,1]^d} 
q_t(x)\hat{\RRF}^2_t(x)\,dx\right) \nonumber\\
&\leq \exp\!\left(\frac{\eta}{1-\gamma}\sum_{t=1}^T\RRF_t(x_t) 
+ \frac{\eta^2}{1-\gamma}\sum_{t=1}^T\hat{\RRF}_t(x_t)\right).
\label{ineq:Band_2}
\end{align}

\paragraph{Lower bound on $W_{T+1}$.}
We restrict the integral to the optimal grid cell $\rg^*$. Since 
$\fu_t$ is piecewise-linear with axis-parallel boundaries, $x^*$ lies 
at a corner of $\rg^*$; without loss of generality assume $x^* = 
\langle l^*_0\mu, \dots, l^*_{d-1}\mu\rangle$. The integral factorizes 
over coordinates:
\begin{align}
W_{T+1} &\geq \int_{\rg^*}\exp(\eta\hat{H}_{T+1}(x))\,dx \nonumber\\
&= \prod_{i=0}^{d-1}\int_{(l^*_i-1)\mu}^{l^*_i\mu}
\exp(\eta\hat{H}_{T+1}(x))\,dx^{(i)}. \label{ineq:Band_3a}
\end{align}
For each coordinate $i$, since $\hat{H}_{T+1}$ is linear over $\rg^*$ 
with slope $\hat{\Fu}^{(i)}_{T+1}(x^*)$ and $x^*$ is at the right 
endpoint $l^*_i\mu$, we have for all $x^{(i)} \in [(l^*_i-1)\mu, 
l^*_i\mu]$:
\[
\hat{H}_{T+1}(x) \geq \hat{Y}^\Opt - 
\hat{\Fu}^{(i)}_{T+1}(x^*)(l^*_i\mu - x^{(i)}).
\]
Therefore:
\begin{align}
\int_{(l^*_i-1)\mu}^{l^*_i\mu}\exp(\eta\hat{H}_{T+1}(x))\,dx^{(i)}
&\geq e^{\eta\hat{Y}^\Opt}\int_{(l^*_i-1)\mu}^{l^*_i\mu}
\exp\!\left(-\eta\hat{\Fu}^{(i)}_{T+1}(x^*)
(l^*_i\mu - x^{(i)})\right)dx^{(i)} \nonumber\\
&= e^{\eta\hat{Y}^\Opt}
\int_0^{\mu}\exp\!\left(-\eta\hat{\Fu}^{(i)}_{T+1}(x^*)
\,s\right)ds \nonumber\\
&\overset{(b)}{\geq} e^{\eta\hat{Y}^\Opt}
\left(\mu - \frac{\eta\mu^2\hat{\Fu}^{(i)}_{T+1}(x^*)}{2}\right),
\label{ineq:Band_3b}
\end{align}
where (b) applies $e^{-u} \geq 1 - u$ to bound 
$\int_0^\mu e^{-\eta\hat{\Fu}^{(i)}_{T+1}(x^*)s}\,ds \geq 
\int_0^\mu (1 - \eta\hat{\Fu}^{(i)}_{T+1}(x^*)s)\,ds = \mu - 
\frac{\eta\mu^2\hat{\Fu}^{(i)}_{T+1}(x^*)}{2}$. Combining 
\eqref{ineq:Band_3a} and \eqref{ineq:Band_3b}:
\begin{align}
W_{T+1} \geq e^{\eta\hat{Y}^\Opt}\prod_{i=0}^{d-1}
\left(\mu - \frac{\eta\mu^2\hat{\Fu}^{(i)}_{T+1}(x^*)}{2}\right).
\label{ineq:Band_3}
\end{align}

\paragraph{Combining the bounds.}
Merging~\eqref{ineq:Band_2} and~\eqref{ineq:Band_3}, taking 
logarithms and dividing by $\eta$:
\begin{align*}
\hat{Y}^\Opt + \frac{\sum_{i=0}^{d-1}\log\!\left(\mu - 
\frac{\eta\mu^2\hat{\Fu}^{(i)}_{T+1}(x^*)}{2}\right)}{\eta}
\leq \frac{1}{1-\gamma}Y^\Band + \frac{\eta}{1-\gamma}
\sum_{t=1}^T\hat{\RRF}_t(x_t).
\end{align*}
Adding $Y^\Opt$ to both sides:
\begin{align}
Y^\Opt - Y^\Band 
&\leq \frac{\gamma}{1-\gamma}Y^\Band + \frac{\eta}{1-\gamma}
\hat{Y}^\Band - \frac{\sum_{i=0}^{d-1}\log\!\left(\mu - 
\frac{\eta\mu^2\Fu^{(i)}_{T+1}(x^*)}{2}\right)}{\eta} 
+ Y^\Opt - \hat{Y}^\Opt \nonumber\\
&\overset{(a)}{\leq} 2\gamma Y^\Band + 2\eta\hat{Y}^\Band 
- \frac{\sum_{i=0}^{d-1}\log\!\left(\mu - 
\frac{\eta\mu^2\Fu^{(i)}_{T+1}(x^*)}{2}\right)}{\eta} 
+ Y^\Opt - \hat{Y}^\Opt,
\label{ineq:Band_4}
\end{align}
where (a) holds since $\gamma \leq \nicefrac{1}{2}$. Taking 
expectations and bounding each term independently: following 
\cite{cohen2017online}, the first two terms are bounded by $2\gamma T$ 
and $\nicefrac{2\eta T}{\mu}$ respectively. By \Cref{obs:ckd}, 
\Cref{obs:ck1} applies coordinate-wise, giving 
$\mathbb{E}[Y^\Opt - \hat{Y}^\Opt] \leq dk\mu\sigma T$. Applying 
$-\log(1-u) \leq u$ for $u \in [0,1)$ and the assumption 
$\hat{\Fu}^{(i)}_{T+1}(x^*) \leq \frac{2}{\eta\mu}$ coordinate-wise 
then yields the stated bound.
\end{proof}

\begin{corollary}\label{cor:ddblin}
For $\gamma = \frac{1}{2}T^{-1/3}$, $\mu = \frac{1}{k\sigma}T^{-1/3}$, 
and $\eta = T^{-2/3}$, the regret of \Cref{thm:ddblin} is bounded by
\[
(3 + 2dk\sigma)T^{2/3} - \frac{d}{3}T^{2/3}\log T 
= O\!\left(\operatorname{poly}(d, k, \sigma, \log T)\, T^{2/3}\right).
\]
\end{corollary}

\subsection{Time Complexity}
\begin{theorem}\label{thm:bandcomp}
Under the bandit feedback setting with a fixed grid of $(1/\mu)^d$ 
cells, the per-round time complexity of $\Band$ is
\[
O\!\left(d\log^{d}(1/\mu) + d^3\right).
\]
In particular, for the optimal choice $\mu = \frac{1}{k\sigma}T^{-1/3}$ 
this becomes $O\!\left(d\log^{d}(k\sigma T^{1/3}) + d^3\right)$.
\end{theorem}
\begin{proof}
Since $\Tree$ is static with a fixed grid of $(1/\mu)^d$ cells, all 
level-$i$ trees have height $O(\log(1/\mu))$ throughout, as no 
insertions occur after initialization. Each round consists of two 
operations.

\paragraph{$\Update$ cost.}
After observing $\fu_t(x_t)$, the algorithm locates the leaf 
$(l_0,\dots,l_{d-1})$ containing $x_t$ and propagates the update 
bottom-up through all $d$ levels. At each level $i$, the traversal 
visits $O(\log(1/\mu))$ nodes with $O(1)$ cost each, giving total 
$\Update$ cost $O(d\log(1/\mu))$.

\paragraph{$\Draw$ cost.}
The Draw procedure proceeds in $d$ sequential stages. At each stage 
$i\in\{0,\dots,d-1\}$, one operation over the static $d$-layered 
structure costs $O\!\left(\prod_{j=0}^{d-1}\EH(\Tree^{(j)})\right) = 
O(\log^d(1/\mu))$, as it traverses all $d$ layers each of height 
$O(\log(1/\mu))$. Summing over all $d$ stages, the atomic-region 
sampling cost is $O(d\log^d(1/\mu))$. Given the sampled atomic region 
$\rho$, sampling $x_t\in\rho$ via Hit-and-Run requires $O(d^3)$ steps 
by \Cref{them:vempa}, since $g(x,\widehat{F}_t(x))$ is log-concave 
over the convex region $\rho$. The total $\Draw$ cost is therefore 
$O(d\log^{d}(1/\mu)+d^3)$.

\paragraph{Total per-round cost.}
The $\Update$ cost $O(d\log(1/\mu))$ is dominated by the $\Draw$ cost, 
giving total per-round complexity
\[
O\!\left(d\log^{d}(1/\mu) + d^3\right).
\]
Substituting the optimal $\mu = \frac{1}{k\sigma}T^{-1/3}$ yields 
$O\!\left(d\log^{d}(k\sigma T^{1/3}) + d^3\right)$.
\end{proof}

\end{document}